%% file: arXiv.tex
\documentclass[a4paper, 11pt, twoside]{article}
\usepackage{fullpage}

\usepackage{amsmath,amsfonts,amssymb,amsthm,array}
\usepackage{algorithm}
\usepackage{caption}
\usepackage{bm}
\usepackage{color}
\usepackage{graphicx}
\usepackage{enumitem}
\usepackage{nicefrac}

\usepackage{url}            
\usepackage{booktabs}       
\usepackage{amsfonts}       
\usepackage{nicefrac}       
\usepackage{microtype}      

\usepackage{natbib}
\usepackage{sidecap}

\input{macros.tex}

\usepackage{comment}

\usepackage[utf8]{inputenc} 
\usepackage[T1]{fontenc}    
\usepackage{url}            
\usepackage{booktabs}       
\usepackage{amsfonts}       
\usepackage{nicefrac}       
\usepackage{microtype}      
\usepackage{xcolor}         
\usepackage{subcaption}
\usepackage[warn]{textcomp}

\usepackage{enumitem}

\begin{document}

\title{\textbf{Federated Optimization Algorithms with Random Reshuffling and Gradient Compression}}

\author{Abdurakhmon Sadiev$^{1,2}$\thanks{Part of the work was done while A.~Sadiev was a research intern at KAUST, working in the Optimization and Machine Learning Lab of P.\ Richt\'{a}rik.} \quad
  Grigory Malinovsky$^{1}$\quad 
  Eduard Gorbunov$^{1,2,3,4}$\thanks{Corresponding author:  eduard.gorbunov@mbzuai.ac.ae. Part of the work was done when E.~Gorbunov was a researcher at MIPT and Mila \& UdeM and also a visiting researcher at KAUST, in the Optimization and Machine Learning Lab of P.~Richt\'arik.}\\ Igor Sokolov$^{1}$ \quad Ahmed Khaled$^5$\quad  Konstantin Burlachenko$^1$\quad Peter Richt\'{a}rik$^1$}
\date{$^1$ King Abdullah University of Science and Technology, Saudi Arabia\\
	$^2$ Moscow Institute of Physics and Technology, Russian Federation\\    
    $^3$ Mila, Universit\'e de Montr\'eal, Canada\\
	$^4$ Mohamed bin Zayed University of Artificial Intelligence, UAE\\    
    $^5$ Princeton University, USA}
    
\maketitle

\begin{abstract}

Gradient compression is a popular technique for improving communication complexity of stochastic first-order methods in distributed training of machine learning models. However, the existing works consider only with-replacement sampling of stochastic gradients. In contrast, it is well-known in practice and recently confirmed in theory that stochastic methods based on without-replacement sampling, e.g., Random Reshuffling (\algname{RR}) method, perform better than ones that sample the gradients with-replacement. In this work, we close this gap in the literature and provide the first analysis of methods with gradient compression and without-replacement sampling. We first develop a na\"ive combination of random reshuffling with gradient compression (\algname{Q-RR}). Perhaps surprisingly, but the theoretical analysis of \algname{Q-RR} does not show any benefits of using \algname{RR}. Our extensive numerical experiments confirm this phenomenon. This happens due to the additional compression variance. To reveal the true advantages of \algname{RR} in the distributed learning with compression, we propose a new method called \algname{DIANA-RR} that reduces the compression variance and has provably better convergence rates than existing counterparts with with-replacement sampling of stochastic gradients. Next, to have a better fit to Federated Learning applications, we incorporate local computation, i.e., we propose and analyze the variants of \algname{Q-RR} and \algname{DIANA-RR} -- \algname{Q-NASTYA} and \algname{DIANA-NASTYA} that use local gradient steps and different local and global stepsizes. Finally, we conducted several numerical experiments to illustrate our theoretical results.

\end{abstract}

\section{Introduction}

Federated learning (FL) \citep{FEDLEARN,mcmahan2017communication} is a framework for distributed learning and optimization where multiple nodes connected over a network try to collaboratively carry out a learning task. Each node has its own dataset and cannot share its data with other nodes or with a central server, so algorithms for federated learning often have to rely on local computation and cannot access the entire dataset of training examples. Federated learning has applications in language modelling for mobile keyboards~\citep{liu21_feder_learn_meets_natur_languag_proces}, healthcare~\citep{antunes22_fed_learn_healthcare}, wireless communications~\citep{yang22_fed_learn_6g}, and continues to find applications in many other areas~\citep{kairouz19_advan_open_probl_feder_learn}.

Federated learning tasks are often solved through \emph{empirical-risk minimization} (ERM), where the \(m\)-th devices contributes an empirical loss function \(f_m (x)\) representing the average loss of model \(x\) on its local dataset, and our goal is to then minimize the average loss over all the nodes:
\begin{align}
\label{eq:erm-opt-orig}
\min_{x \in \mathbb{R}^d} \left[ f(x) \eqdef \frac{1}{M} \sum_{m=1}^M f_m (x) \right],
\end{align}
where the function \(f\) represents the average loss. Every \(f_m\) is an average of sample loss functions \(f_{m}^{i}\) each representing the loss of model \(x\) on the \(i\)-th datapoint on the \(m\)th clients' dataset: that is for each \(m \in \{1, 2, \ldots, M\}\) we have
\begin{align*}
f_m (x) \eqdef \frac{1}{n_m} \sum_{i=1}^{n_m} f_{m}^{i} (x).
\end{align*}
For simplicity we shall assume that the datasets on all clients are of equal size: \(n_1 = n_2 = \ldots = n_M\), though this assumption is only for convenience and our results easily extend to the case when clients have datasets of unequal sizes. Thus our optimization problem is
\begin{align}\label{eq:erm-opt}
\min_{x \in \mathbb{R}^d} \left[ f(x) = \frac{1}{nM} \sum_{m=1}^M \sum_{i=1}^n f_{m}^{i} (x) \right].
\end{align}
Because \(d\) is often very large in practice, the dominant paradigm for solving~\eqref{eq:erm-opt} relies on first-order (gradient) information. Federated learning algorithms have access to two key primitives: (a) local computation, where for a given model \(x \in \mathbb{R}^d\) we can compute stochastic gradients \(\nabla f_{m}^{i} (x)\) locally on client \(m\), and (b) communication, where the different clients can exchange their gradients or models with a central server.

\subsection{Communication bottleneck: from one to multiple local steps}

In practice, communication is more expensive than local computation~\citep{kairouz19_advan_open_probl_feder_learn}, and as such one of the chief concerns of algorithms for federated learning is communication efficiency. Algorithms for federated learning have thus focused on achieving communication efficiency, with one common ingredient being the use of multiple \emph{local steps} \citep{wang21_field_guide_to_feder_optim,malinovskiy2020local}, where each node uses multiple gradients locally for several descent steps between communication steps. In general, algorithms using local steps fit the following pattern of generalized \algname{FedAvg} (due to \citep{wang21_field_guide_to_feder_optim}); see Algorithm~\ref{alg:framework}.

\begin{algorithm}[h]
\caption{The generalized \algname{FedAvg} framework for methods with local steps}
\label{alg:framework}
	\begin{algorithmic}[1]
		\Require $x_0$ -- starting point, $\gamma > 0$ -- local stepsize, $\eta > 0$ -- global stepsize, \(H\) - number of local steps
	    \For{communication rounds $t =0,1,\dots, T-1$}
    		\For{clients $m \in [M]$ in parallel}
    		    \State Receive $x_t$ from the server and set $x^0_{t,m} = x_t$
    		    \For{local steps $i = 0, 1,\dots, H$}
    		        \State Set $x^{i+1}_{t,m} = \mathrm{ClientUpdate}(x_{t, m}^i, \gamma, m, i)$
    		    \EndFor
    		    \State Send \(x_{t, m}^H\) to the server, or alternatively send the update \(\Delta_{t, m} = x_{t, m}^H - x_t\) to the server
    		\EndFor
    		\State Compute $x_{t+1} = \frac{1}{M} \sum_{m=1}^M x_{t, m}^H$ (or \(x_{t+1} = x_t + \frac{1}{M} \sum_{m=1}^M \Delta_{t, m} \) if clients sent updates)
            \State Broadcast $x_{t+1}$ to the clients
    	\EndFor
    	\Ensure $x_T$
	\end{algorithmic}
\end{algorithm}

When the client update method in Algorithm~\ref{alg:framework} is stochastic gradient descent, we get the \algname{FedAvg} algorithm (also known as \algname{Local SGD}). While \algname{FedAvg} is popular in practice, recent theoretical progress has given tight analysis of the algorithm and shown that it can be definitively slower than its non-local counterparts~\citep{khaled19_tight_theor_local_sgd_ident_heter_data,woodworth20_minib_vs_local_sgd_heter_distr_learn, glasgow21_sharp_bound_feder_averag_local}. However, by using bias-reduction techniques one can use local steps and still maintain convergence rates at least as fast as non-local methods~\citep{karimireddy19_scaff}, or in some cases even faster~\citep{mishchenko22_proxs}. Thus local steps continue to be a useful algorithmic ingredient in both theory and practice for achieving communication efficiency.

\subsection{Communication bottleneck: from full-dimensional to compressed communication}

Another useful ingredient in distributed optimization is \emph{gradient compression}, where each client sends a compressed or quantized version of their update \(\Delta_{t, m}\) instead of the full update vector, potentially saving communication bandwidth by sending fewer bits over the network. There are many operators that can be used for compressing the update vectors: stochastic quantization~\citep{alistarh17_qsgd}, random sparsification~\citep{wangni18_sparsification, stich18_spars_sgd_with_memor}, and others~\citep{tang20_commun_effic_distr_deep_learn}. In this work we consider compression operators satisfying the following assumption:

\begin{assumption}\label{asm:quantization_operators}
A compression operator is an operator \(\cQ: \mathbb{R}^d \rightarrow \mathbb{R}^d\) such that for some \(\omega > 0\), the relations
\begin{align*}
\ec{\cQ(x)} = x \qquad \text {and} \qquad \ecn{\cQ(x) - x} \leq \omega \sqn{x}
\end{align*}
hold for \(x \in \mathbb{R}^d\).
\end{assumption}

Unbiased compressors can reduce the number of bits clients communicate per round, but also increases the variance of the stochastic gradients used slowing down overall convergence, see e.g.~\citep[Theorem 5.2]{khirirat18_distr_learn_with_compr_gradien} and \citep[Theorem 1]{stich20_commun_compr_distr_optim_heter_data}. By using control iterates, \citet{mishchenko19_distr_learn_with_compr_gradien_differ} developed \algname{DIANA}---an algorithm that can reduce the variance  due to gradient compression with unbiased compression operators, and thus ensure fast convergence. \algname{DIANA} has been extended and analyzed in many settings~\citep{horvath19_stoch_distr_learn_with_gradien, stich20_commun_compr_distr_optim_heter_data, safaryan21_smoot_matric_beat_smoot_const} and forms an important tool in our arsenal for using gradient compression.

\subsection{Communication bottleneck: from with replacement to without replacement sampling}

The algorithmic framework of generalized \algname{FedAvg} (Algorithm~\ref{alg:framework}) requires specifying a client update method that is used locally on each client. The typical choice is stochastic gradient descent (\algname{SGD}), where at each time step we sample \(j\) from \(\{1, \ldots, n\}\) uniformly at random and then do a gradient descent step using the stochastic gradient \(\nabla f_m^j (x_{t, m}^i)\), resulting in the client update:
\begin{align*}
\mathrm{ClientUpdate}(x_{t, m}^i, \gamma, m, i) = x_{t, m}^i - \gamma \nabla f_m^j (x_{t, m}^i).
\end{align*}
This procedure thus uses \emph{with-replacement sampling} in order to select the stochastic gradient used at each local step from the dataset on node \(m\). In contrast, we can use {\em without-replacement sampling} to select the gradients: that is, at the beginning of each \emph{epoch} we choose a permutation \(\pi_1, \pi_2, \ldots, \pi_n\) of \(\{1, 2, \ldots, n\}\) and do the \(i\)-th update using the \(\pi_i\)-ith gradient:
\begin{align*}
\mathrm{ClientUpdate}(x_{t, m}^i, \gamma, m, i) = x_{t, m}^i - \gamma \nabla f_m^{\pi_i} (x_{t, m}^i).
\end{align*}
Without-replacement sampling \algname{SGD}, also known as Random Reshuffling (\algname{RR}), typically achieves better asymptotic convergence rates compared to with-replacement SGD and can improve upon it in many settings as shown by recent theoretical progress~\citep{mishchenko20_random_reshuf, ahn20_sgd_with_shuff, rajput20_closin_conver_gap_sgd_without_replac, safran21_random_shuff_beats_sgd_only}. While with-replacement \algname{SGD} achieves an error proportional to \(\mathcal{O} \br{\frac{1}{T}}\) after \(T\) steps~\citep{stich19_unified_optim_analy_stoch_gradien_method}, Random Reshuffling achieves an error of \(\mathcal{O} \br{\frac{n}{T^2}}\) after \(T\) steps, faster than SGD when the number of steps \(T\) is large.

The success of \algname{RR} in the single-machine setting has inspired several recent methods that use it as a local update method as part of distributed training: \citet{mishchenko21_proxim_feder_random_reshuf} developed a distributed variant of random reshuffling, \algname{FedRR}. \algname{FedRR} fits into the framework of Algorithm~\ref{alg:framework} and uses RR as a local client update method in lieu of \algname{SGD}. They show that \algname{FedRR} can improve upon the convergence of \algname{Local SGD} when the number of local steps is fixed as the local dataset size, i.e.\ when \(H = n\). \citet{yun21_minib_vs_local_sgd_with_shuff} study the same method under the name \algname{Local RR} under a more restrictive assumption of bounded inter-machine gradient deviation and show that by varying \(H\) to be smaller than \(n\) better rates can be obtained in this setting than the rates  of~\citet{mishchenko21_proxim_feder_random_reshuf}. Other work has explored more such combinations between \algname{RR} and distributed training algorithms~\citep{huang21_distr_random_reshuf_over_networ,malinovsky22_server_side_steps_sampl_without, horvath22_fedsh}.

\subsection{Three tricks for achieving communication efficiency}

To summarize, we have at our disposal the following tricks and techniques for achieving communication efficiency in distributed training: (a) Local steps, (b) Gradient compression, and (c) Random Reshuffling. Each has found its use in federated learning and poses its own challenges, requiring special analysis or bias/variance-reduction techniques to achieve the best theoretical convergence rates and practical performance. Client heterogeneity causes local methods (with or without random reshuffling) to be biased, hence requiring bias-reduction techniques~\citep{karimireddy19_scaff, murata21_bias_varian_reduc_local_sgd} or decoupling local and server stepsizes~\citep{malinovsky22_server_side_steps_sampl_without}. Gradient compression reduces the number of bits clients have to send per round, but causes an increase in variance, and  we hence also need variance-reduction techniques to achieve better convergence rates under gradient compression~\citep{mishchenko19_distr_learn_with_compr_gradien_differ, stich19_error_feedb_framew}. However, it is not clear apriori \emph{how these techniques should be combined to improve the convergence speed}, and this is our starting point.

\subsection{Contributions}

In this paper, we aim to develop methods for federated optimization that combine gradient compression, random reshuffling, and/or local steps. While each of these techniques can aid in reducing the communication complexity of distributed optimization, their combination is under-explored. Thus our goal is to design methods that improve upon existing algorithms in convergence rates and in practice. We summarize our contributions as:

\begin{itemize}[leftmargin=*]

\item[$\diamond$]\textbf{The issue: na\"ive combination has no improvements.} As a natural step towards our goal, we start with non-local methods and propose a new algorithm, \algname{Q-RR} (Algorithm~\ref{alg_new_Q_RR}), that combines random reshuffling with gradient compression at every communication round. However, for \algname{Q-RR} our theoretical results do not show any improvement upon \algname{QSGD} when the compression level is reasonable. Moreover, we observe similar performance of \algname{Q-RR} and \algname{QSGD} in various numerical experiments. Therefore, we conclude that this phenomenon is not an artifact of our analysis but rather an issue of \algname{Q-RR}: communication compression adds an additional noise that dominates the one coming from the stochastic gradients sampling.

\item[$\diamond$]\textbf{The remedy: reduction of compression variance.} To remove the additional variance added by the compression and unleash the potential of Random Reshuffling in distributed learning with compression, we propose \algname{DIANA-RR} (Algorithm~\ref{alg_new_RR_DIANA}), a combination of \algname{Q-RR} and the \algname{DIANA} algorithm. We derive the convergence rates of the new method and show that it improves upon the convergence rates of \algname{Q-RR}, \algname{QSGD}, and \algname{DIANA}. We point out that to achieve such results we use $n$ shift-vectors per worker in \algname{DIANA-RR} unlike \algname{DIANA} that uses only $1$ shift-vector.

\item[$\diamond$]\textbf{Extensions to the local steps.} Inspired by the \algname{NASTYA} algorithm of \citet{malinovsky22_server_side_steps_sampl_without}, we propose a variant of \algname{NASTYA}, \algname{Q-NASTYA} (Algorithm~\ref{alg:Q_NASTYA}), that na\"ively mixes quantization, local steps with random reshuffling, and uses different local and server stepsizes. Although it improves in per-round communication cost over \algname{NASTYA} but, similar to \algname{Q-RR}, we show that \algname{Q-NASTYA} suffers from added variance due to gradient quantization. To overcome this issue, we propose another algorithm, \algname{DIANA-NASTYA} (Algorithm~\ref{alg:diana-nastya}), that adds \algname{DIANA}-style variance reduction to \algname{Q-NASTYA} and removes the additional variance.
\end{itemize}

Finally, to illustrate our theoretical findings we conduct experiments on federated linear regression tasks.

\subsection{Related work}

Federated optimization has been the subject of intense study, with many open questions even in the setting when all clients have identical data~\citep{woodworth20_is_local_sgd_better_than_minib_sgd, woodworth20_minib_vs_local_sgd_heter_distr_learn, woodworth21_min_max_compl_distr_stoch}. The \algname{FedAvg} algorithm (also known as \algname{Local SGD}) has also been a subject of intense study, with tight bounds obtained only very recently by \citet{glasgow21_sharp_bound_feder_averag_local}. It is now understood that using many local steps adds bias to distributed \algname{SGD}, and hence several methods have been developed to mitigate it, e.g.~\citep{karimireddy19_scaff, murata21_bias_varian_reduc_local_sgd}, see the work of \citet{gorbunov21_local_sgd} for a unifying lens on many variants of \algname{Local SGD}. Note that despite the bias, even vanilla \algname{FedAvg}/\algname{Local SGD} still reduces the overall communication overhead in practice~\citep{ortiz21_trade_offs_local_sgd_at_scale}.

There are several methods that combine compression or quantization and local steps: both \citet{Basu19_qsparse_local_sgd} and \citet{reisizadeh19_fedpaq} combined \algname{Local SGD} with quantization and sparsification, and \citet{haddadpour20_feder_learn_with_compr} later improved their results using a gradient tracking method, achieving linear convergence under strong convexity. In parallel, \citet{mitra21_lin_converg_fl} also developed a variance-reduced method, \algname{FedLin}, that achieves linear convergence under strong convexity despite using local steps and compression. The paper most related to our work is \citep{malinovsky22_feder_random_reshuf_with_compr_varian_reduc} in which the authors combine \emph{iterate} compression, random reshuffling, and local steps. We study \emph{gradient} compression instead, which is a more common form of compression in both theory and practice~\citep{kairouz19_advan_open_probl_feder_learn}. We compare our results against \citep{malinovsky22_feder_random_reshuf_with_compr_varian_reduc} and show we obtain better rates compared to their work.

\section{Algorithms and Convergence Theory}

We will primarily consider the setting of strongly-convex and smooth optimization. We assume that the average function \(f\) is strongly convex:

\begin{assumption}\label{asm:sc_general_f}
Function $f: \R^d \rightarrow \R$ is $\mu$-strongly convex, i.e., for all \(x, y \in \mathbb{R}^d\),
\begin{equation}
\label{convexity}
    f(x) - f(y) - \la\nabla f(y), x - y\ra \geq \frac{\mu}{2}\|x-y\|^2,
\end{equation}
and functions $f^i_1, f^i_2, \dots, f^i_M: \R^d \rightarrow \R$  are convex for all $i = 1,\dots, n$.
\end{assumption}

Examples of objectives satisfying Assumption~\ref{asm:sc_general_f} include \(\ell_2\)-regularized linear and logistic regression. Throughout the paper, we assume that \(f\) has the unique minimizer \(x_{*} \in \mathbb{R}^{d}\). We also use the assumption that each individual loss \(f_m^i\) is smooth, i.e.\ has Lipschitz-continuous first-order derivatives:

\begin{assumption}\label{asm:lip_max_f_m}
Function $f^i_m: \R^d \rightarrow \R$ is $L_{i,m}$-smooth for every $i \in [n]$ and $m \in [M]$, i.e., for all \(x, y \in \mathbb{R}^d\) and for all \(m \in [M]\) and \(i \in [n]\),
\begin{equation}
    \|\nabla f^i_m(x) - \nabla f^i_m(y)\|\leq L_{i,m}\|x - y\|.
\end{equation}
We denote the maximal smoothness constant as $L_{\max} \eqdef \max_{i,m}L_{i,m}$.
\end{assumption}

For some methods, we shall additionally impose the assumption that \emph{each} function is strongly convex:

\begin{assumption}
\label{asm:sc_each_f_m}
Each function $f_m^i: \R^d \rightarrow \R$ is $\widetilde{\mu}$-strongly convex.
\end{assumption}

The \emph{Bregman divergence} associated with a convex function \(h\) is defined for all \(x, y \in \mathbb{R}^d\) as $$D_h (x, y) \eqdef h(x) - h(y) - \ev{\nabla h(y), x-y}.$$ Note that the inequality \eqref{convexity} defining strong convexity  can be compactly written as $D_f(x,y) \geq \frac{\mu}{2}\|x - y\|^2$.

\subsection{Algorithm \algname{Q-RR}}

The first method we introduce is \algname{Q-RR} (Algorithm~\ref{alg_new_Q_RR}). \algname{Q-RR} is a straightforward combination of distributed random reshuffling and gradient quantization. This method can be seen as the stochastic without-replacement analogue of the distributed quantized gradient method of~\citet{khirirat18_distr_learn_with_compr_gradien}.

\begin{algorithm}[h]
\caption{\algname{Q-RR:} Distributed Random Reshuffling with Quantization}
\label{alg_new_Q_RR}
	\begin{algorithmic}[1]
		\Require $x_0$ -- starting point, $\gamma > 0$ -- stepsize
	    \For{$t =0,1,\dots, T-1$}
		    \State Receive $x_t$ from the server and set $x^0_{t,m} = x_t$
		    \State Sample random permutation of $[n]$: $\pi_m = (\pi^0_m, \dots, \pi^{n-1}_m)$
		    \For{$i = 0, 1,\dots, n-1$}
		    	\For{$m = 1,\dots, M$ in parallel}
		    	\State Receive $x^i_t$ from the server, compute and send $\cQ\left(\nabla f^{\pi^{i}_m}_m(x^{i}_{t})\right)$ back
		    	\EndFor
		        \State Compute and send $x^{i+1}_{t} = x^{i}_{t} - \gamma\frac{1}{M}\sum^M_{m=1} \cQ\left(\nabla f^{\pi^{i}_m}_m(x^{i}_{t})\right)$ to the workers
		    \EndFor
		    \State $x_{t+1} = x^n_t$
    	\EndFor
    	\Ensure $x_T$
	\end{algorithmic}
\end{algorithm}

We shall the use the notion of \emph{shuffling radius} defined by \citet{mishchenko21_proxim_feder_random_reshuf} for the analysis of distributed methods with random reshuffling:

\begin{definition}
Define the iterate sequence \(x^{i+1}_{\star} = x^{i}_{\star} -\frac{\gamma}{M}\sum^M_{m=1}\nabla f^{\pi^i_m}_m(x_{\star})\). Then the shuffling radius is the quantity
$$\sigma^2_{\text{rad}} \eqdef \max_{i}\left\{\frac{1}{\gamma^2 M}\sum^M_{m=1}\EE D_{f^{\pi^i}_m}(x^i_{\star}, x_{\star})\right\}.$$
\end{definition}

We now state the main convergence theorem for Algorithm~\ref{alg_new_Q_RR}:

\begin{theorem}
    \label{th_conv_new_rr_q}
    Let Assumptions~\ref{asm:quantization_operators}, \ref{asm:lip_max_f_m}, \ref{asm:sc_each_f_m} hold and let the stepsize satisfy  $
    	0 < \gamma \leq \frac{1}{\left( 1+2\frac{\omega}{M} \right)L_{\max}}.$
    Then, for all $T \geq 0$ the iterates produced by \algname{Q-RR} (Algorithm~\ref{alg_new_Q_RR}) satisfy
    \begin{eqnarray}
    	\EE\|x_{T}-x_{\star}\|^2 &\leq& \left(1-\gamma\widetilde{\mu}\right)^{nT}\|x_0-x_{\star}\|^2 + \frac{2\gamma^2\sigma^2_{\text{rad}}}{\widetilde{\mu}}\notag\\
    	&&\quad + \frac{2\gamma\omega}{\widetilde{\mu} M}(\zeta_{\star}^2 + \sigma_{\star}^2),\label{eq:QRR_bad_term}
    \end{eqnarray}
    where $\zeta^2_{\star} \eqdef \frac{1}{M}\sum\limits_{m=1}^M\|\nabla f_m (x_{\star})\|^2,$ and $\sigma_{\star}^2 \eqdef \frac{1}{Mn}\sum\limits_{m=1}^M\sum\limits_{i=1}^n \|\nabla f_m^i(x_\star) - \nabla f_m(x_\star)\|^2.$
\end{theorem}

All proofs are relegated to the appendix. By choosing the stepsize \(\gamma\) properly, we can obtain the communication complexity (number of communication rounds) needed to find an \(\varepsilon\)-approximate solution as follows:

\begin{corollary}
	Under the same conditions as Theorem~\ref{th_conv_new_rr_q} and for Algorithm~\ref{alg_new_Q_RR}, there exists a stepsize \(\gamma > 0\) such that the number of communication rounds $nT$ to find a solution with accuracy $\varepsilon > 0$ (i.e. \(\EE\sqn{x_T - x_{*}} \leq \epsilon\)) is equal to
\begin{equation}
	\widetilde{\cO}\left( \left( 1+ \frac{\omega}{M} \right) \frac{L_{\max}}{\widetilde{\mu}} + \frac{\omega(\zeta^2_{\star} + \sigma_{\star}^2)}{M\widetilde{\mu}^2\varepsilon}+ \frac{\sigma_{\text{rad}}}{\sqrt{\widetilde{\mu}^3\varepsilon}}\right),\label{eq:qrr-complexity}
\end{equation}
where $\widetilde{\cO}(\cdot)$ hides constants and logarithmic factors.
\end{corollary}

The complexity of Quantized SGD (\algname{QSGD}) is~\citep{gorbunov2020unified}: $$
	\widetilde{\cO}\left(\left(1 + \frac{\omega}{M}\right)\frac{L_{\max}}{\mu} + \frac{\left(\omega\zeta_{\star}^2 + (1+\omega)\sigma_{\star}^2\right)}{M\mu^2 \varepsilon}\right).$$
For simplicity, let us neglect the differences between $\mu$ and $\widetilde{\mu}$. First, when $\omega = 0$ we recover the complexity of \algname{FedRR} \citep{mishchenko21_proxim_feder_random_reshuf} which is known to be better than the one of \algname{SGD} as long as $\varepsilon$ is sufficently small as we have $\nicefrac{n\mu\sigma_{\star}^2}{8}\leq\sigma_{\text{rad}}^2 \leq \nicefrac{nL\sigma_{\star}^2}{4}$ from \citep{mishchenko21_proxim_feder_random_reshuf}. Next, when $M = 1$ and $\omega = 0$ (single node, no compression) our results recovers the rate of \algname{RR} \citep{mishchenko20_random_reshuf}.

However, it is more interesting to compare \algname{Q-RR} and \algname{QSGD} when $M > 1$ and $\omega > 1$, which is typically the case. In these settings, \algname{Q-RR} and \algname{QSGD} have \emph{the same complexity} since the $\cO(\nicefrac{1}{\varepsilon})$ term dominates the $\cO(\nicefrac{1}{\sqrt{\varepsilon}})$ one if $\varepsilon$ is sufficiently small. That is, the derived result for \algname{Q-RR} has no advantages over the known one for \algname{QSGD} unless $\omega$ is very small, which means that there is almost no compression at all. We also observe this phenomenon in the experiments.

The main reason for that is the variance appearing due to compression. Indeed, even if the current point is the solution of the problem ($x_t^i = x_\ast$), the update direction $-\gamma\frac{1}{M}\sum^M_{m=1} \cQ\left(\nabla f^{\pi^{i}_m}_m(x^{i}_{t})\right)$  has the compression variance
\begin{align*}
	\EE_{\cQ}\left[\left\|\frac{\gamma}{M}\sum^M_{m=1} \left(\cQ(\nabla f^{\pi^{i}_m}_m(x_{*})) - \nabla f^{\pi^{i}_m}_m(x_{*})\right)\right\|^2\right] \leq \frac{\gamma^2\omega}{M^2}\sum^M_{m=1}\| \nabla f^{\pi^{i}_m}_m(x_{*})\|^2.
\end{align*}
This upper bound is tight and non-zero in general. Moreover, it is proportional to $\gamma^2$ that creates the term proportional to $\gamma$ in \eqref{eq:QRR_bad_term} like in the convergence results for \algname{QSGD}/\algname{SGD}, while the \algname{RR}-variance is proportional to $\gamma^2$ in the same bound. Therefore, during the later stages of the convergence \algname{Q-RR} behaves similarly to \algname{QSGD} when we decrease the stepsize.

\subsection{Algorithm \algname{DIANA-RR}}

To reduce the additional variance caused by compression, we apply \algname{DIANA}-style shift sequences \citep{mishchenko19_distr_learn_with_compr_gradien_differ,horvath19_stoch_distr_learn_with_gradien}. Thus we obtain \algname{DIANA-RR} (Algorithm~\ref{alg_new_RR_DIANA}). We notice that unlike \algname{DIANA}, \algname{DIANA-RR} has $n$ shift-vectors on each node.

\begin{algorithm}[h]
\caption{\algname{DIANA-RR}}
\label{alg_new_RR_DIANA}
	\begin{algorithmic}[1]
		\Require $x_0$ -- starting point, $\{h_{0,m}^i\}_{m,i=1,1}^{M,n}$ -- initial shift-vectors, $\gamma > 0$ -- stepsize,  $\alpha > 0$ -- stepsize for learning the shifts
	    \For{$t =0,1,\dots, T-1$}
	    		    \State Receive $x_t$ from the server and set $x^0_{t,m} = x_t$
            \State Sample random permutation of $[n]$: $\pi_m = (\pi^0_m, \dots, \pi^{n-1}_m)$
		    \For{$i = 0, 1,\dots, n-1$}
		        \For{$m = 1, 2,\dots, M$ in parallel}
					\State Receive $x^i_t$ from the server, compute and send $\cQ\left(\nabla f^{\pi^{i}_m}_m(x^{i}_{t}) - h^{\pi^i_m}_{t,m}\right)$ back
    		        \State Set $\hat{g}^{\pi^i_m}_{t,m} = h^{\pi^i_m}_{t,m} + \cQ\left(\nabla f^{\pi^i_m}_m(x^i_{t,m}) - h^{\pi^i_m}_{t,m}\right)$
    		        \State Set $h^{\pi^i_m}_{t+1,m} = h^{\pi^i_m}_{t,m} + \alpha\cQ\left(\nabla f^{\pi^i_m}_m(x^i_{t,m}) - h^{\pi^i_m}_{t,m}\right)$
    		    \EndFor
    		    \State Compute $x^{i+1}_{t} = x^i_{t} - \gamma \frac{1}{M}\sum^M_{m=1}\hat{g}^{\pi^i_m}_{t,m}$ and send $x_t^{i+1}$ to the workers
    		\EndFor
    		\State $x_{t+1} = x^n_t$
    	\EndFor
    	\Ensure $x_T$
	\end{algorithmic}
\end{algorithm}

\begin{theorem}
    \label{th_conv_rr_diana}
    Let Assumptions~\ref{asm:quantization_operators}, \ref{asm:lip_max_f_m}, \ref{asm:sc_each_f_m} hold and suppose that the stepsizes satisfy $0 < \gamma \leq \min\left\{\frac{\alpha}{2n\widetilde\mu}, \frac{1}{\left( 1+\frac{6\omega}{M} \right)L_{\max}}\right\},$ and $\alpha \leq \frac{1}{1+\omega}.$
Define the following Lyapunov function for every $t \geq 0$
\begin{eqnarray}
    \Psi_{t+1} &\eqdef& \|x_{t+1}-x_{\star}\|^2  +\frac{4\omega\gamma^2}{\alpha M^2}\sum^M_{m=1}\sum^{n-1}_{j=0}(1-\gamma\mu)^j\|\Delta_{t+1,m}^j\|^2,\label{lyapunov_func_rr_diana}
\end{eqnarray}
where $\Delta_{t+1,m}^j = h^{\pi^j_m}_{t+1,m} - \nabla f^{\pi^j_m}_m(x_{\star})$
	Then, for all $T \geq 0$ the iterates produced by \algname{DIANA-RR} (Algorithm~\ref{alg_new_RR_DIANA}) satisfy
    \begin{equation*}
        \EE\left[\Psi_T\right] \leq \left(1-\gamma\widetilde{\mu}\right)^{nT}\Psi_0 +\frac{2\gamma^2\sigma^2_{\text{rad}}}{\widetilde\mu}
    \end{equation*}
\end{theorem}
\begin{corollary}
Under the same conditions as Theorem~\ref{th_conv_rr_diana} and for Algorithm~\ref{alg_new_RR_DIANA}, there exists stepsizes \(\gamma, \alpha > 0\) such that the number of communication rounds $nT$ to find a solution with accuracy $\varepsilon > 0$ is
	\begin{equation*}
		\widetilde{\cO}\left(n(1+\omega)+ \br{1 + \frac{\omega}{M}} \frac{L_{\max}}{\widetilde \mu}+\frac{\sigma_{\text{rad}}}{\sqrt{\varepsilon\widetilde{\mu}^3}}\right).
	\end{equation*}
\end{corollary}

Unlike \algname{Q-RR}/\algname{QSGD}/\algname{DIANA}, \algname{DIANA-RR} does not have a $\widetilde\cO(\nicefrac{1}{\varepsilon})$-term, which makes it superior to \algname{Q-RR}/\algname{QSGD}/\algname{DIANA} for small enough $\varepsilon$. However, the complexity of \algname{DIANA-RR} has an additive $\widetilde{\cO}(n(1+\omega))$ term arising due to learning the shifts $\{h_{t,m}^i\}_{m\in [M], i \in [n]}$. Nevertheless, this additional term is not the dominating one when $\varepsilon$ is small enough. Next, we elaborate a bit more on the comparison between \algname{DIANA} and \algname{DIANA-R}. That is, \algname{DIANA}  has $
	\widetilde{\cO}\left(\left(1 + \frac{\omega}{M}\right)\frac{L_{\max}}{\mu} + \frac{(1+\omega)\sigma_{\star}^2}{M\mu^2 \varepsilon}\right)$
complexity \citep{gorbunov2020unified}. Neglecting the differences between $\mu$ and $\widetilde{\mu}$, $L_{\max}$ and $L_{\max}$, we observe a similar relation between \algname{DIANA-RR} and \algname{DIANA} as between \algname{RR} and \algname{SGD}: instead of the term $\cO(\nicefrac{(1+\omega)\sigma_{\star}^2}{(M\mu^2 \varepsilon)})$ appearing in the complexity of \algname{DIANA}, \algname{DIANA-RR} has $\cO(\nicefrac{\sigma_{\text{rad}}}{\sqrt{\varepsilon\widetilde{\mu}^3}})$ term much better depending on $\varepsilon$. To the best of our knowledge, our result is the only known one establishing the theoretical superiority of \algname{RR} to regular \algname{SGD} in the context of distributed learning with gradient compression. Moreover, when $\omega = 0$ (no compression) we recover the rate of \algname{FedRR} and when additionally $M = 1$ (single worker) we recover the rate of \algname{RR}.

\subsection{Algorithm \algname{Q-NASTYA}}

By adding local steps to \algname{Q-RR}, we can do enable each client to do more local work and only communicate once per epoch rather than at each iteration of every epoch. We follow the framework of the \algname{NASTYA} algorithm~\citep{malinovsky22_server_side_steps_sampl_without} and extend it by allowing for quantization, resulting in \algname{Q-NASTYA} (Algorithm~\ref{alg:Q_NASTYA}).

\begin{algorithm}
\caption{\algname{Q-NASTYA}}
\label{alg:Q_NASTYA}
	\begin{algorithmic}[1]
		\Require $x_0$ -- starting point, $\gamma > 0$ -- local stepsize, $\eta > 0$ -- global stepsize
	    \For{$t =0,1,\dots, T-1$}
    		\For{$m \in [M]$ in parallel}
    		    \State Receive $x_t$ from the server and set $x^0_{t,m} = x_t$
    		    \State Sample random permutation of $[n]$: $\pi_m = (\pi^0_m, \dots, \pi^{n-1}_m)$
    		    \For{$i = 0, 1,\dots, n-1$}
    		        \State Set $x^{i+1}_{t,m} = x^{i}_{t,m} - \gamma \nabla f^{\pi^{i}_m}_m(x^{i}_{t,m})$
    		    \EndFor
    		    \State Compute $g_{t,m} = \frac{1}{\gamma n}\left(x_t - x^n_{t,m}\right)$ and send $\cQ_t(g_{t,m})$ to the server
    		\EndFor
    		\State Compute $g_t = \frac{1}{M}\sum_{m=1}^M\mathcal{Q}_t(g_{t,m})$
    		\State Compute $x_{t+1} = x_{t} - \eta g_t$ and send $x_{t+1}$ to the workers
    	\EndFor
    	\Ensure $x_T$
	\end{algorithmic}
\end{algorithm}

\begin{theorem}
\label{thm:convergence_Q_NASTYA}
Let Assumptions~\ref{asm:quantization_operators}, \ref{asm:sc_general_f}, \ref{asm:lip_max_f_m} hold. Let the stepsizes $\gamma$, $\eta$ satisfy $0 < \eta \leq \frac{1}{16L_{\max}\left(1+\frac{\omega}{M}\right)},$ $0 < \gamma \leq \frac{1}{5nL_{\max}}.$
Then, for all $T \geq 0$ the iterates produced by \algname{Q-NASTYA} (Algorithm~\ref{alg:Q_NASTYA}) satisfy
\begin{eqnarray*}
	\mathbb{E}\left[\|x_T-x_{\star}\|^2\right] &\leq& \left(1-\frac{\eta\mu}{2}\right)^T\|x_0 - x_{\star}\|^2 +\frac{9}{2}\frac{\gamma^2nL_{\max}}{\mu}\left((n+1)\zeta^2_{\star} + \sigma_{\star}^2\right) +  8\frac{\eta\omega}{\mu M}\zeta^2_{\star}.
\end{eqnarray*}
\end{theorem}

\begin{corollary}
Under the same conditions as Theorem~\ref{thm:convergence_Q_NASTYA} and for Algorithm~\ref{alg:Q_NASTYA}, there exist stepsizes \(\gamma = \nicefrac{\eta}{n}\) and \(\eta > 0\) such that the number of communication rounds $T$ to find a solution with accuracy $\varepsilon > 0$ is
\begin{equation*}
			\widetilde{\cO}\left(\frac{L_{\max}}{\mu}\left(1+\frac{\omega}{M}\right)+ \frac{\omega}{M}\frac{\zeta^2_{\star}}{\varepsilon\mu^3}+\sqrt{\frac{ L_{\max} }{\varepsilon\mu^3}} \sqrt{\zeta^2_{\star}+\frac{\sigma_{\star}^2}{n}}\right).
\end{equation*}
If $\gamma \rightarrow 0$, one can choose $\eta > 0$ such that the above complexity bound improves to
\begin{equation*}
			\widetilde{\cO}\left( \frac{L_{\max}}{\mu}\left(1+\frac{\omega}{M}\right)+\frac{\omega}{M}\frac{\zeta^2_{\star}}{\varepsilon\mu^3}\right).
		\end{equation*}
\end{corollary}
We emphasize several differences with the known theoretical  results. First, the \algname{FedCOM} method of~\citet{haddadpour20_feder_learn_with_compr} was analyzed in the homogeneous setting only, i.e., $f_m(x) = f(x)$ for all $m \in [M]$, which is an unrealistic assumption for FL applications. In contrast, our result holds in the fully heterogeneous case. Next, the analysis of \algname{FedPAQ}of ~\citet{reisizadeh19_fedpaq} uses a bounded variance assumption, which is also known to be restrictive. Nevertheless, let us compare to their result. \citet{reisizadeh19_fedpaq} derive the following complexity for their method: $$
	\widetilde{\cO}\left( \frac{L_{\max}}{\mu}\left(1+\frac{\omega}{M}\right)+\frac{\omega}{M}\frac{\sigma^2}{\mu^2\varepsilon} + \frac{\sigma^2}{M\mu^2\varepsilon}\right).
$$  This result is inferior to the one we show for \algname{Q-NASTYA}: when $\omega$ is small, the main term in the complexity bound of \algname{FedPAQ} is $\widetilde{\cO}\left(\nicefrac{1}{\varepsilon}\right)$, while for \algname{Q-NASTYA} the dominating term is of the order  $\widetilde{\cO}\left(\nicefrac{1}{\sqrt{\varepsilon}}\right)$ (when $\omega$ and $\varepsilon$ are sufficiently small). 
We also highlight that \algname{FedCRR}~\citep{malinovsky22_feder_random_reshuf_with_compr_varian_reduc} does not converge if $\omega > \nicefrac{M^2 \gamma \mu \varepsilon}{(2\left\|x_{*, m}^{n}\right\|^{2})}$, while \algname{Q-NASTYA} does for any $\omega \geq 0$. Finally, when $\omega = 0$ (no compression) we recover \algname{NASTYA} as a special case, and using $\gamma = \nicefrac{\eta}{n}$, we recover the rate of \algname{FedRR}~\citep{mishchenko21_proxim_feder_random_reshuf}.

\subsection{Algorithm \algname{DIANA-NASTYA}}
As in the case of \algname{Q-RR}, the complexity bound for \algname{Q-NASTYA} includes a $\widetilde\cO(\nicefrac{\omega}{\varepsilon})$ term, appearing due to  quantization noise. To reduce it, we apply \algname{DIANA}-style correction sequences, which leads to  a new method for which we coin the name \algname{DIANA-NASTYA} (Algorithm~\ref{alg:diana-nastya}).

\begin{algorithm}
\caption{\algname{DIANA-NASTYA}}
\label{alg:diana-nastya}
	\begin{algorithmic}[1]
		\Require $x_0$ -- starting point, $\{h_{0,m}\}_{m=1}^M$ -- initial shift-vectors, $\gamma > 0$ -- local stepsize, $\eta > 0$ -- global stepsize, $\alpha > 0$ -- stepsize for learning the shifts
	    \For{$t =0,1,\dots, T-1$}
    	   \For{$m = 1,\dots, M$ in parallel}
    		    \State Receive $x_t$ from the server and set $x^0_{t,m} = x_t$
    		    \State Sample random permutation of $[n]$: $\pi_m = (\pi^0_m, \dots, \pi^{n-1}_m)$
    		    \For{$i = 0, 1,\dots, n-1$}
    		        \State Set $x^{i+1}_{t,m} = x^{i}_{t,m} - \gamma \nabla f^{\pi^{i}_m}_m(x^{i}_{t,m})$
    		    \EndFor
    		    \State Compute $g_{t,m} = \frac{1}{\gamma n}\left(x_t - x^n_{t,m}\right)$ and send $\cQ_t\left(g_{t,m} - h_{t,m}\right)$ to the server
    		    \State Set $h_{t+1,m} = h_{t,m} + \alpha\cQ_t\left(g_{t,m} - h_{t,m}\right)$
    		    \State Set $\hat{g}_{t,m} = h_{t,m} + \cQ_t\left(g_{t,m} - h_{t,m}\right)$
    		\EndFor
			\State $h_{t+1} = \frac{1}{M}\sum^M_{m=1}h_{t+1,m} = h_t + \frac{\alpha}{M}\sum^M_{m=1} \cQ_t\left(g_{t,m} - h_{t,m}\right)$
    		\State $\hat{g}_t = \frac{1}{M}\sum^M_{m=1}\hat{g}_{t,m} = h_t + \frac{1}{M}\sum^M_{m=1}\cQ_t\left(g_{t,m} - h_{t,m}\right)$
    		\State $x_{t+1} = x_{t} - \eta \hat{g}_t$
    	\EndFor
    	\Ensure $x_T$
	\end{algorithmic}
\end{algorithm}

\begin{theorem}
    \label{thm:diana-nastya-conv}
    Let Assumptions~\ref{asm:quantization_operators}, \ref{asm:sc_general_f}, \ref{asm:lip_max_f_m} hold. Suppose the stepsizes $\gamma$, $\eta, \alpha$ satisfy $
        0 < \gamma \leq  \frac{1}{16L_{\max}n}$, $0 < \eta \leq \min\left\{\frac{\alpha}{2\mu}, \frac{1}{16L_{\max}\left(1+\frac{9\omega}{M}\right)}\right\},$ and $\alpha \leq \frac{1}{1+\omega}.$
    Define the following Lyapunov function:
    \begin{equation}
    \label{eq:diana-nastya-lyapunov-function}
    \Psi_{t+1} \eqdef \|x_{t+1}-x_{\star}\|^2 +\frac{8\omega\eta^2}{\alpha M^2}\sum^M_{m=1}\|h_{t+1,m}-h^{\star}_m\|^2.
    \end{equation}
    Then, for all $T \geq 0$ the iterates produced by \algname{DIANA-NASTYA} (Algorithm~\ref{alg:diana-nastya}) satisfy
    \begin{equation}
    \EE\left[\Psi_T\right] \leq \left(1-\frac{\eta\mu}{2}\right)^T\Psi_0 + \frac{9}{2}\frac{\gamma^2 nL}{\mu}\left((n+1)\zeta^2_{\star} + \sigma_{\star}^2\right).
    \end{equation}
\end{theorem}

\begin{corollary}
Under the same conditions as Theorem~\ref{thm:diana-nastya-conv} and for Algorithm~\ref{alg:diana-nastya}, there exist stepsizes \(\gamma = \nicefrac{\eta}{n}\), \(\eta > 0\), \(\alpha > 0\) such that the number of communication rounds $T$ to find a solution with accuracy $\varepsilon > 0$ is
\begin{equation*}
			\widetilde{\cO}\left(\omega + \frac{L_{\max}}{\mu}\left(1+\frac{\omega}{M}\right)+\sqrt{\frac{ L_{\max} }{\varepsilon\mu^3}} \sqrt{\zeta^2_{\star}+\frac{\sigma_{\star}^2}{n}}\right).
\end{equation*}
If $\gamma \rightarrow 0$, one can choose $\eta > 0$ such that the above complexity bound improves to
\begin{equation*}
			\widetilde{\cO}\left(\omega + \frac{L_{\max}}{\mu}\left(1+\frac{\omega}{M}\right)\right).
		\end{equation*}		
\end{corollary}

In contrast to \algname{Q-NASTYA}, \algname{DIANA-NASTYA} does not suffer from the $\widetilde{\cO}(\nicefrac{1}{\varepsilon})$ term in the complexity bound. This shows the superiority of \algname{DIANA-NASTYA} to \algname{Q-NASTYA}. Next, \algname{FedCRR-VR}~\citep{malinovsky22_feder_random_reshuf_with_compr_varian_reduc} has the rate $$
	\widetilde{\mathcal{O}}\left(\frac{(\omega+1)\left(1-\frac{1}{\kappa}\right)^{n}}{\left(1-\left(1-\frac{1}{\kappa}\right)^{n}\right)^{2}}+\frac{\sqrt{\kappa}\left(\zeta_{\star}+ \sigma_{\star}\right)}{\mu \sqrt{\varepsilon}}\right),$$ which depends on $\widetilde{\cO}\left(\nicefrac{1}{\sqrt{\varepsilon}}\right)$. However, the first term is close to $\widetilde{\cO}\left((\omega+1)\kappa^2\right)$ for a large condition number. \algname{FedCRR-VR-2} utilizes variance reduction technique from~\citet{malinovsky2021random} and it allows to get rid of permutation variance. This method has $$\widetilde{\cO}\left(\frac{(\omega+1)\left(1-\frac{1}{\kappa \sqrt{\kappa n}}\right)^{\frac{n}{2}}}{\left(1-\left(1-\frac{1}{\kappa \sqrt{\kappa n}}\right)^{\frac{n}{2}}\right)^{2}}+\frac{\sqrt{\kappa}\zeta_{\star}}{\mu \sqrt{\varepsilon}}\right) $$ complexity, but it requires additional assumption on number of functions $n$ and thus not directly comparable with our result. Note that if we have no compression $(\omega = 0)$, \algname{DIANA-NASTYA} recovers rate of \algname{NASTYA}.

\section{Experiments}\label{sec:experiments}

We tested our methods on solving a logistic regression problem and in training neural networks.

\subsection{Logistic Regression}

To confirm our theoretical results we conducted several numerical experiments on binary classification problem with L2 regularized logistic regression of the form
\begin{eqnarray}\label{eq:log-reg}
\min _{x \in \mathbb{R}^{d}}\left[f(x) \eqdef \frac{1}{M} \sum_{m=1}^{M} \frac{1}{n_m} \sum_{i=1}^{n_m} f_{m,i} \right] \text {, }
\end{eqnarray}
where $f_{m,i}\eqdef \log \left(1+\exp({-y_{mi} a_{mi}^{\top} x})\right)+\lambda \|x\|^2_2$ $(a_{mi},  y_{mi}) \in \mathbb{R}^{d} \times \in\{-1,1\}, i =1,\dots,n_m$ are the training data samples stored on machines $m =1,\dots,M$, and $\lambda>0$ is a regularization parameter. In all experiments, for each method, we used the largest stepsize allowed by its theory multiplied by some individually tuned constant multiplier. For better parallelism, each worker $m$ uses mini-batches of size $\approx 0.1 n_m$. In all algorithms, as a compression operator $\cQ$, we use Rand-$k$ \citep{beznosikov20_biased_compr_distr_learn}  with fixed compression ratio $\nicefrac{k}{d} \approx 0.02$, where $d$ is the number of features in the dataset. 
We provide more details on experimental setups, hardware and datasets in Appendix~\ref{subsec:extra_exp}.

\begin{figure*}[t]
	\centering
	\begin{subfigure}[t]{\textwidth}
		\includegraphics[width=\linewidth]{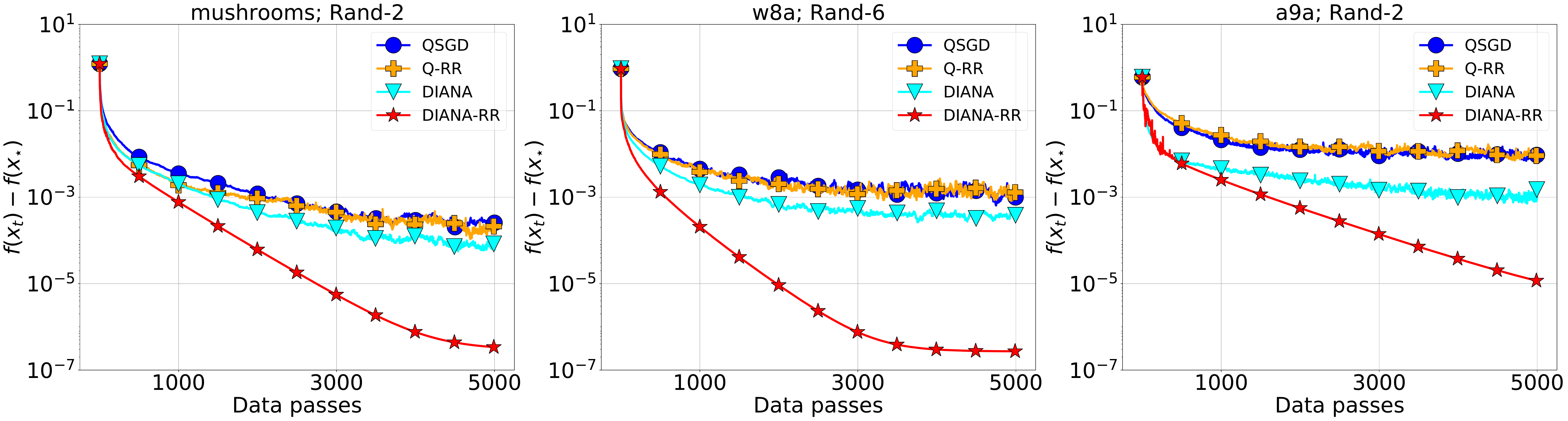} 
		\caption{Non-local methods}\label{fig:non_local}
	\end{subfigure}
	\hfill
	\begin{subfigure}[t]{\textwidth}
		\includegraphics[width=\linewidth]{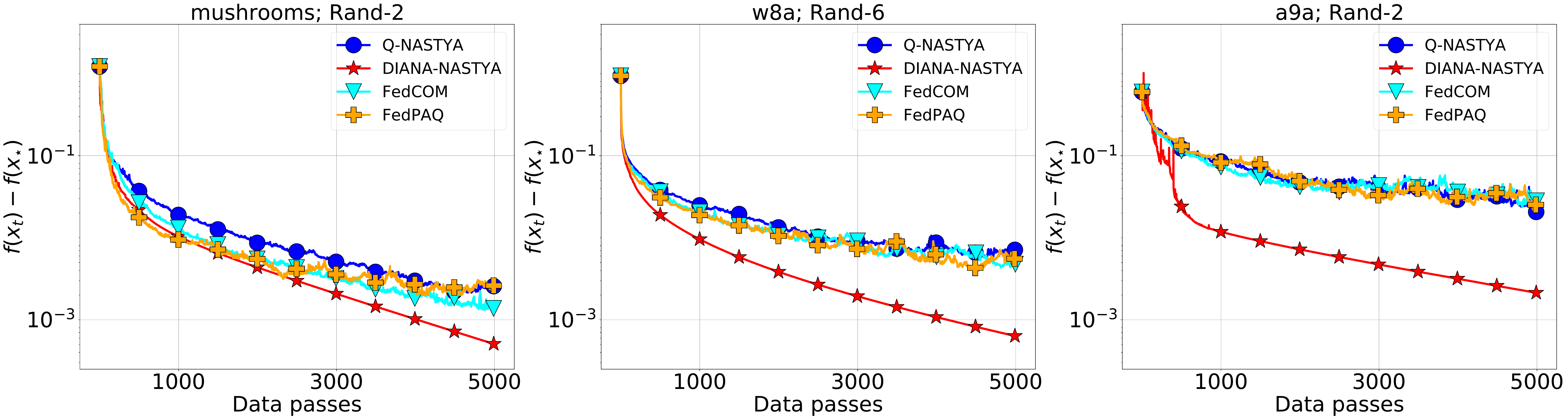} 
		\caption{Local methods}\label{fig:local}
	\end{subfigure}
	\caption{The comparison of the four proposed methods (\algname{Q-NASTYA}, \algname{DIANA-NASTYA}, \algname{Q-RR}, \algname{DIANA-RR}) and existing baselines (\algname{FedCOM}, \algname{FedPAQ}) with tuned stepsizes and Rand-$k$ compressor.}\label{fig:tuned experimets}
\end{figure*}

\paragraph{Experiment 1: Comparison of the proposed non-local methods with existing baselines.} 
In our first experiment (see Figure~\ref{fig:non_local}), we compare \algname{Q-RR} and \algname{DIANA-RR} with corresponding classical baselines (\algname{QSGD} \citep{alistarh17_qsgd}, \algname{DIANA} \citep{mishchenko19_distr_learn_with_compr_gradien_differ}) that use a with-replacement mini-batch \algname{SGD} estimator. Figure~\ref{fig:non_local} illustrates that \algname{Q-RR} experiences similar behavior as \algname{QSGD} both losing in speed to \algname{DIANA} method in all considered datasets. However, \algname{DIANA-RR} shows the best rate among all considered non-local methods, efficiently reducing the variance, and achieving the lowest functional sub-optimality tolerance. The results observed in numerical experiments are in perfect correspondence with the derived theory.

\paragraph{Experiment 2: Comparison of the proposed local methods with existing baselines.}
The second experiment shows that \algname{DIANA}-based method can significantly outperform in practice when one applies it to local methods as well. In particular, whereas \algname{Q-NASTYA} shows comparative behavior as existing methods \algname{FedCOM} \citep{haddadpour20_feder_learn_with_compr}, \algname{FedPAQ} \citep{reisizadeh19_fedpaq} in all considered datasets, \algname{DIANA-NASTYA} noticeably outperforms other methods.

\subsection{Training \texttt{ResNet-18} on \texttt{CIFAR10}}
Since random reshuffling is a very popular technique in training neural networks, it is natural to test the proposed methods on such problems. Therefore, in the second set of experiments, we consider training \texttt{ResNet-18}~\citep{resnet18_56} model on the \texttt{CIFAR10} dataset \cite{krizhevsky2009learning}. To conduct these experiments we use \texttt{FL\_PyTorch} simulator \citep{burlachenko2021fl_pytorch}. Further technical details are deferred to Appendix~\ref{subsec:extra_exp}.

\paragraph{Experiment 3: Comparison of the proposed non-local methods in training \texttt{ResNet-18} on \texttt{CIFAR10}.} The main goal of this experiment is to verify the phenomenon observed in Experiment 1 on the training of a deep neural network. That is, we tested \algname{Q-RR}, \algname{QSGD}, \algname{DIANA}, and \algname{DIANA-RR} in the distributed training of \texttt{ResNet-18} on \texttt{CIFAR10}, see Figure~\ref{fig:NN_plots_main}. As in the logistic regression experiments, we observe that (i) \algname{Q-RR} and \algname{QSGD} behave similarly and (ii) \algname{DIANA-RR} outperforms \algname{DIANA}.

\begin{figure}[h]
	\centering
	\captionsetup[sub]{font=small,labelfont={}}	
	\begin{subfigure}[ht]{0.49\columnwidth}
		\includegraphics[width=\textwidth]{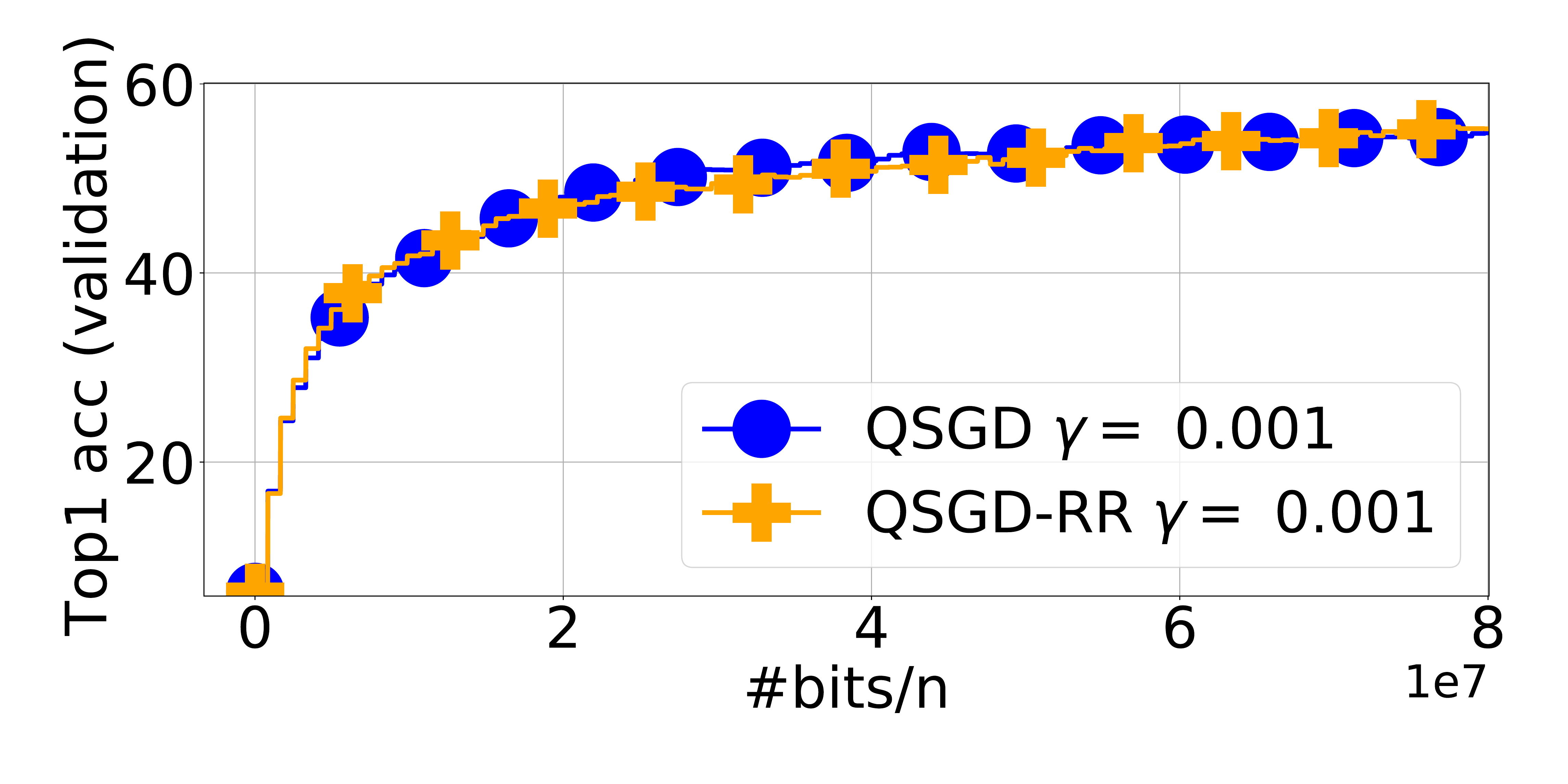} \caption{\algname{Q-RR} vs \algname{QSGD}}
	\end{subfigure}
	\begin{subfigure}[ht]{0.49\columnwidth}
		\includegraphics[width=\textwidth]{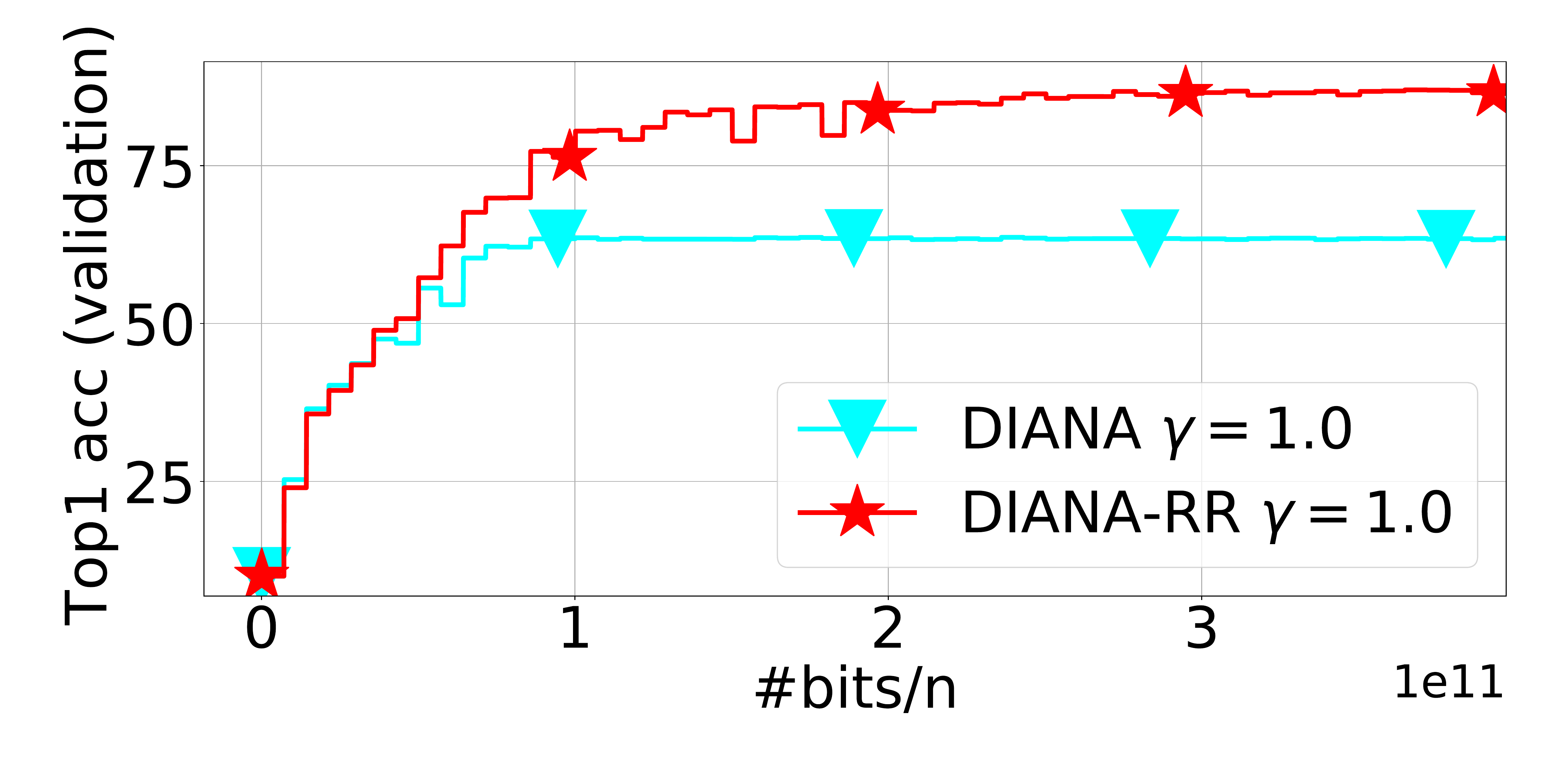} \caption{\algname{DIANA-RR} vs \algname{DIANA}}
	\end{subfigure}
	\caption{The comparison of \algname{Q-RR}, \algname{QSGD}, \algname{DIANA}, and \algname{DIANA-RR} on the task of training \texttt{ResNet-18} on \texttt{CIFAR-10} with $n=10$ workers. Top-1 accuracy on test set is reported. Stepsizes were tuned and workers used Rand-$k$ compressor with $\nicefrac{k}{d} \approx 0.05$. Further details and additional experiments are provided in Appendix~\ref{subsec:extra_exp}.}
	\label{fig:NN_plots_main}
\end{figure}

\section*{Acknowledgements}
The work of A.~Sadiev and E.~Gorbunov was partially supported by a grant for research centers in the field of artificial intelligence, provided by the Analytical Center for the Government of the Russian Federation in accordance with the subsidy agreement (agreement identifier 000000D730321P5Q0002) and the agreement with the Moscow Institute of Physics and Technology dated November 1, 2021 No. 70-2021-00138.

\bibliography{main}
\clearpage

\appendix

\tableofcontents

\clearpage
\section{Experiments: Missing Details and Extra Results}\label{subsec:extra_exp}

In this section, we provide missing details on the experimental setting from Section \ref{sec:experiments}.  The codes are provided in the following Github repository: \url{https://github.com/IgorSokoloff/rr_with_compression_experiments_source_code}.

\subsection{Logistic Regression}

\paragraph{Hardware and Software.}
All algorithms were written in Python 3.8. We used three different CPU cluster node types:
\begin{enumerate}
	\item AMD EPYC 7702 64-Core; 
	\item Intel(R) Xeon(R) Gold 6148 CPU @ 2.40GHz;
	\item Intel(R) Xeon(R) Gold 6248 CPU @ 2.50GHz.
\end{enumerate}
\paragraph{Datasets.}
The datasets were taken from open LibSVM library \citet{chang2011libsvm}, sorted in ascending order of labels, and equally split among 20 machines $\backslash$clients$\backslash$workers. The remaining part of size $N - 20\cdot \left \lfloor{\nicefrac{N}{20}}\right \rfloor $ was assigned to the last worker,  where $N = \sum_{m=1}^{M}n_m$ is the total size of the dataset. 
A summary of the splitting and the data samples distribution between clients can be found in Tables \ref{tbl:datasets_summary}, \ref{tbl:mushrooms_partition}, \ref{tbl:w8a_partition}, \ref{tbl:a9a_partition}.
\begin{table}[h!]
	\caption{Summary of the datasets and splitting of the data samples among clients.}
	\label{tbl:datasets_summary}
	\centering
	\begin{tabular}{l l l l l }
		\toprule
		Dataset  & $M$ & $N$ (dataset size) & $d$ (\# of features)  & $n_m$ (\# of datasamples per client) \\
		\midrule			
		\texttt{mushrooms} & $20$ & $8120$ & $112$  & $406$\\ 
		\texttt{w8a} & $20$  &$49749$ & $300$ &  $2487$\\
		\texttt{a9a} & $20$ &$32560$ & $123$  &$1628$ \\ 
		\bottomrule
	\end{tabular}
\end{table}
\begin{table}[h!]
	\caption{Partition of the \texttt{mushrooms} dataset among clients.}
	\label{tbl:mushrooms_partition}
	\centering
	\begin{tabular}{l l l }
		\toprule
		Client's \textnumero &\# of datasamples of class "-1" & \# of datasamples of class "+1"  \\
		\midrule			
		$1$ -- $9$& $406$& $0$\\
		$10$& $262$& $144$\\
		$11$ -- $19$& $0$& $406$\\
		$20$& $0$& $410$\\
		\bottomrule
	\end{tabular}
\end{table}
\begin{table}[h!]
	\caption{Partition of the \texttt{w8a} dataset among clients.}
	\label{tbl:w8a_partition}
	\centering
	\begin{tabular}{l l l }
		\toprule
		Client's \textnumero &\# of datasamples of class "-1" & \# of datasamples of class "+1"  \\
		\midrule			
		$1$ -- $19$& $2487$& $0$\\
		$20$& $1017$& $1479$\\
		\bottomrule
	\end{tabular}
\end{table}
\begin{table}[h!]
	\caption{Partition of the \texttt{a9a} dataset among clients.}
	\label{tbl:a9a_partition}
	\centering
	\begin{tabular}{l l l }
		\toprule
		Client's \textnumero &\# of datasamples of class "-1" & \# of datasamples of class "+1"  \\
		\midrule			
		$1$ -- $14$& $1628$& $0$\\
		$15$& $1328$& $300$\\
		$16$ -- $19$& $0$& $1628$\\
		$20$& $0$& $1629$\\
		\bottomrule
	\end{tabular}
\end{table}
\paragraph{Hyperparameters.}
Regularization parameter $\lambda$ was chosen individually for each dataset to guarantee the condition number $\nicefrac{L}{\mu}$ to be approximately $10^4$,  where $L$ and $\mu$ are the smoothness and strong-convexity constants of function $f$. For the chosen logistic regression problem of the form \eqref{eq:log-reg}, smoothness and strong convexity constants $L$, $L_m$, $L_{i,m}$, $\mu$, $ \widetilde{\mu}$ of functions  $f$, $f_m$ and $f_m^i$ were computed explicitly as 

\begin{eqnarray*}
L &=& \lambda_{max}\rb{\frac{1}{M}\sum_{m=1}^{M}\frac{1}{4n_m}\mA_m^{\top}\mA_m + 2\lambda \mI} \\
L_m &=& \lambda_{max}\rb{\frac{1}{4n_m}\mA_m^{\top}\mA_m + 2\lambda \mI}  \\L_{i,m} &=& \lambda_{\max}\rb{ \frac{1}{4} a_{mi}a_{mi}^{\top} + 2\lambda \mI} \\
 \mu &=& 2\lambda \\
 \widetilde{\mu}  &=& 2\lambda, 
\end{eqnarray*}  
where $\mA_m $ is the dataset associated with client $m$, and $a_{mi}$ is the $i$-th row of data matrix $\mA_m$. In general, the fact that $f$ is $L$-smooth with $$L \le \frac{1}{M}\sum_{m=1}^{M}\frac{1}{n_m}\sum_{i=1}^{n_m} L_{i,m}$$ follows from the $L_{i,m}$-smoothness of $f_m^i$ (see Assumption \ref{asm:lip_max_f_m}). 

In all algorithms, as a compression operator $\cQ$, we use Rand-$k$  as a canonical example of unbiased compressor with relatively bounded variance, and fix the compression parameter $k = \lfloor 0.02 d\rfloor$, where $d$ is the number of features in the dataset. 

In addition, in all algorithms, for all clients $m=1,\dots,M$, we set the batch size for the \algname{SGD} estimator to be $b_m = \lfloor 0.1 n_m\rfloor$, where $n_m$ is the size of the local dataset. 

The summary of the values $L$, $L_m$, $L_{i,m}$ $L_{\max}$, $\mu$, $b_m$ and $k$ for each dataset can be found in Table \ref{tbl:hyperparameters_per_dataset_summary}.

\begin{table}[h!]
	\caption{Summary of the hyperparameters.}
	\label{tbl:hyperparameters_per_dataset_summary}
	\centering
	\begin{tabular}{l l l l l l l}
		\toprule
		Dataset  & $L$ & $L_{\max}$& $\mu$ & $\lambda$ & $k$&$b_m$ (batchsize)\\
		\midrule			
		\texttt{mushrooms} & $2.59$ & $5.25$ & $2.58\cdot10^{-4}$ & $1.29\cdot10^{-4}$ & $2$&$40$\\ 
		\texttt{w8a} & $0.66$ &$28.5$  & $6.6\cdot10^{-5}$ & $3.3\cdot10^{-5}$ &$6$ & $248$\\
		\texttt{a9a} & $1.57$  & $3.5$ &  $1.57\cdot10^{-4}$ & $7.85\cdot10^{-5}$& $2$& $162$\\ 
		\bottomrule
	\end{tabular}
\end{table}

In all experiments, we follow constant stepsize strategy within the whole iteration procedure. For each method, we set the largest possible stepsize predicted by its theory multiplied by some individually tuned constant multiplier. 
For a more detailed explanation of the tuning routine, see Sections \ref{subsec:non_local_methods} and \ref{subsec:local_methods}.

\paragraph{SGD implementation.} 
We considered two approaches to minibatching: random reshuffling and with-replacement sampling. In the first, all clients $m=1,\dots, M$ independently permute their local datasets and pass through them within the next subsequent  $\lfloor \frac{n_m}{b_m} \rfloor$ steps. In our implementations of \algname{Q-RR}, \algname{Q-NASTYA} and \algname{DIANA-NASTYA}, all clients permuted their datasets in the beginning of every new epoch, whereas for the \algname{DIANA-RR} method they do so only once in the beginning of the iteration procedure. Second approach of minibatching is called with-replacement sampling, and it requires every client to draw $b_m$ 
data samples from the local dataset uniformly at random. We used this strategy in the baseline algorithms (\algname{QSGD}, \algname{DIANA}, \algname{FedCOM} and \algname{FedPAQ}) we compared our proposed methods to.

\paragraph{Experimental setup.}
To compare the performance of methods within the whole optimization process, we track the functional suboptimality metric $f(x_t) - f(x_{\star})$ that was recomputed after each epoch. For each dataset, the value $f(x_{\star})$ was computed once at the preprocessing stage with $10^{-16}$ tolerance via conjugate gradient method.
We terminate our algorithms after performing $5000$ epochs.

\subsubsection{Experiment 1: Comparison of the Proposed Non-Local Methods with Existing Baselines (Extra Details)}\label{subsec:non_local_methods} 
For each of the considered non-local methods, we take the stepsize as the largest one predicted by the theory premultiplied by the individually tuned constant factor from the set 
$\left\{0.000975, 0.00195, 0.0039,	0.0078,	0.0156,	0.0312,	0.0625,	0.125,	0.25,	0.5, 1,	2,	4,	8,	16,	32,	64,\right.$\\
$\left.128, 256, 512,1024, 2048, 4096\right\}.$

Therefore, for each local method on every dataset, we performed $20$ launches to find the stepsize multiplier showing the best convergence behavior (the fastest reaching the lowest possible level of functional suboptimality $f(x_t) - f(x_{\star})$).

Theoretical stepsizes for methods \algname{Q-RR} and \algname{DIANA-RR} are provided by the Theorems \ref{th_conv_new_rr_q} and \ref{th_conv_rr_diana}, whereas stepsizes for \algname{QSGD} and \algname{DIANA} were taken from the paper \cite{gorbunov2020unified}.

\subsubsection{Experiment 2: Comparison of the Proposed Local Methods with Existing Baselines (Extra Details)}\label{subsec:local_methods}

In this set of experiments, we tuned stepsizes similarly to the non-local methods. However, for algorithms \algname{Q-NASTYA}, \algname{DIANA-NASTYA}, and \algname{FedCOM} we needed to independently adjust the client and server stepsizes, leading to a more extensive tunning routine. 

As before, for each local method on every dataset, tuned client and server stepsizes are defined by the theoretical one and adjusted constant multiplier.
Theoretical stepsizes for methods \algname{Q-NASTYA} and \algname{DIANA-NASTYA} are given by the Theorems \ref{thm:convergence_Q_NASTYA} and \ref{thm:diana-nastya-conv}, whereas \algname{FedCOM} and \algname{FedPAQ} stepsizes  were taken from the papers by \citet{haddadpour20_feder_learn_with_compr} and \citet{reisizadeh19_fedpaq} respectively.
We now list all the considered multipliers of client and server stepsizes for every method (i.e. $\gamma$ and $\eta$ respectively):
\begin{itemize}
	\item \algname{Q-NASTYA}:
	\begin{itemize}
		\item
		Multipliers for $\gamma:$
		$\{$0.000975, 0.00195, 0.0039,	0.0078,	0.0156,	0.0312,	0.0625,	0.125,	0.25, 0.5, 1,	2,	4,	8,	16,	32,	64,128$\}$;
		\item
		Multipliers for $\eta:$
		$\{$0.0039,	0.0078,	0.0156,	0.0312,	0.0625,	0.125,	0.25, 0.5, 1,	2,	4,	8,	16,	32,	64, 128$\}$. 
	\end{itemize}
	\item \algname{DIANA-NASTYA}:
	\begin{itemize}
		\item Multipliers for $\gamma$ and $\eta:$
		$\{$0.000975, 0.00195, 0.0039,	0.0078,	0.0156,	0.0312,	0.0625,	0.125,	0.25, 0.5, 1,	2,	4,	8,	16,	32,	64,128$\}$;
	\end{itemize}
	\item \algname{FedCOM}:
	\begin{itemize}
		\item
		Multipliers for $\gamma:$
		$\{$0.0312,	0.0625,	0.125,	0.25, 0.5, 1,	2,	4,	8,	16,	32,	64, 128, 256, 512,1024, 2048, 4096, 8192, 16384, 32768 $\}$;
		\item
		Multipliers for $\eta:$
		$\{$0.000975, 0.00195, 0.0039,	0.0078,	0.0156,	0.0312,	0.0625,	0.125,	0.25, 0.5, 1,	2,	4,	8,	16,	32,	64, 128$\}$.
	\end{itemize}
	\item \algname{FedPAQ}:
	\begin{itemize}
		\item 
		Multipliers for $\gamma:$
		$\{$0.00195, 0.0039, 0.0078, 0.0156, 0.0312, 0.0625, 0.125, 0.25, 0.5, 1, 2, 4, 8, 16, 32, 64, 128, 256, 512, 1024, 2048, 4096, 8192, 16384, 32768, 65536, 131072, 262144, 524288, 1048576 $\}$.
	\end{itemize}
\end{itemize}

For example, to find the best pair $(\gamma, \eta)$ for \algname{FedCOM} method on each dataset, we performed $378$ launches. A similar subroutine was executed for all algorithms on all datasets independently.

\subsection{Training Deep Neural Network model: ResNet-18 on CIFAR-10}
To illustrate the behavior of the proposed methods in training Deep Neural Networks (DNN), we consider the \texttt{ResNet-18}~\citep{resnet18_56} model. This model is used for image classification, feature extraction for image segmentation, object detection, image embedding, and image captioning. We train all layers of \texttt{ResNet-18} model meaning that the dimension of the optimization problem equals $d=11,173,962$. During the training, the \texttt{ResNet-18} model normalizes layer inputs via exploiting $20$ Batch Normalization ~\citep{ioffe2015batch} layers that are applied directly before nonlinearity in the computation graph of this model. Batch normalization (BN) layers add $9600$ trainable parameters to the model. Besides trainable parameters, a BN layer has its internal state that is used for computing the running mean and variance of inputs due to its own specific regime of working. 
We use \textit{He} initialization~\citep{he2015delving}.

\subsubsection{Computing Environment} 
We performed numerical experiments on a server-grade machine running Ubuntu 18.04 and Linux Kernel v5.4.0, equipped with 16-cores (2 sockets by 16 cores per socket) 3.3 GHz Intel Xeon, and four NVIDIA A100 GPU with 40GB of GPU memory. The distributed environment is simulated in Python 3.9 via using the software suite \texttt{FL\_PyTorch} ~\citep{burlachenko2021fl_pytorch} that serves for carrying complex Federate Learning experiments. \texttt{FL\_PyTorch} allowed us to simulate the distributed environment in the local machine. Besides storing trainable parameters per client, this simulator stores all not trainable parameters including BN statistics per client.

\subsubsection{Loss Function} 
Training of \texttt{ResNet-18} can be formalized as problem \eqref{eq:erm-opt-orig} with the following choice of $f_m^i$ 
\begin{equation}
	f_m(x)=\frac{1}{|n_{m}|}\sum_{j=1}^{|n_{m}|} CE(b^{(j)}, g(a^{(j)},x)),
\end{equation}
where $CE(p, q) \eqdef -\sum_{k=1}^{\mathrm{\#classes}} p_i \cdot \log(q_i)$ with agreement $0\cdot\log(0) = 0$ is a standard cross-entropy loss, function $g: \mathbb{R}^{28 \times 28} \times \mathbb{R}^d \to [0,1]^{\mathrm{\#classes}}$ is a neural network taking image $a^{(j)}$ and vector of parameters $x$ as an input and returning a vector in probability simplex, and $n_m$ is the size of the dataset on worker $m$. 

\subsubsection{Dataset and Metric} 
In our experiments, we used \texttt{CIFAR10} dataset \cite{krizhevsky2009learning}. The dataset consists of input variables $a_i \in \mathbb{R}^{28 \times 28 \times 3}$, and response variables $b_i \in \{0,1\}^{10}$ and is used for training 10-way classification. The sizes of training and validation set are $5\times 10^4$ and $10^4$ respectively. The training set is partitioned heterogeneously across $10$ clients. 
To measure the performance, we evaluate the loss function value $f(x)$, norm of the gradient $\|\nabla f(x)\|_2$ and the Top-1 accuracy of the obtained model as a function of passed epochs and the normalized number of bits sent from clients to the server.

\subsubsection{Tuning Process}  
In this set of experiments, we tested \algname{QSGD}~\citep{alistarh17_qsgd}, \algname{Q-RR} (Algorithm \ref{alg_new_Q_RR}), \algname{DIANA}~\citep{mishchenko2019distributed} and \algname{DIANA-RR} (Algorithm \ref{alg_new_RR_DIANA}) algorithms. For all algorithms, we tuned the strategy $\in \{A, B, C\}$ of decaying stepsize model via selecting the best in terms of the norm of the full gradient on the train set in the final iterate produced after $20 000$ rounds. The stepsize policies are described below.
\begin{enumerate}[label=\Alph*.]
	\item Stepsizes decaying as inverse square root of the number epochs
	\[
	\gamma_e = \begin{cases}
		\gamma_{init}  \cdot \dfrac{1}{\sqrt{e - s + 1}}, &\text{if } e \ge s,\\
		\gamma_{init}, &\text{if } e < s,
	\end{cases}
	\]
	where $\gamma_e$ denotes the stepsize used during epoch $e+1$, $s$ is a fixed shift.
	\item Stepsizes decaying as inverse of number epochs
	\[
	\gamma=\begin{cases}
		\gamma_{init}  \cdot \dfrac{1}{{e - s + 1}}, &\text{if } e \ge s,\\
		\gamma_{init}, &\text{if } e < s.
	\end{cases}
	\]
	\item Fixed stepsize
	\[
	\gamma=\gamma_{init}.	
	\]
	
\end{enumerate}

We say that the algorithm passed $e$ epochs if the total number of computed gradient oracles lies between $e \sum_{m=1}^M n_m$ and $(e+1) \sum_{m=1}^M n_m$. For each algorithm the used stepsize $\gamma_{init}$ and shift parameter $s$ were tuned via selecting from the following sets:
\begin{eqnarray*}
	\gamma_{init} \in \gamma_{set} \eqdef \{4.0, 3.75, 3.00, 2.5, 2.00, 1.25, 1.0, 0.75, 0.5, 0.25,\\
	0.2, 0.1, 0.06, 0.03, 0.01, 0.003, 0.001, 0.0006\}.
\end{eqnarray*}
\begin{eqnarray*}
	s \in s_{set} \eqdef \{50, 100, 200, 500, 1000\}.
\end{eqnarray*}

In all tested methods, clients independently apply Rand-$k$  compression with carnality $k= \lfloor 0.05 d \rfloor$. Computation for all gradient oracles is carried out in single precision float (fp32) arithmetic.

\begin{figure}[H]
	\centering
	\captionsetup[sub]{font=small,labelfont={}}	
	\begin{subfigure}[ht]{0.32\textwidth}
		\includegraphics[width=\textwidth]{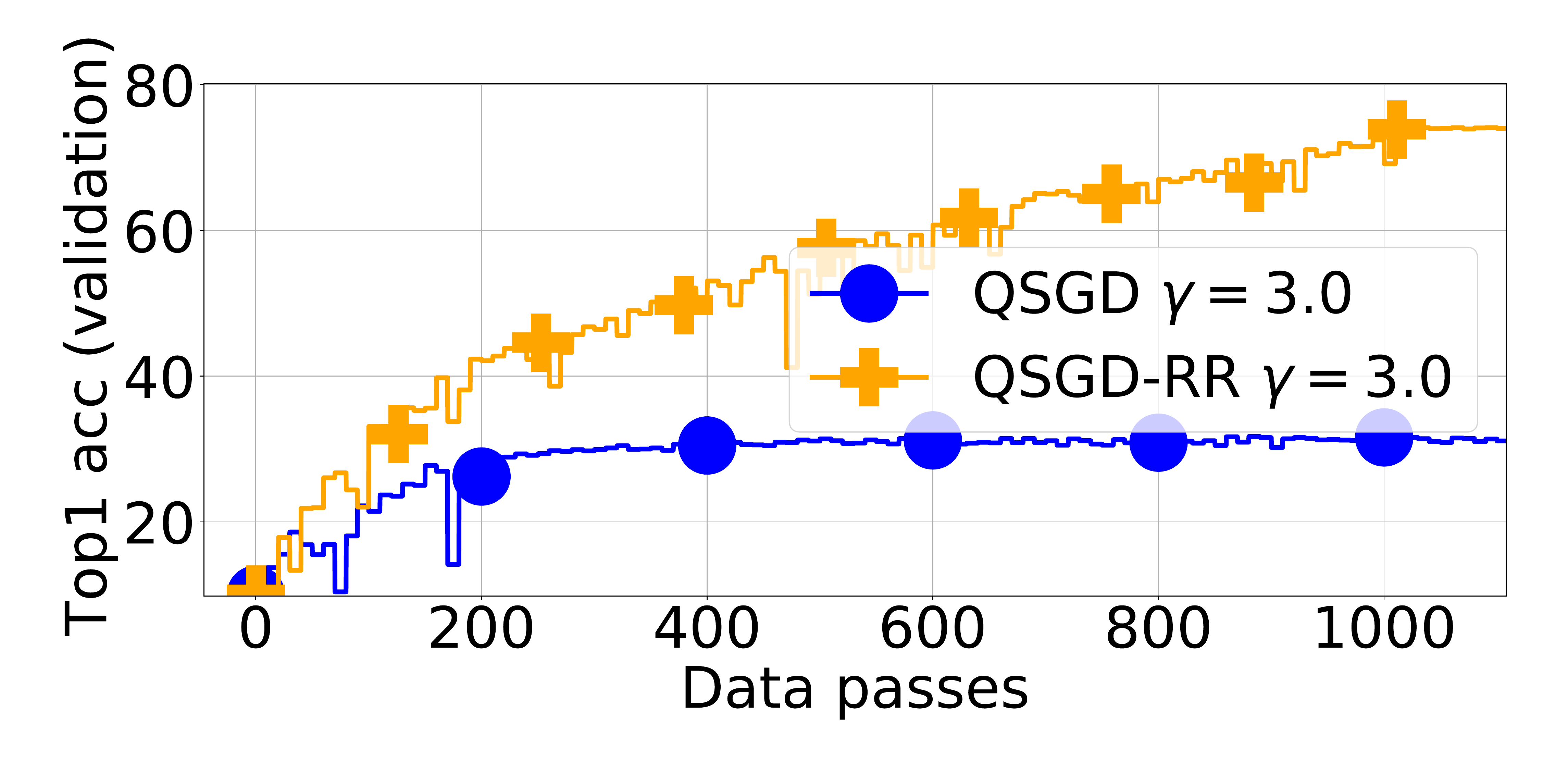} \caption{}\label{fig:train_resnet18_qsgd_a}
	\end{subfigure}
	\begin{subfigure}[ht]{0.32\textwidth}
		\includegraphics[width=\textwidth]{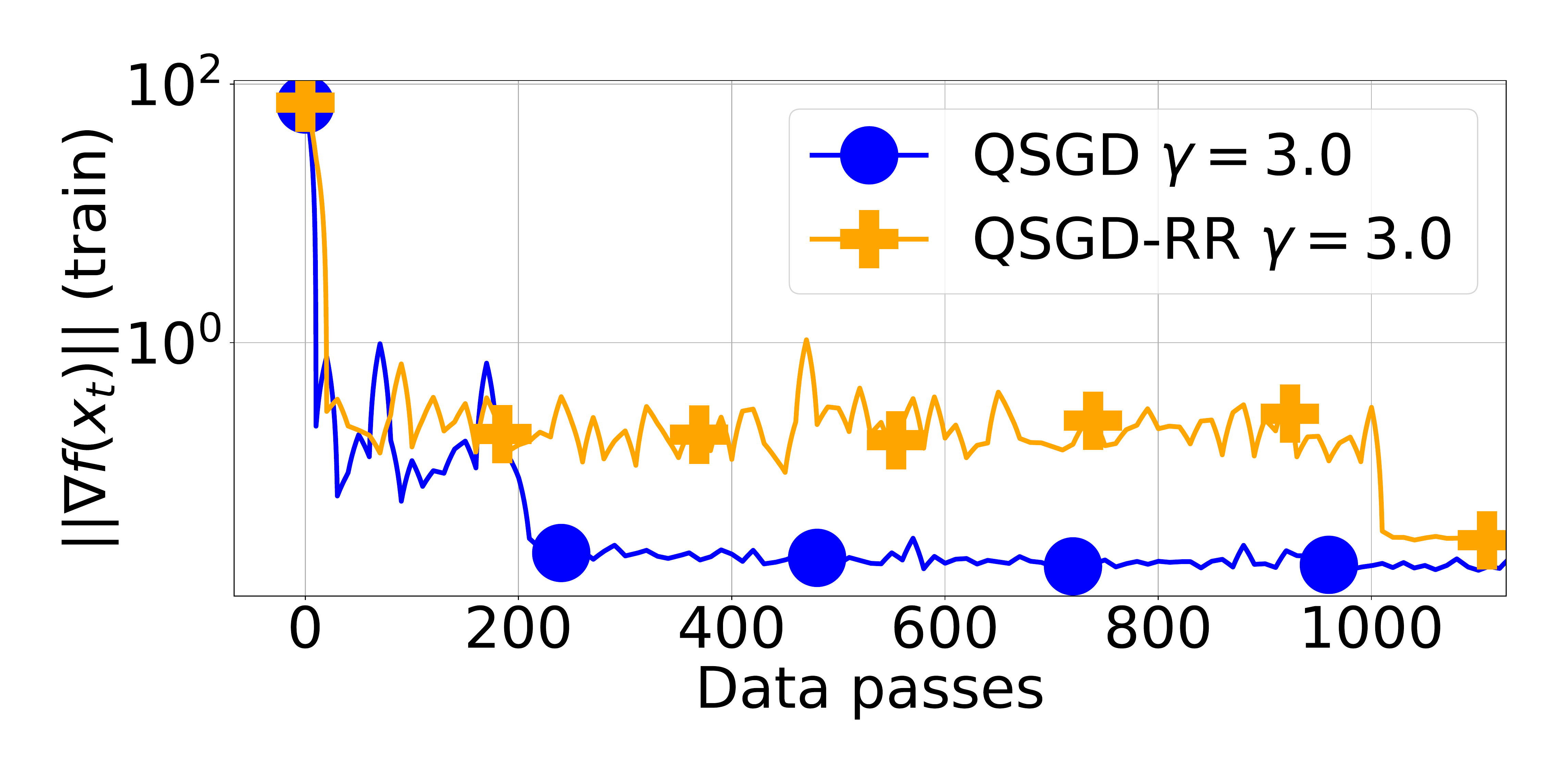} \caption{}\label{fig:train_resnet18_qsgd_b}
	\end{subfigure}
	\begin{subfigure}[ht]{0.32\textwidth}
		\includegraphics[width=\textwidth]{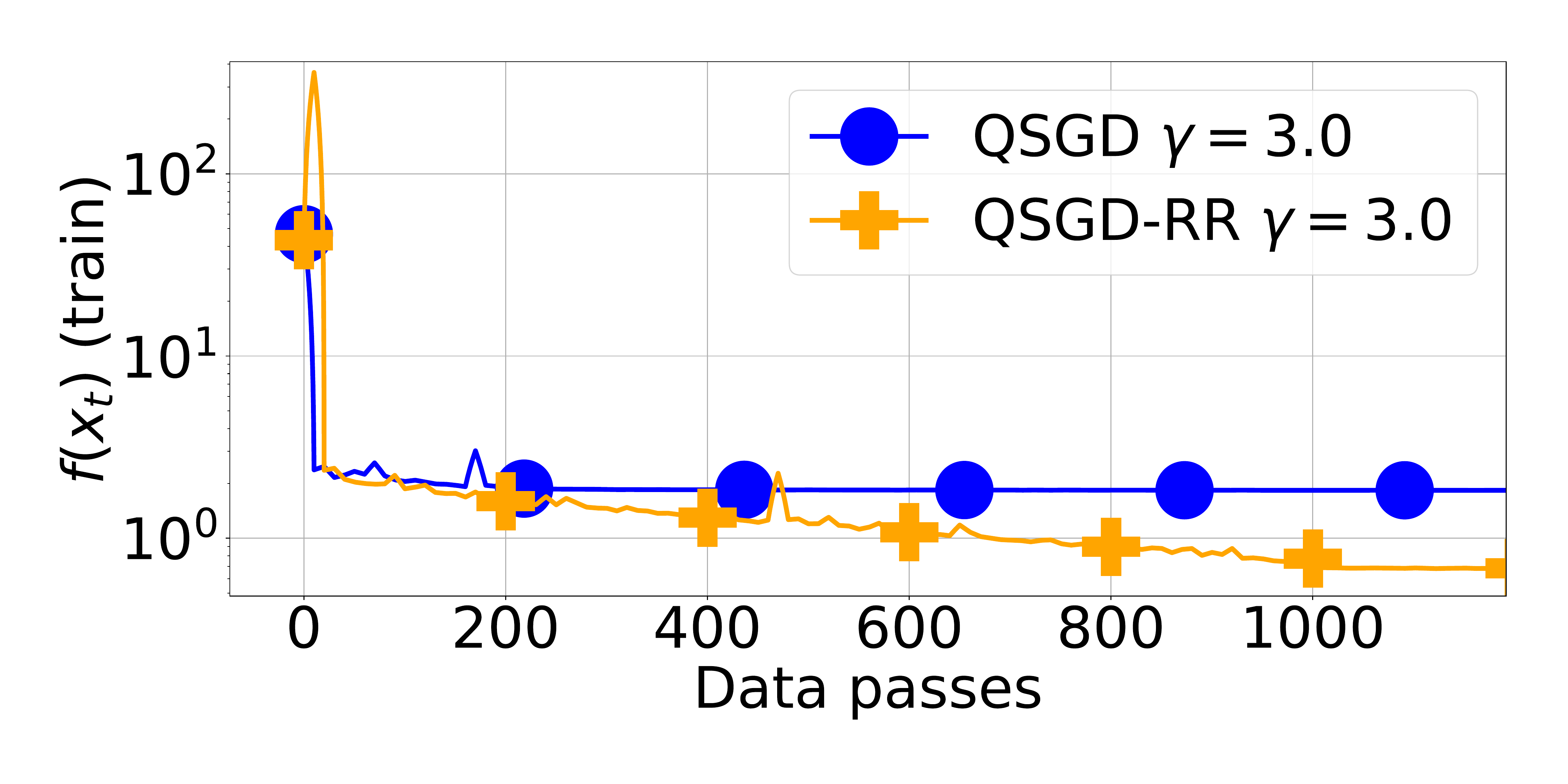} \caption{}\label{fig:train_resnet18_qsgd_c}
	\end{subfigure}
	
	\begin{subfigure}[ht]{0.32\textwidth}
		\includegraphics[width=\textwidth]{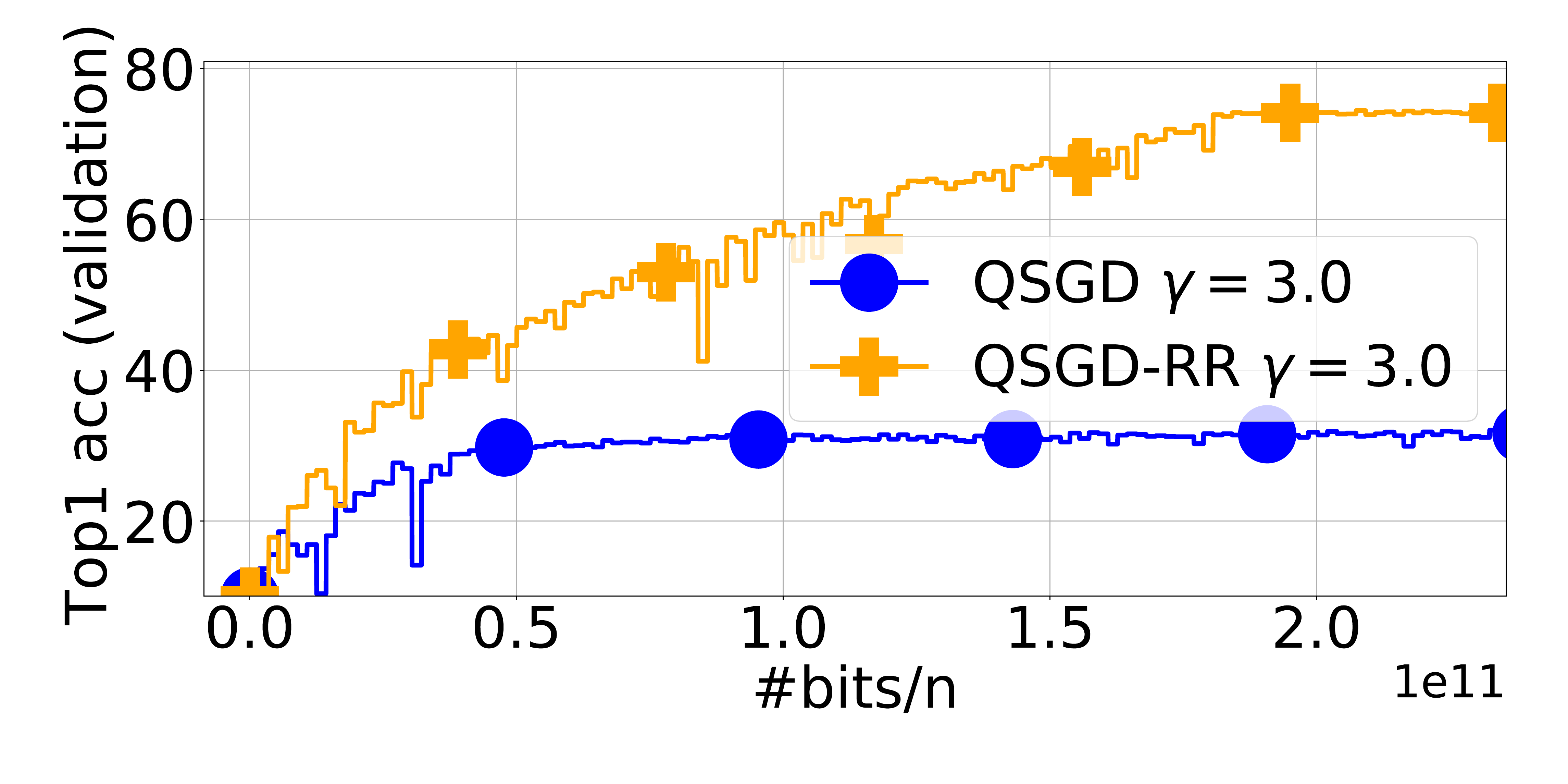} \caption{}\label{fig:train_resnet18_qsgd_d}
	\end{subfigure}
	\begin{subfigure}[ht]{0.32\textwidth}
		\includegraphics[width=\textwidth]{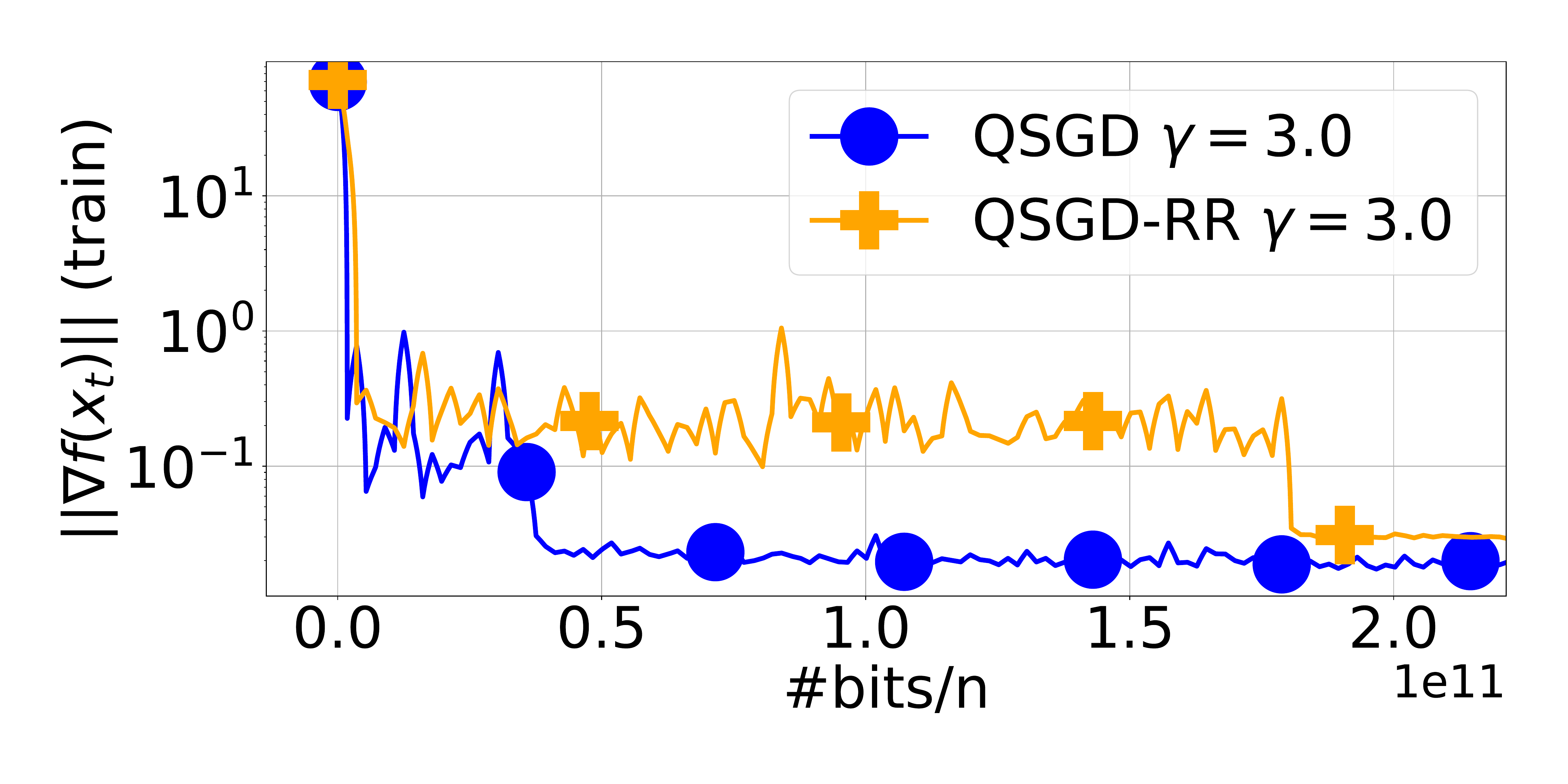}\caption{}\label{fig:train_resnet18_qsgd_e}
	\end{subfigure}
	\begin{subfigure}[ht]{0.32\textwidth}
		\includegraphics[width=\textwidth]{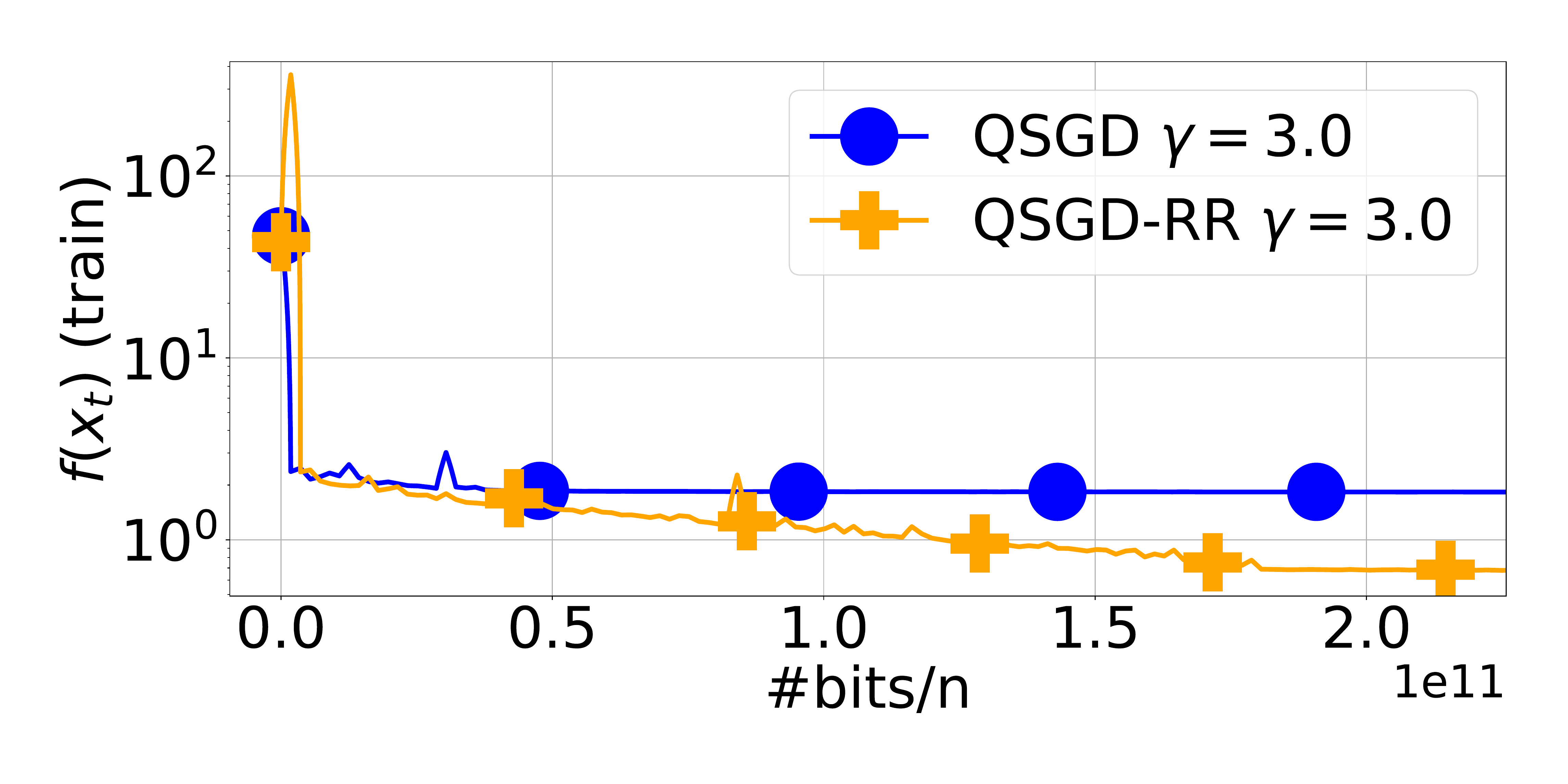}\caption{}\label{fig:train_resnet18_qsgd_f}
	\end{subfigure}
	\caption{{Comparison of \algname{QSGD} and \algname{Q-RR} in the training of \texttt{ResNet-18} on \texttt{CIFAR-10}, with $n=10$ workers. Here (a) and (d) show Top-1 accuracy on test set, (b) and (e) -- norm of full gradient on the train set, (c) and (f) -- loss function value on the train set. Stepsizes and decay shift has been tuned from $s_{set}$ and $\gamma_{set}$ based on minimum achievable value of loss function on the train set.}}
	\label{fig:train_resnet18_qsgd}
\end{figure}

\begin{figure}[H]
	\centering
	\captionsetup[sub]{font=small,labelfont={}}	
	\begin{subfigure}[ht]{0.32\textwidth}
		\includegraphics[width=\textwidth]{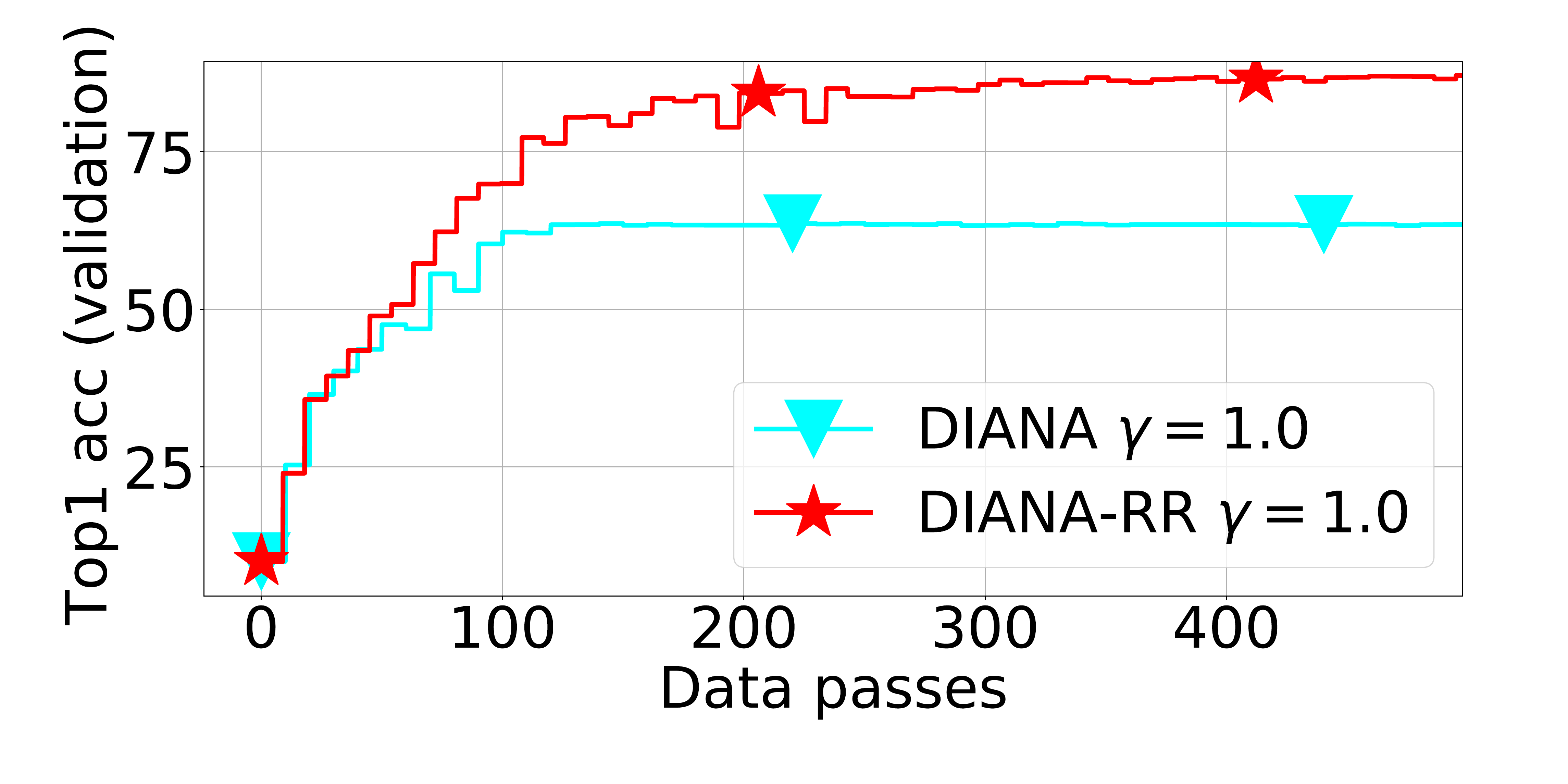} \caption{}\label{fig:train_resnet18_diana_a}
	\end{subfigure}
	\begin{subfigure}[ht]{0.32\textwidth}
		\includegraphics[width=\textwidth]{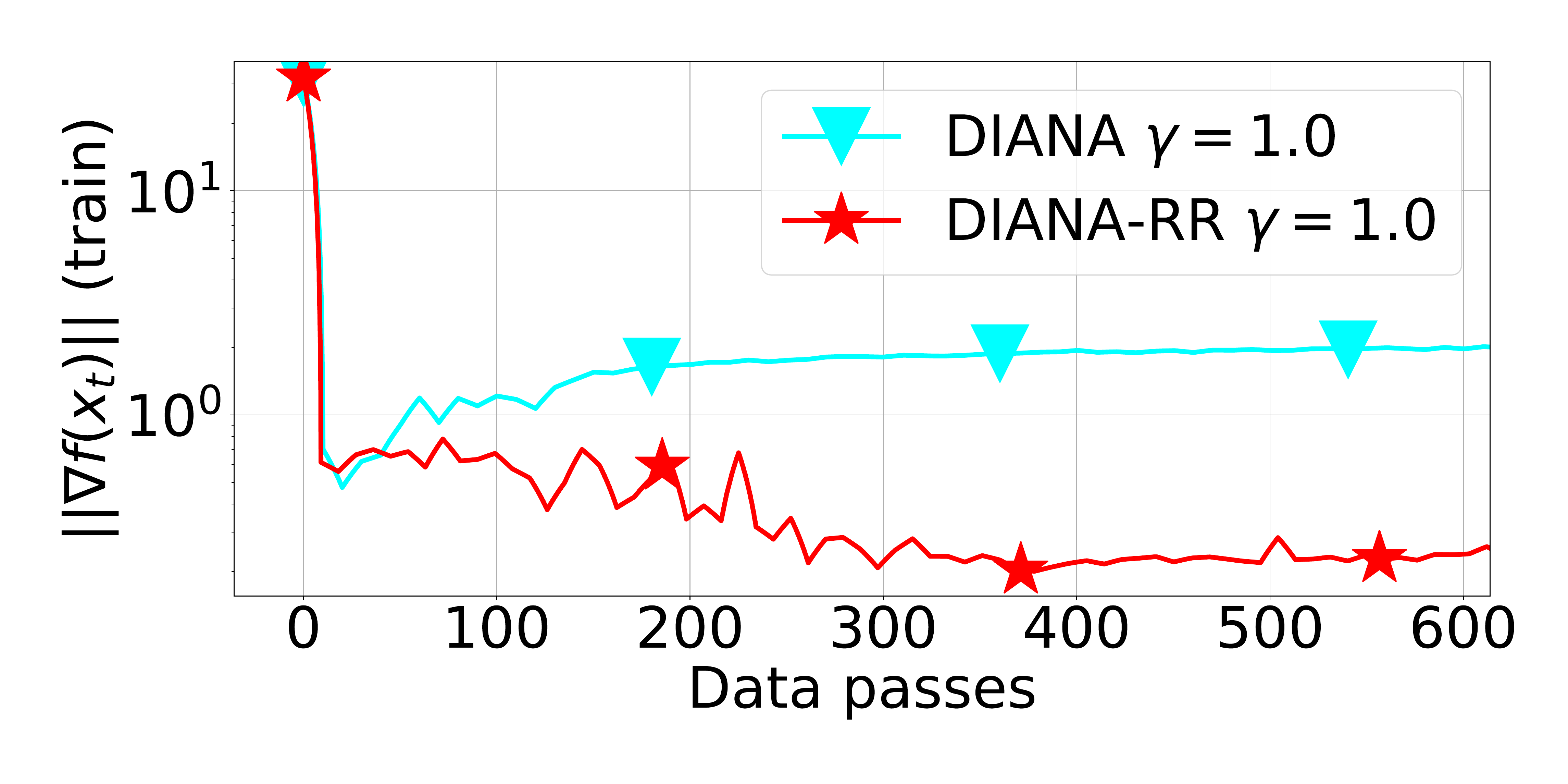} \caption{}\label{fig:train_resnet18_diana_b}
	\end{subfigure}
	\begin{subfigure}[ht]{0.32\textwidth}
		\includegraphics[width=\textwidth]{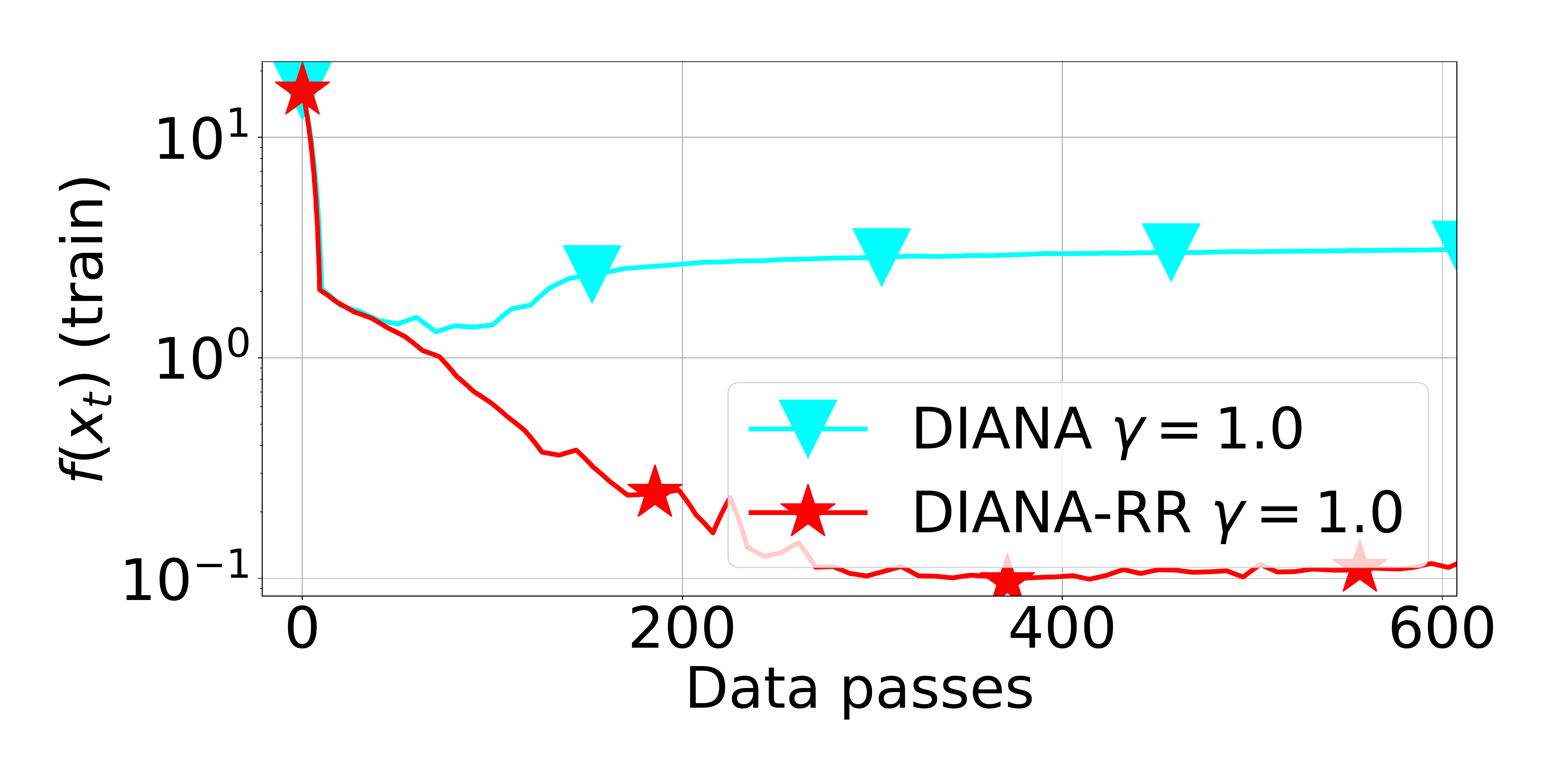} \caption{}\label{fig:train_resnet18_diana_c}
	\end{subfigure}
	
	\begin{subfigure}[ht]{0.32\textwidth}
		\includegraphics[width=\textwidth]{./plots_nn/diana_fig_2.pdf} \caption{}\label{fig:train_resnet18_diana_d}
	\end{subfigure}
	\begin{subfigure}[ht]{0.32\textwidth}
		\includegraphics[width=\textwidth]{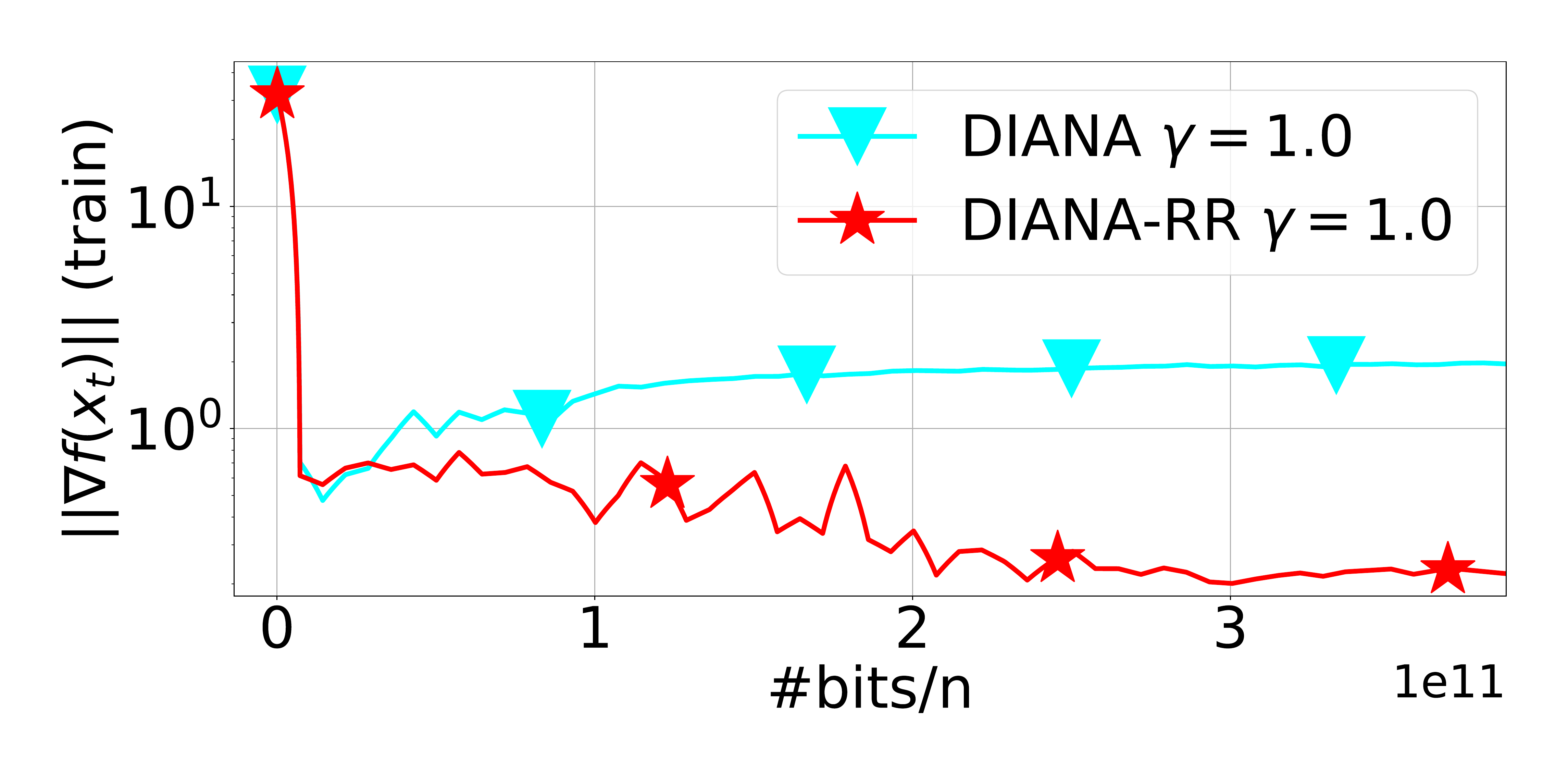}\caption{}\label{fig:train_resnet18_diana_e}
	\end{subfigure}
	\begin{subfigure}[ht]{0.32\textwidth}
		\includegraphics[width=\textwidth]{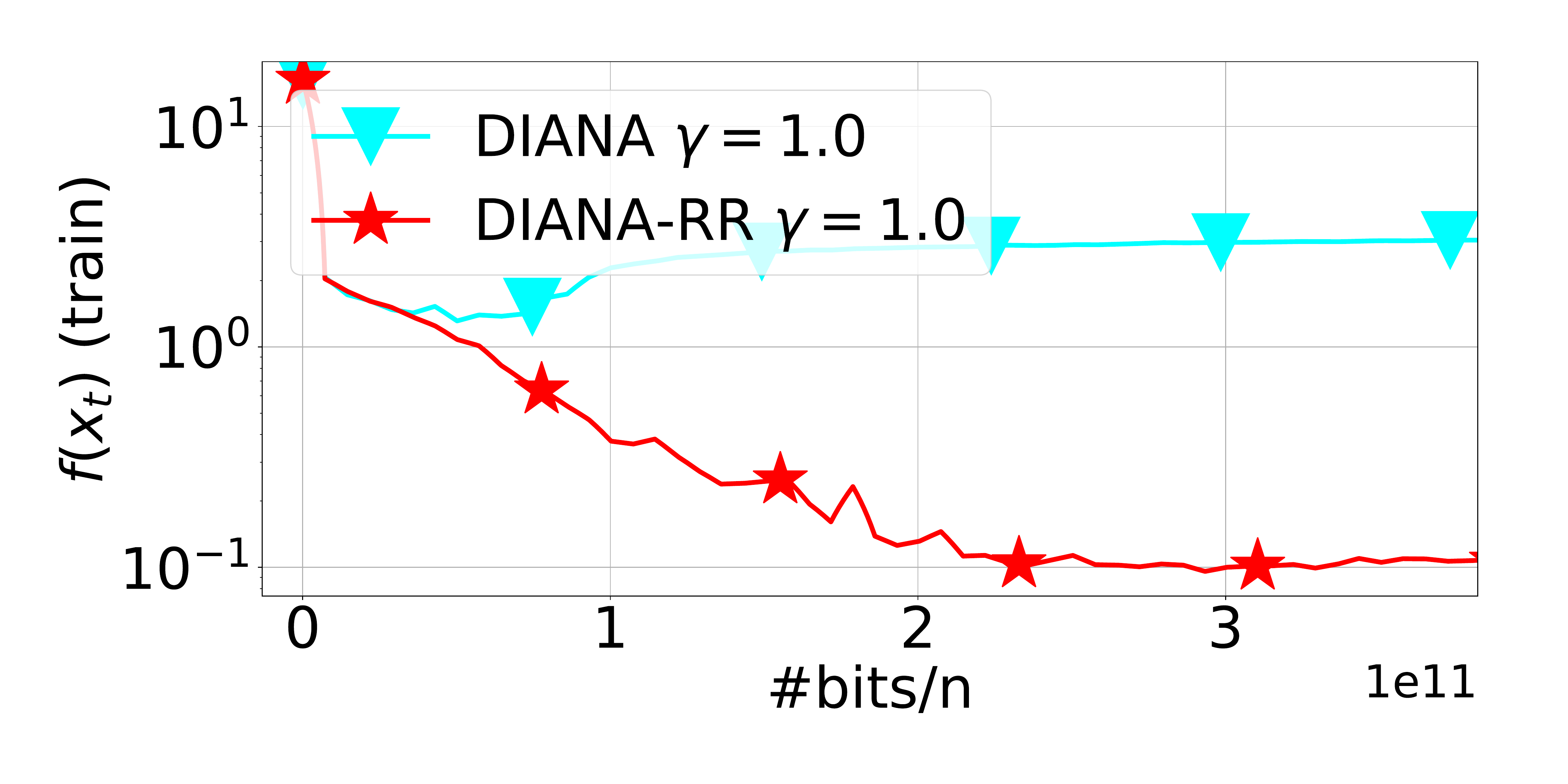}\caption{}\label{fig:train_resnet18_diana_f}
	\end{subfigure}
	\caption{{Comparison of \algname{DIANA} and \algname{DIANA-RR} in the training of \texttt{ResNet-18} on \texttt{CIFAR-10}, with $n=10$ workers. Here (a) and (d) show Top-1 accuracy on test set, (b) and (e) -- norm of full gradient on the train set, (c) and (f) -- loss function value on the train set. Stepsizes and decay shift has been tuned from $s_{set}$ and $\gamma_{set}$ based on minimum achievable value of loss function on the train set. For both algorithms stepsize is fixed. For both algorithms stepsize is decaying according to srategy $B$.}}\label{fig:train_resnet18_diana}
\end{figure}
\subsubsection{Optimization-Based Fine-Tuning for Pretrained \texttt{ResNet-18}.}
In this setting, we trained \texttt{ResNet-18} image classification in a distributed way across $n=10$ clients. In this experiment, we have trained only the last linear layer. 

Next, we have turned off batch normalization. Turning off batch normalization implies that the computation graph of NN $g(a, x)$ with weights of NN denoted as $x$ is a deterministic function and does not include any internal state.

The loss function is a standard cross-entropy loss augmented with extra $\ell_2$-regularization $\alpha \nicefrac{\|x\|^2}{2}$ with $\alpha=0.0001$. Initially used weights of NN are pretrained parameters after training the model on ImageNet. 

The dataset distribution across clients has been set in a heterogeneous manner via presorting dataset $D$ by label class and after this, it was split across $10$ clients. 

The comparison of stepsizes policies used in \algname{QSGD} and \algname{Q-RR} is presented in Figure~\ref{fig:train_resnet18_qsgd_all_bad}. The behavior of the algorithms with best tuned step sizes is presented in Figure~\ref{fig:train_resnet18_qsgd_best_to_best_bad}. These results demonstrate that in this setting there is no real benefit of using \algname{Q-RR} in comparison to \algname{QSGD}.

\subsubsection{Experiments}

The comparison of \algname{QSGD} and \algname{Q-RR} is presented in Figure \ref{fig:train_resnet18_qsgd}. In particular, Figures \ref{fig:train_resnet18_qsgd_b} and \ref{fig:train_resnet18_qsgd_e} show that in terms of the convergence to stationary points both algorithms exhibit similar behavior. However, \algname{Q-RR} has better generalization and in fact, converges to the better loss function value. This experiment demonstrates that \algname{Q-RR} with manually tuned stepsize can be better compared to \algname{QSGD} in terms of the final quality of obtained Deep Learning model. For \algname{QSGD} the tuned meta parameters are: $\gamma_{init}=3.0$,$s=200$, $\text{strategy}=B$. For  \texttt{QSGD-RR} tuned meta parameters are: $\gamma_{init}=3.0$, $s=1000$, $\text{strategy}=B$.

The results of comparison of \algname{DIANA} and \algname{DIANA-RR} are presented in Figure~\ref{fig:train_resnet18_diana}. For \algname{DIANA} the tuned meta parameters are: $\gamma_{init}=1.0$,$s=0$, $\text{strategy}=C$ and for  \algname{DIANA-RR} tuned meta parameters are: $\gamma_{init}=1.0$, $s=0$, $\text{strategy}=C$. These results show that \algname{DIANA-RR} outperforms \algname{DIANA} in terms of the all reported metrics.

\begin{figure}[H]
	\centering
	\captionsetup[sub]{font=small,labelfont={}}	
	\begin{subfigure}[ht]{0.32\textwidth}
		\includegraphics[width=\textwidth]{./plots_nn_not_good/best2best_fig1.pdf} \caption{}
	\end{subfigure}
	\begin{subfigure}[ht]{0.32\textwidth}
		\includegraphics[width=\textwidth]{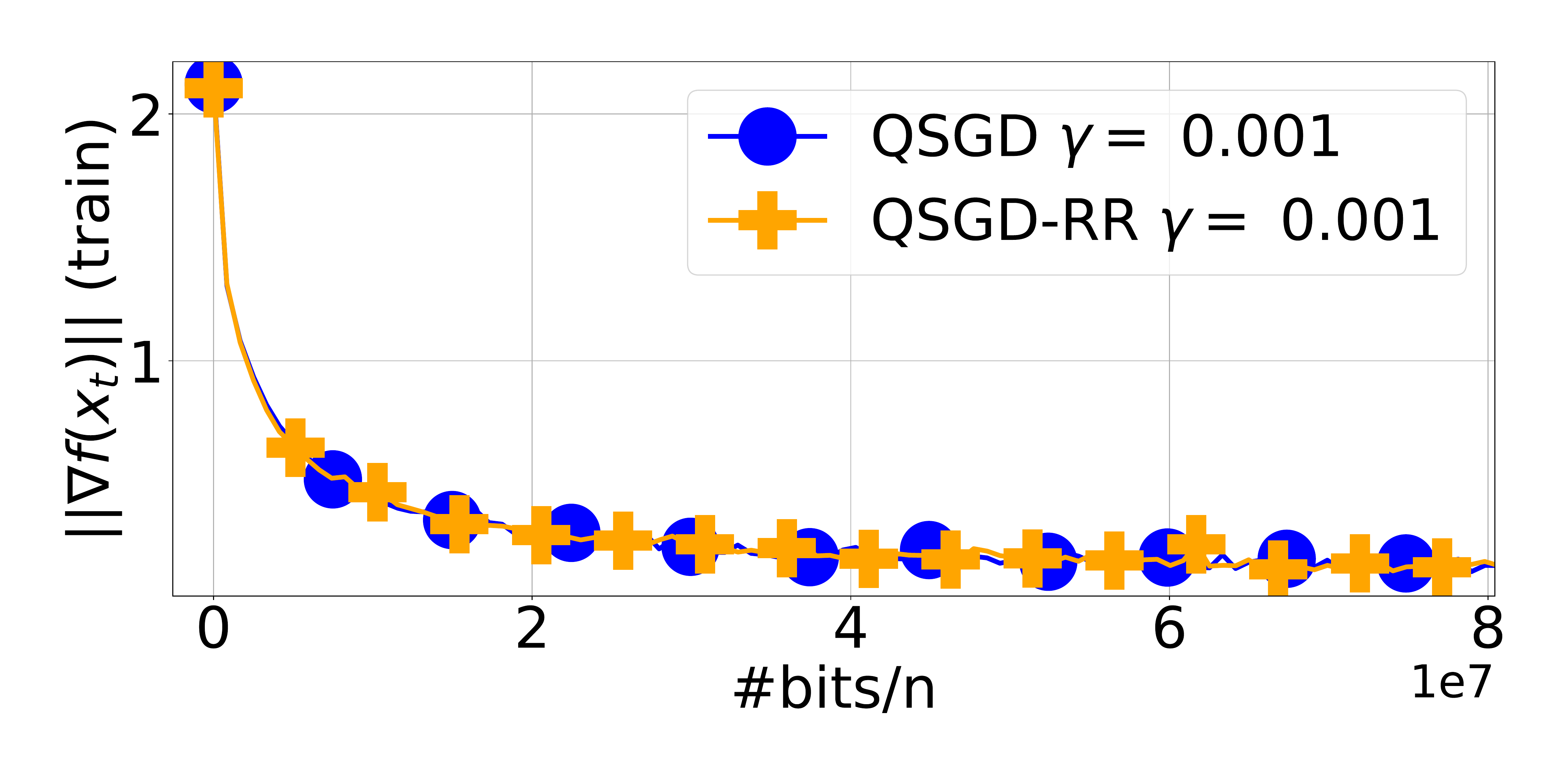} \caption{}
	\end{subfigure}
	\begin{subfigure}[ht]{0.32\textwidth}
		\includegraphics[width=\textwidth]{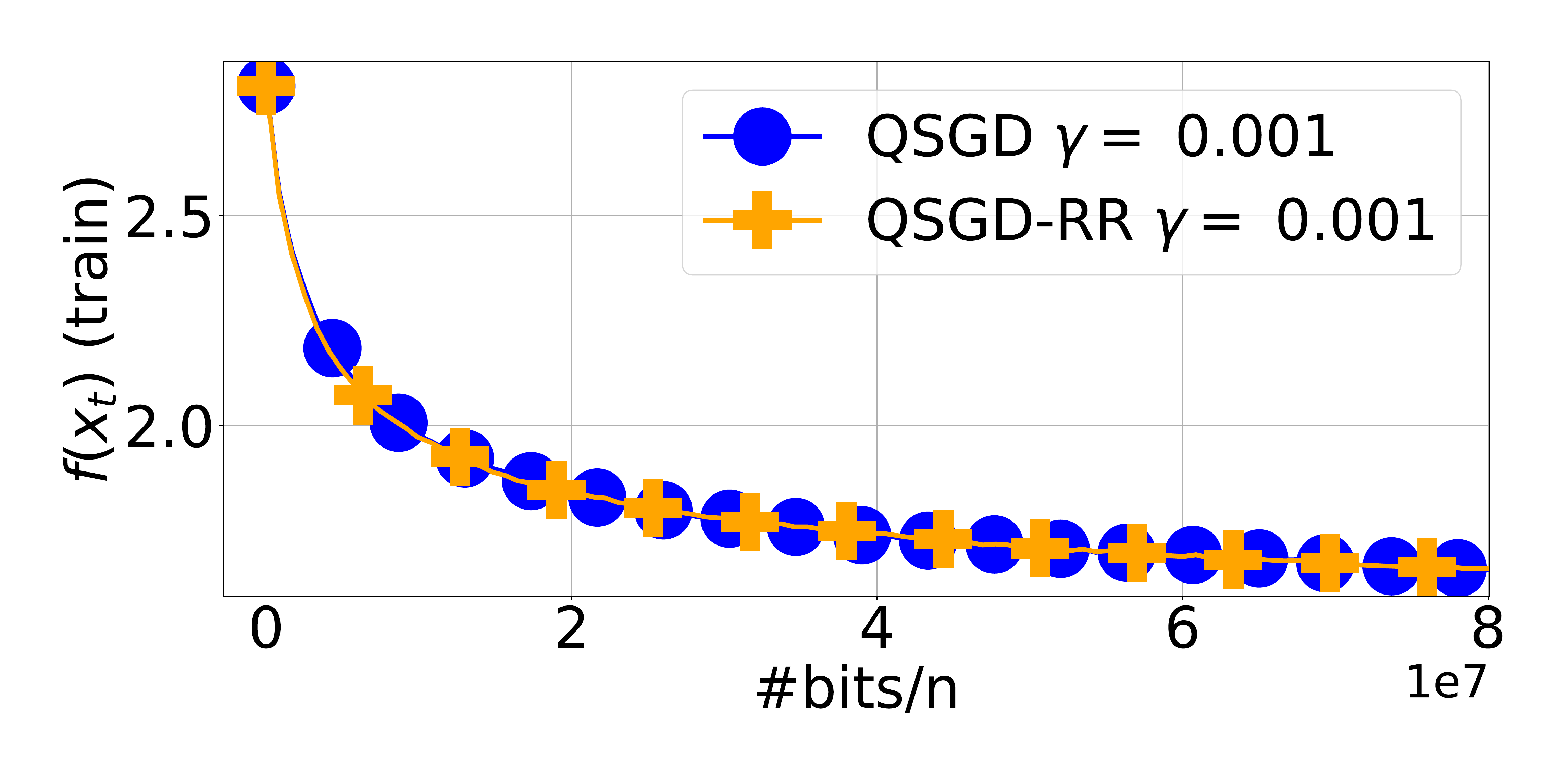} \caption{}
	\end{subfigure}
	\caption{{Comparison of \algname{QSGD} and \algname{Q-RR} in the training of the last linear layer of \texttt{ResNet-18} on \texttt{CIFAR-10}, with $n=10$ workers. Here (a) shows Top-1 accuracy on test set, (b) -- norm of full gradient on the train set, (c) -- loss function value on the train set. Stepsizes and decay shift has been tuned from $s_{set}$ and $\gamma_{set}$ based on minimum achievable value of loss function on the train set. Both algorithms used fixed stepsize during training.}}\label{fig:train_resnet18_qsgd_best_to_best_bad}
\end{figure}

\begin{figure}[H]
	\centering
	\captionsetup[sub]{font=small,labelfont={}}	
	\begin{subfigure}[ht]{0.49\textwidth}
		\includegraphics[width=\textwidth]{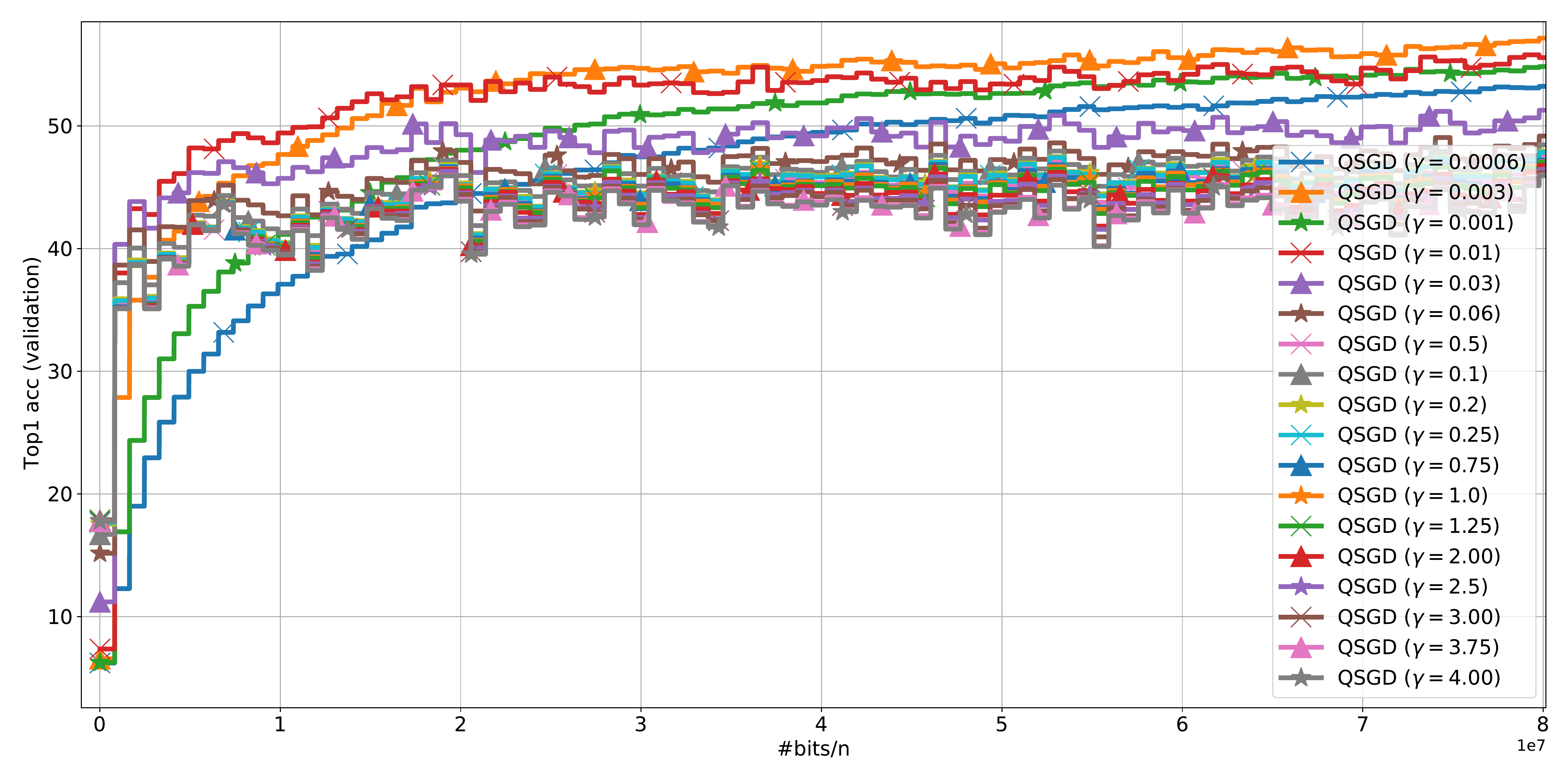} \caption{}
	\end{subfigure}
	\begin{subfigure}[ht]{0.49\textwidth}
		\includegraphics[width=\textwidth]{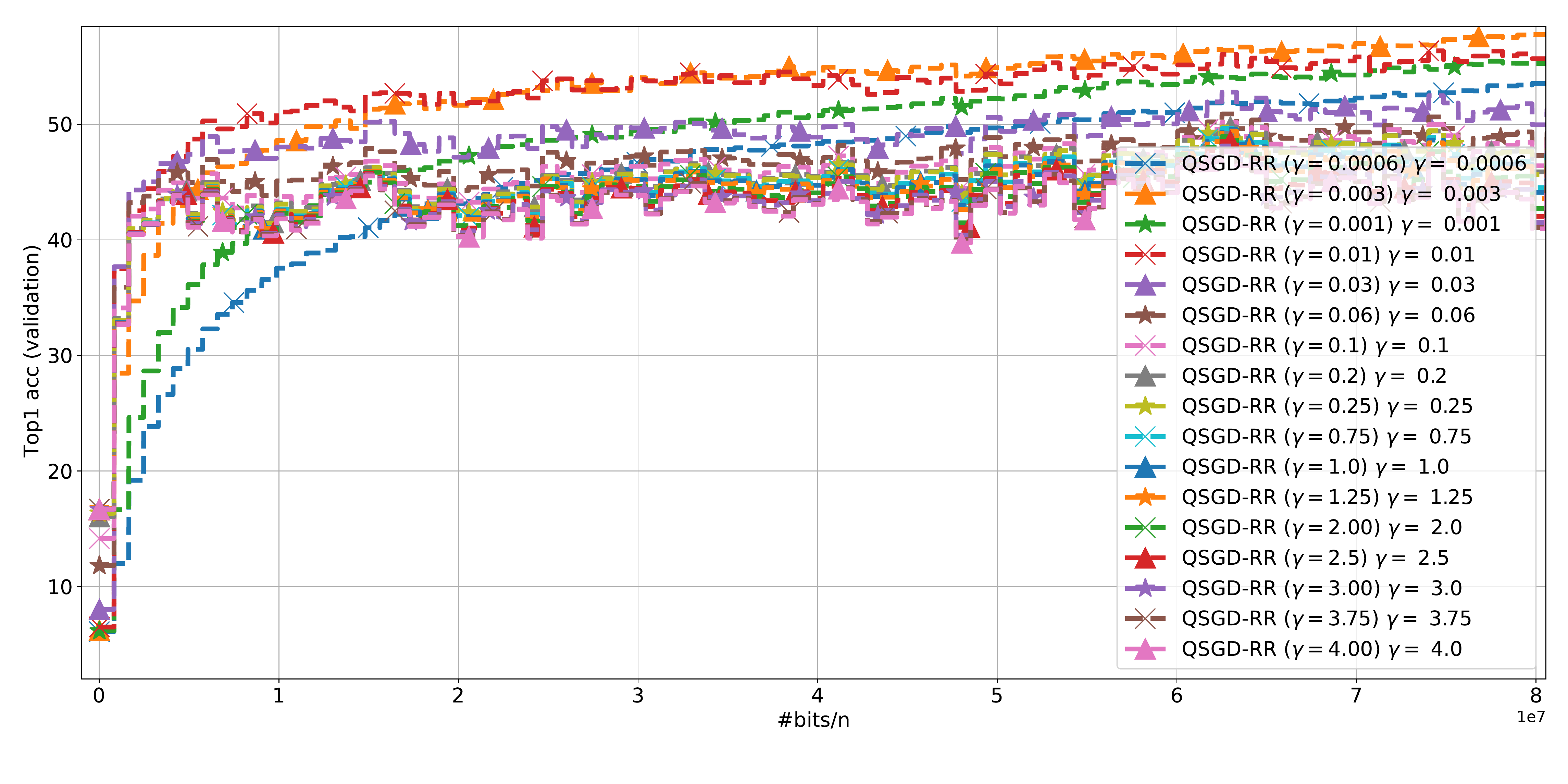} \caption{}
	\end{subfigure}

	\begin{subfigure}[ht]{0.49\textwidth}
		\includegraphics[width=\textwidth]{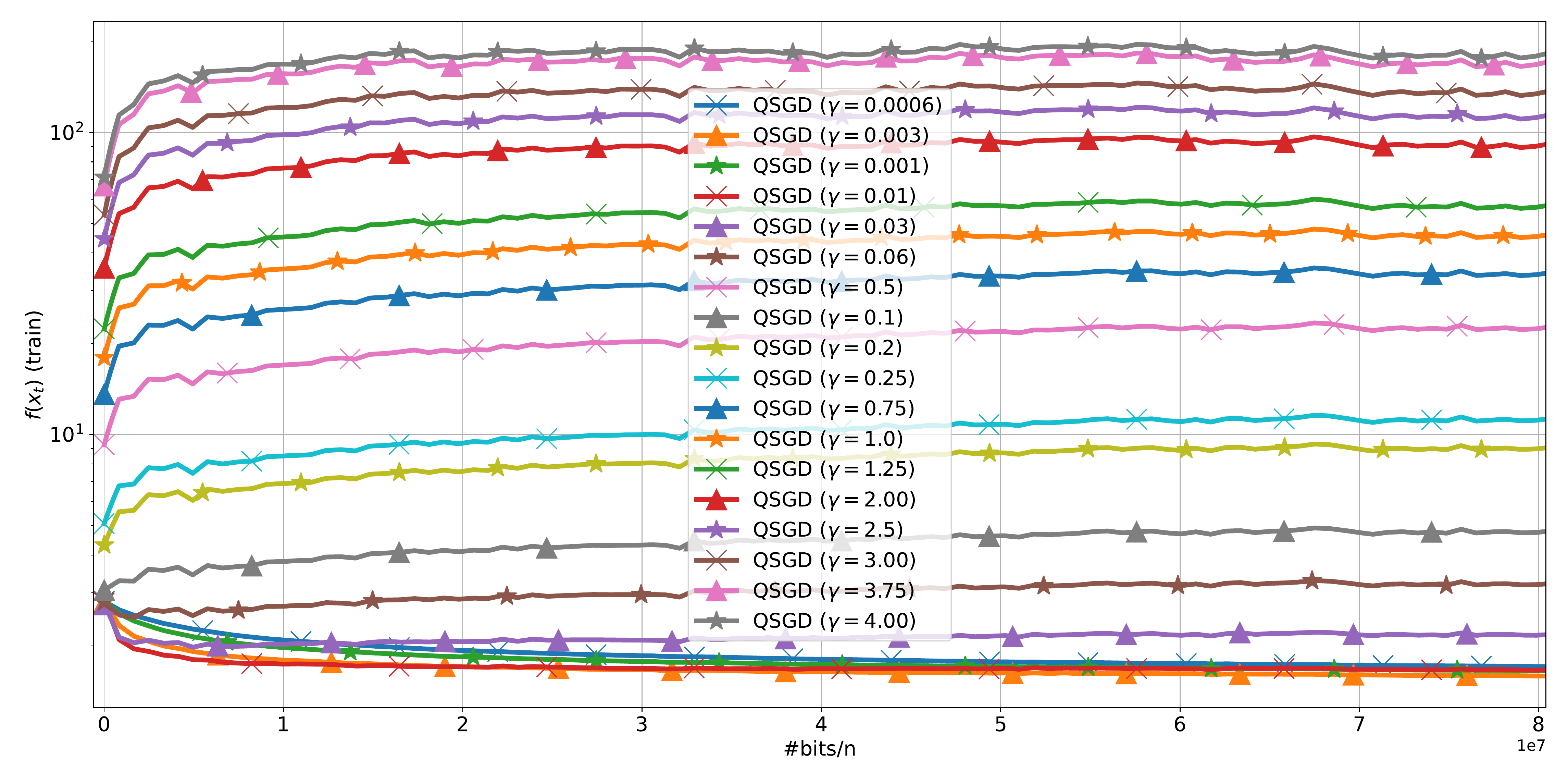} \caption{}
	\end{subfigure}
	\begin{subfigure}[ht]{0.49\textwidth}
		\includegraphics[width=\textwidth]{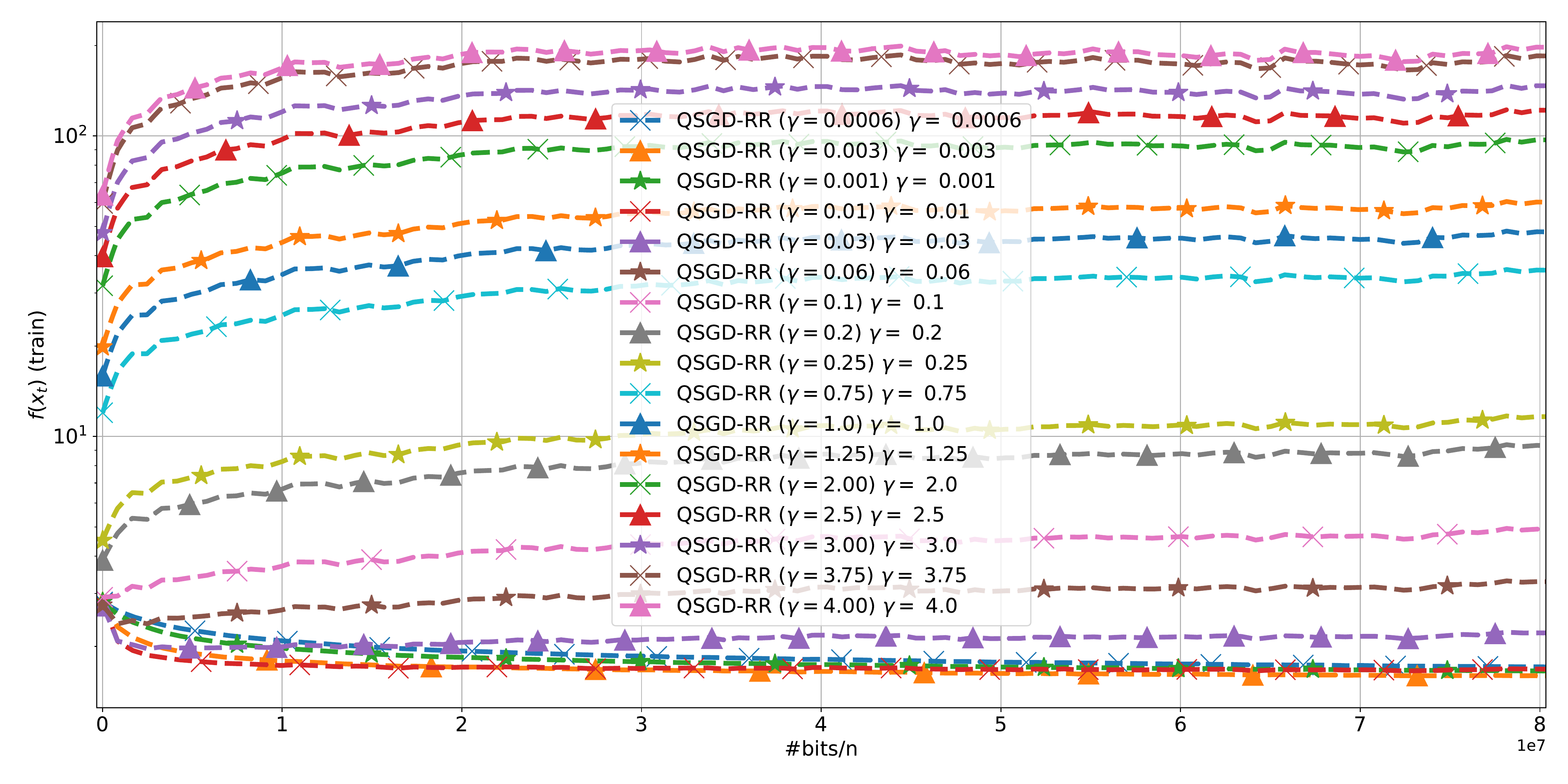} \caption{}
	\end{subfigure}
	
	\begin{subfigure}[ht]{0.49\textwidth}
		\includegraphics[width=\textwidth]{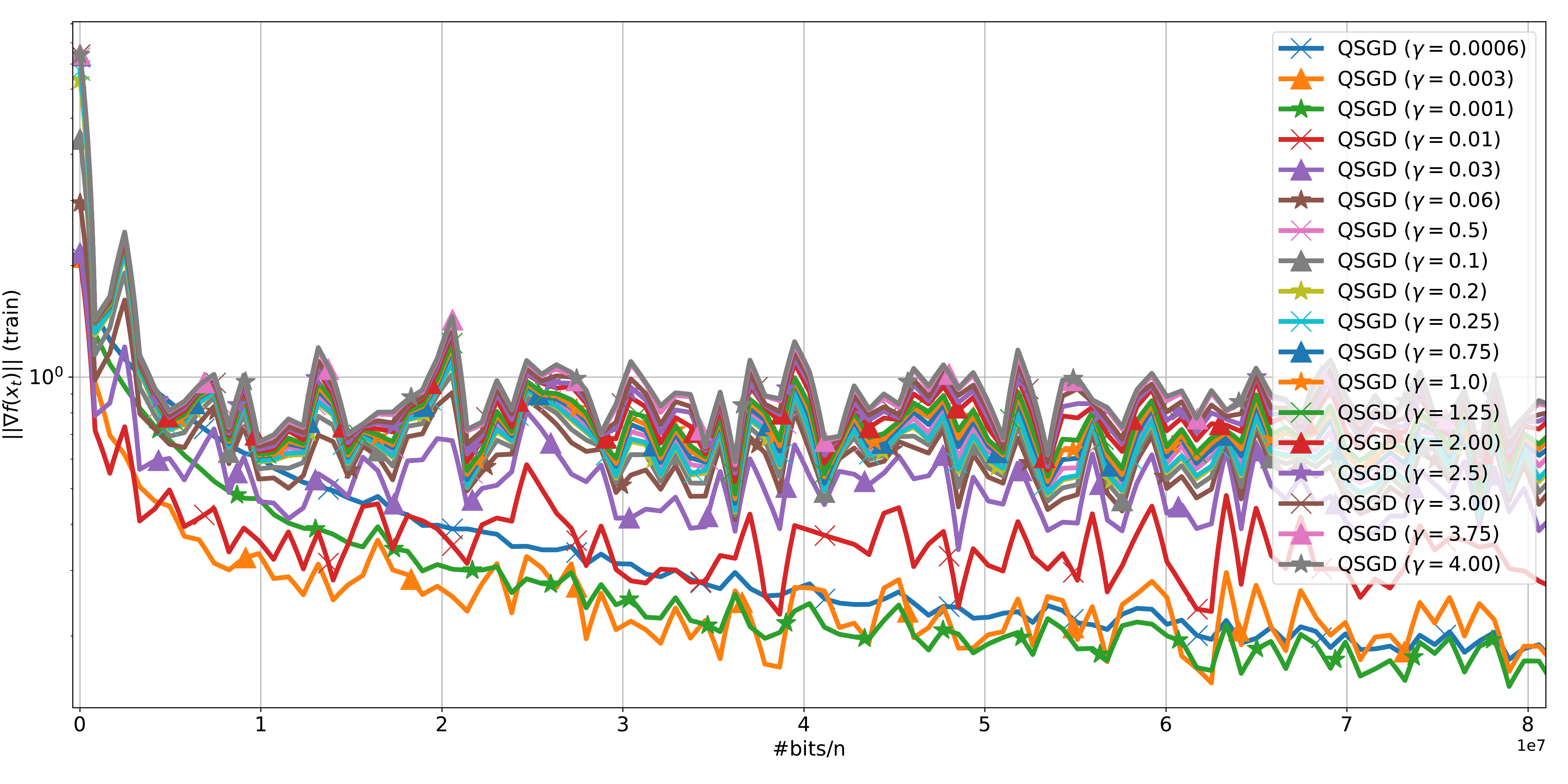} \caption{}
	\end{subfigure}
	\begin{subfigure}[ht]{0.49\textwidth}
		\includegraphics[width=\textwidth]{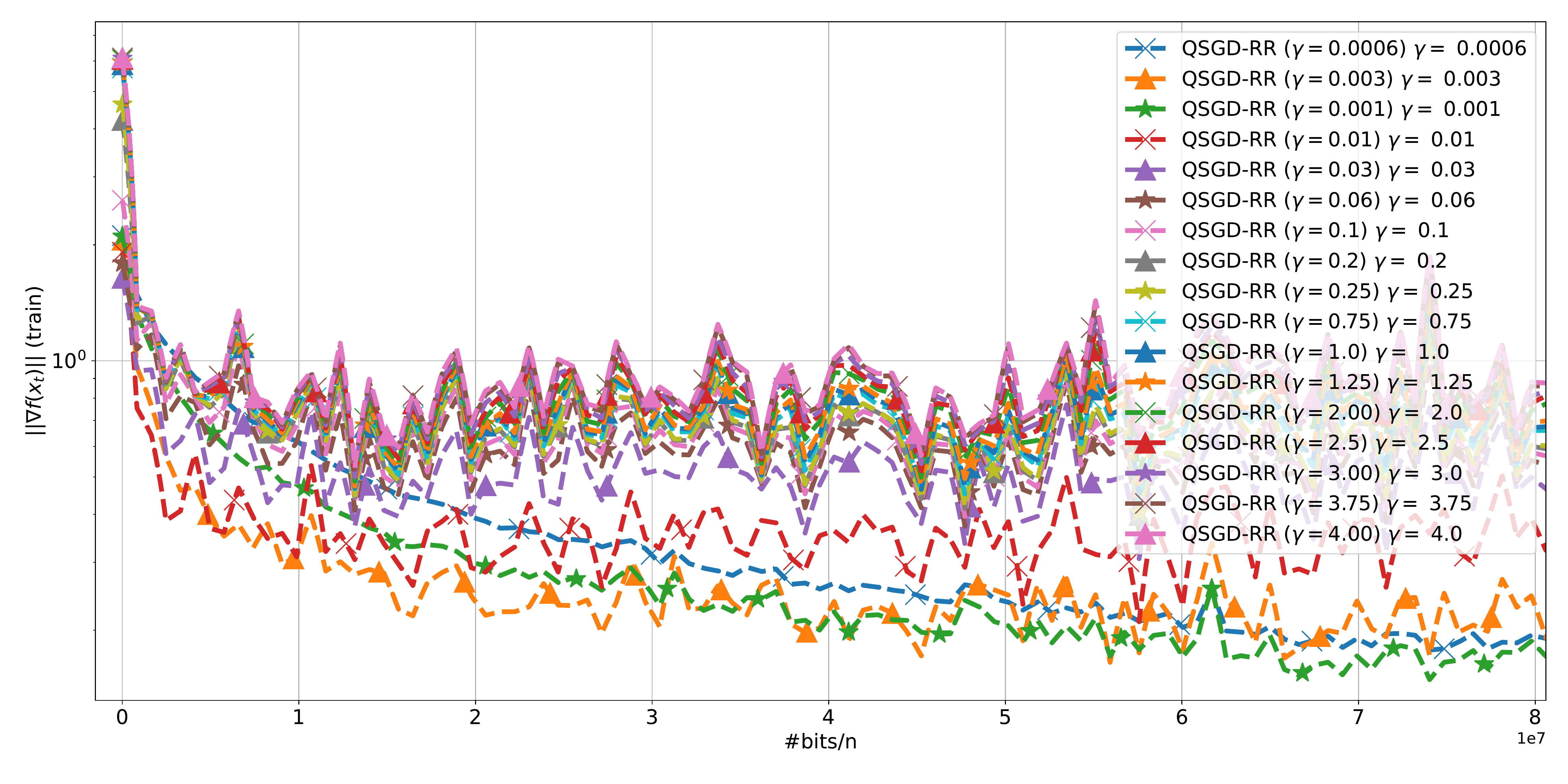} \caption{}
	\end{subfigure}
	\caption{{Comparison of \algname{QSGD} and \algname{Q-RR} in the training of the last linear layer of \texttt{ResNet-18} on \texttt{CIFAR-10}, with $n=10$ workers. Here (a) and (b) show Top-1 accuracy on test set, (c) and (d) -- loss function value on the train set, (e) and (f) -- norm of full gradient on the train set. Stepsizes and decay shift has been tuned from $s_{set}$ and $\gamma_{set}$ based on minimum achievable value of loss function on the train set. During training stepsize was fixed. Batch Normalization was turned off.}}\label{fig:train_resnet18_qsgd_all_bad}
\end{figure}

\newpage

\section{Missing Proofs For Q-RR}

In the main part of the paper, we intoduce Assumptions~\ref{asm:lip_max_f_m} and \ref{asm:sc_each_f_m} for the analysis of \algname{Q-RR} and \algname{DIANA-RR}. These assumptions can be refined as follows.

\begin{assumption}
	\label{asm:lip_avr_f}
	Function $f^{\pi^i} = \frac{1}{M}\sum^M_{i=1}f^{\pi_m^i}_m: \R^d \rightarrow \R$ is $\widetilde{L}$-smooth for all sets of permutations $\pi = (\pi_1,\ldots,\pi_m)$ from $[n]$ and all $i \in [n]$, i.e.,
	\begin{equation}
		\max_{i \in [n], \pi}\|\nabla f^{\pi^i}(x) - \nabla f^{\pi^i}(y)\|\leq \widetilde{L}\|x - y\| \quad \forall x, y \in \R^d.\notag
	\end{equation}
\end{assumption}

\begin{assumption}
\label{asm:sc_each_f_m_refined}
Function $f^{\pi^i} = \frac{1}{M}\sum^M_{i=1}f^{\pi_m^i}_m: \R^d \rightarrow \R$ is $\widetilde{\mu}$-strongly convex for all sets of permutations $\pi = (\pi_1,\ldots,\pi_m)$ from $[n]$ and all $i \in [n]$, i.e.,
\begin{equation}
	\min_{i \in [n], \pi}\left\{f^{\pi^i}(x) - f^{\pi^i}(y) - \la\nabla f^{i,\pi}(y), x - y\ra\right\} \geq \frac{\widetilde{\mu}}{2}\|x - y\|^2 \quad \forall x, y \in \R^d. \notag
\end{equation}
Moreover, functions $f^i_1, f^i_2, \dots, f^i_M: \R^d \rightarrow \R$  are convex for all $i = 1,\dots, n$.
\end{assumption}

We notice that Assumptions~\ref{asm:lip_max_f_m} and \ref{asm:sc_each_f_m} imply Assumptions~\ref{asm:lip_avr_f} and \ref{asm:sc_each_f_m_refined}. Moreover, $\widetilde{L} \leq L_{\max}$. In the proofs of the results for \algname{Q-RR} and \algname{DIANA-RR}, we use Assumptions~\ref{asm:lip_avr_f} in addition to Assumption~\ref{asm:lip_max_f_m} and we use Assumption~\ref{asm:sc_each_f_m_refined} instead of Assumption~\ref{asm:sc_each_f_m}.

\subsection{Proof of Theorem \ref{th_conv_new_rr_q}}
For convenience, we restate the theorem below.
\begin{theorem}[Theorem~\ref{th_conv_new_rr_q}]
	\label{thm:advaced_conv_Q_RR}
	Let Assumptions \ref{asm:quantization_operators}, \ref{asm:lip_max_f_m}, \ref{asm:lip_avr_f}, \ref{asm:sc_each_f_m_refined} hold and $0 < \gamma \leq \frac{1}{\widetilde{L}+2\frac{\omega}{M}L_{\max}}$. Then, for all $T \geq 0$ the iterates produced by \algname{Q-RR} (Algorithm~\ref{alg_new_Q_RR}) satisfy
	\begin{equation}
		\EE\|x_{T}-x_{\star}\|^2\leq \left(1-\gamma\widetilde{\mu}\right)^{nT}\|x_0-x_{\star}\|^2 + \frac{2\gamma^2\sigma^2_{\text{rad}}}{\widetilde{\mu}}+\frac{2\gamma\omega}{\widetilde{\mu} M}\left(\zeta^2_{\star}+\sigma^2_{\star}\right),\notag
	\end{equation}
	where $\zeta^2_{\star} \eqdef \frac{1}{M}\sum\limits_{m=1}^M\|\nabla f_m (x_{\star})\|^2,$ and $\sigma_{\star}^2 \eqdef \frac{1}{Mn}\sum\limits_{m=1}^M\sum\limits_{i=1}^n \|\nabla f_m^i(x_\star) - \nabla f_m(x_\star)\|^2.$
\end{theorem}
\begin{proof}
	Using \(x^{i+1}_{\star} = x^{i}_{\star} -\frac{\gamma}{M}\sum^M_{m=1}\nabla f^{\pi^i_m}_m(x_{\star})\) and line 7 of Algorithm \ref{alg_new_Q_RR}, we get
	\begin{eqnarray*}
		\|x^{i+1}_t-x^{i+1}_{\star}\|^2 &=& \left\|x^i_t-x^{i}_{\star}- \gamma\frac{1}{M}\sum^M_{m=1} \left(\cQ\left(\nabla f^{\pi^{i}_m}_m(x^{i}_{t})\right) - \nabla f^{\pi^{i}_m}_m(x_{\star})\right)\right\|^2\\
		&=& \left\|x^i_t-x^{i}_{\star}\right\|^2 - 2\gamma\left\la\frac{1}{M}\sum^M_{m=1} \left(\cQ\left(\nabla f^{\pi^{i}_m}_m(x^{i}_{t})\right) - \nabla f^{\pi^{i}_m}_m(x_{\star})\right), x^i_t-x^{i}_{\star} \right\ra\\
		&&+\gamma^2\left\|\frac{1}{M}\sum^M_{m=1} \left(\cQ\left(\nabla f^{\pi^{i}_m}_m(x^{i}_{t})\right) - \nabla f^{\pi^{i}_m}_m(x_{\star})\right)\right\|^2.
	\end{eqnarray*}
	Taking the expectation w.r.t.\ $\cQ$, we obtain 
	\begin{eqnarray*}
		\EE_{\cQ}\left[\|x^{i+1}_t-x^{i+1}_{\star}\|^2 \right]
		&=& \left\|x^i_t-x^{i}_{\star}\right\|^2 - 2\gamma\left\la\frac{1}{M}\sum^M_{m=1} \left(\nabla f^{\pi^{i}_m}_m(x^{i}_{t}) - \nabla f^{\pi^{i}_m}_m(x_{\star})\right), x^i_t-x^{i}_{\star} \right\ra\\
		&&+\gamma^2\EE_{\cQ}\left[\left\|\frac{1}{M}\sum^M_{m=1} \left(\cQ\left(\nabla f^{\pi^{i}_m}_m(x^{i}_{t})\right) - \nabla f^{\pi^{i}_m}_m(x_{\star})\right)\right\|^2\right].
	\end{eqnarray*}
	In view of Assumption \ref{asm:quantization_operators} and $\EE_{\xi}\|\xi - c\|^2 = \EE_{\xi}\|\xi - \EE_{\xi}\xi\|^2+ \|\EE_{\xi}\xi - c\|^2 $, we have 
	\begin{eqnarray*}
		\EE_{\cQ}\left[\|x^{i+1}_t-x^{i+1}_{\star}\|^2 \right]
		&=& \left\|x^i_t-x^{i}_{\star}\right\|^2 - \frac{2\gamma}{M}\sum^M_{m=1}\left\la \nabla f^{\pi^{i}_m}_m(x^{i}_{t}) - \nabla f^{\pi^{i}_m}_m(x_{\star}), x^i_t-x^{i}_{\star} \right\ra\\
		&&+\gamma^2\EE_{\cQ}\left[\left\|\frac{1}{M}\sum^M_{m=1} \left(\cQ\left(\nabla f^{\pi^{i}_m}_m(x^{i}_{t})\right) - \nabla f^{\pi^{i}_m}_m(x^{i}_{t})\right)\right\|^2\right]\\
		&&+ \gamma^2\left\|\frac{1}{M}\sum^M_{m=1} \left(\nabla f^{\pi^{i}_m}_m(x^{i}_{t}) - \nabla f^{\pi^{i}_m}_m(x_{\star})\right)\right\|^2\\
		&\leq& \left\|x^i_t-x^{i}_{\star}\right\|^2 - \frac{2\gamma}{M}\sum^M_{m=1}\left\la \nabla f^{\pi^{i}_m}_m(x^{i}_{t}) - \nabla f^{\pi^{i}_m}_m(x_{\star}), x^i_t-x^{i}_{\star} \right\ra\\
		&&+ \gamma^2\left\|\frac{1}{M}\sum^M_{m=1} \left(\nabla f^{\pi^{i}_m}_m(x^{i}_{t}) - \nabla f^{\pi^{i}_m}_m(x_{\star})\right)\right\|^2 +\frac{\gamma^2\omega}{M^2}\sum^M_{m=1}\left\| \nabla f^{\pi^{i}_m}_m(x^{i}_{t})\right\|^2,
	\end{eqnarray*}
	where in the last step we apply independence of $\cQ\left(\nabla f^{\pi^{i}_m}_m(x^{i}_{t})\right)$ for $m\in [M]$. Next, we use three-point identity\footnote{For any differentiable function $f : \R^d \to \R^d$ we have: $\langle \nabla f(x) - \nabla f(y), x- z \rangle = D_{f}(z, x) + D_f(x, y) - D_f(z, y)$.} and obtain 
	\begin{eqnarray*}
		\EE_{\cQ}\left[\|x^{i+1}_t-x^{i+1}_{\star}\|^2 \right]
		&\leq& \left\|x^i_t-x^{i}_{\star}\right\|^2 - \frac{2\gamma}{M}\sum^M_{m=1}\left( D_{f^{\pi^{i}_m}_m}(x^{i}_{\star},x^{i}_{t}) +D_{f^{\pi^{i}_m}_m}(x^{i}_{t},x_{\star})- D_{f^{\pi^{i}_m}_m}(x^{i}_{\star},x_{\star})\right) \\
		&&+ \gamma^2\left\|\frac{1}{M}\sum^M_{m=1} \left(\nabla f^{\pi^{i}_m}_m(x^{i}_{t}) - \nabla f^{\pi^{i}_m}_m(x_{\star})\right)\right\|^2 +\frac{\gamma^2\omega}{M^2}\sum^M_{m=1}\left\| \nabla f^{\pi^{i}_m}_m(x^{i}_{t})\right\|^2.
	\end{eqnarray*}
	Applying $\widetilde{L}$-smoothness and convexity of $\frac{1}{M}\sum_{m= 1}^m f_m^{\pi_m^i}$, $\widetilde\mu$-strong convexity of $\frac{1}{M}\sum_{m= 1}^m f_m^{\pi_m^i}$, and $L_{\max}$-smoothness and convexity of $f_m^i$, we get
	\begin{eqnarray*}
		\EE_{\cQ}\left[\|x^{i+1}_t-x^{i+1}_{\star}\|^2 \right]
		&\leq& \left(1-\gamma\widetilde\mu\right)\left\|x^i_t-x^{i}_{\star}\right\|^2 - 2\gamma \left(1 -\widetilde{L}\gamma\right)\frac{1}{M}\sum^M_{m=1}D_{f^{\pi^{i}_m}_m}(x^{i}_{t},x_{\star}) \\
		&&+ 2\gamma\frac{1}{M}\sum^M_{m=1}D_{f^{\pi^{i}_m}_m}(x^{i}_{\star},x_{\star}) +\frac{\gamma^2\omega}{M^2}\sum^M_{m=1}\left\| \nabla f^{\pi^{i}_m}_m(x^{i}_{t})\right\|^2\\
		&\leq& \left(1-\gamma\widetilde\mu\right)\left\|x^i_t-x^{i}_{\star}\right\|^2 - 2\gamma \left(1 -
		\widetilde{L}\gamma\right)\frac{1}{M}\sum^M_{m=1}D_{f^{\pi^{i}_m}_m}(x^{i}_{t},x_{\star}) \\
		&&+ 2\gamma\frac{1}{M}\sum^M_{m=1}D_{f^{\pi^{i}_m}_m}(x^{i}_{\star},x_{\star}) + \frac{2\gamma^2\omega}{M^2}\sum^M_{m=1}\left\| \nabla f^{\pi^{i}_m}_m(x_{\star})\right\|^2\\
		&&+ \frac{2\gamma^2\omega}{M^2}\sum^M_{m=1}\left\| \nabla f^{\pi^{i}_m}_m(x^{i}_{t}) - \nabla f^{\pi^{i}_m}_m(x_{\star})\right\|^2 \\
		&\leq& \left(1-\gamma\widetilde\mu\right)\left\|x^i_t-x^{i}_{\star}\right\|^2  + \frac{2\gamma^2\omega}{M^2}\sum^M_{m=1}\left\| \nabla f^{\pi^{i}_m}_m(x_{\star})\right\|^2+ \frac{2\gamma}{M}\sum^M_{m=1}D_{f^{\pi^{i}_m}_m}(x^{i}_{\star},x_{\star})\\
		&&- 2\gamma \left(1 -\gamma\left(\widetilde{L} +\frac{2\omega L_{\max}}{M}\right)\right)\frac{1}{M}\sum^M_{m=1}D_{f^{\pi^{i}_m}_m}(x^{i}_{t},x_{\star}). 
	\end{eqnarray*}
	Taking the  full  expectation and using a definition of shuffle radius, $0<\gamma \leq \frac{1}{\left(\widetilde{L}+2\frac{\omega}{M}L_{\max}\right)}$, and $D_{f^{\pi^{i}_m}_m}(x^{i}_{t},x_{\star}) \geq 0$, we obtain
	\begin{eqnarray*}
		\EE\left[\|x^{i+1}_t-x^{i+1}_{\star}\|^2 \right]
		&\leq& \left(1-\gamma\widetilde\mu\right)\EE\left[\left\|x^i_t-x^{i}_{\star}\right\|^2\right] + 2\gamma^3\sigma^2_{\text{rad}} + \frac{2\gamma^2\omega}{M^2}\sum^M_{m=1}\EE\left[\left\| \nabla f^{\pi^{i}_m}_m(x_{\star})\right\|^2\right]\\
		&=& \left(1-\gamma\widetilde\mu\right)\EE\left[\left\|x^i_t-x^{i}_{\star}\right\|^2\right] + 2\gamma^3 \sigma^2_{\text{rad}} + \frac{2\gamma^2\omega}{M^2n}\sum^M_{m=1}\sum^n_{j=1}\left\| \nabla f^{j}_m(x_{\star})\right\|^2\\
		&\leq& \left(1-\gamma\widetilde\mu\right)\EE\left[\left\|x^i_t-x^{i}_{\star}\right\|^2\right] + 2\gamma^3 \sigma^2_{\text{rad}} + \frac{2\gamma^2\omega}{M}\left(\zeta^2_{\star}+\sigma^2_{\star}\right).
	\end{eqnarray*}
	Unrolling the recurrence in $i$, we derive 
	\begin{eqnarray*}
		\EE\left[\|x_{t+1}-x_{\star}\|^2 \right]
		&\leq& \left(1-\gamma\widetilde\mu\right)^n\EE\left[\left\|x_t-x_{\star}\right\|^2\right] + 2\gamma^3 \sigma^2_{\text{rad}}\sum^{n-1}_{j=0}(1-\gamma\widetilde\mu)^j\\
		&&+ \frac{2\gamma^2\omega}{M}\left(\zeta^2_{\star}+\sigma^2_{\star}\right)\sum^{n-1}_{j=0}(1-\gamma\widetilde\mu)^j.
	\end{eqnarray*}
	Unrolling the recurrence in $t$, we derive 
	\begin{eqnarray*}
		\EE\left[\|x_{T}-x_{\star}\|^2 \right]
		&\leq& \left(1-\gamma\widetilde\mu\right)^{nT}\left\|x_0-x_{\star}\right\|^2 + 2\gamma^3\sigma^2_{\text{rad}}\sum^{T-1}_{t=0}(1-\gamma\widetilde\mu)^{nt}\sum^{n-1}_{j=0}(1-\gamma\widetilde\mu)^j \\
		&&+ \frac{2\gamma^2\omega}{M}\left(\zeta^2_{\star}+\sigma^2_{\star}\right)\sum^{nT-1}_{j=0}(1-\gamma\widetilde\mu)^{nt}\sum^{n-1}_{j=0}(1-\gamma\widetilde\mu)^j.
	\end{eqnarray*}
	Since $\sum^{nT-1}_{j=0}(1-\gamma\widetilde\mu)^{j} \leq \frac{1}{\gamma\widetilde\mu}$, we get the result.
\end{proof}

\begin{corollary}
\label{cor_convergence_new_Q_RR}
Let the assumptions of Theorem~\ref{thm:advaced_conv_Q_RR} hold and
\begin{equation}
	\label{new_gamma_new_Q_RR}
	\gamma = \min\left\{\frac{1}{\widetilde{L}+2\frac{\omega}{M}L_{\max}}, \sqrt{\frac{\varepsilon\widetilde{\mu}}{6\sigma^2_{\text{rad}}}}, \frac{\varepsilon\widetilde{\mu}M}{6\omega\left(\zeta^2_{\star}+\sigma^2_{\star}\right)}\right\}.
\end{equation}
Then, \algname{Q-RR} (Algorithm~\ref{alg_new_Q_RR}) finds a solution with accuracy $\varepsilon>0$ after the following number of communication rounds: 
\begin{equation}
	\widetilde{\cO}\left(\frac{\widetilde{L}}{\widetilde{\mu}}+\frac{\omega}{M}\frac{L_{\max}}{\widetilde{\mu}} + \frac{\omega}{M}\frac{\zeta^2_{\star}+\sigma^2_{\star}}{\varepsilon\widetilde{\mu}^2} + \frac{\sigma_{\text{rad}}}{\sqrt{\varepsilon\widetilde{\mu}^3}}\right).\notag
\end{equation}
\end{corollary}
\begin{proof}
Theorem~\ref{thm:advaced_conv_Q_RR} implies
\begin{equation}
	\EE\|x_{T}-x_{\star}\|^2\leq \left(1-\gamma\widetilde{\mu}\right)^{nT}\|x_0-x_{\star}\|^2 + \frac{2\gamma^2\sigma^2_{\text{rad}}}{\widetilde{\mu}}+\frac{2\gamma\omega}{\widetilde{\mu} M}\left(\zeta^2_{\star}+\sigma^2_{\star}\right).
\end{equation}
To estimate the number of communication rounds required to find a solution with accuracy $\varepsilon >0$, we need to upper-bound each term from the right-hand side by $\nicefrac{\varepsilon}{3}$. Thus, we get additional conditions on $\gamma$:
\begin{equation*}
	\frac{2\gamma^2\sigma^2_{\text{rad}}}{\widetilde{\mu}} < \frac{\varepsilon}{3},\qquad\frac{2\gamma\omega}{\widetilde{\mu} M}\left(\zeta^2_{\star}+\sigma^2_{\star}\right)<\frac{\varepsilon}{3}
\end{equation*}
and also the upper bound on the number of communication rounds $nT$
\begin{equation*}
	nT = \widetilde{\cO}\left(\frac{1}{\gamma\widetilde{\mu}}\right).
\end{equation*}
Substituting \eqref{new_gamma_new_Q_RR}, we get a final result.
\end{proof}

\subsection{Non-Strongly Convex Summands}
In this section, we provide the analysis of \algname{Q-RR} without using Assumptions~\ref{asm:sc_each_f_m}, \ref{asm:sc_each_f_m_refined}. Before we move one to the proofs, we would like to emphasize that 
\begin{equation*}
	x^{i+1}_t = x^{i}_t - \gamma \frac{1}{M}\sum^{M}_{m=1}\cQ\left(\nabla f^{\pi^i_m}_m(x^i_t)\right).
\end{equation*}
Then we have 
\begin{equation*}
	x_{t+1} = x_t - \gamma \sum^{n-1}_{i=0}\frac{1}{M}\sum^{M}_{m=1}\cQ\left(\nabla f^{\pi^i_m}_m(x^i_t)\right) = x_t - \tau \frac{1}{Mn}\sum^{n-1}_{i=0} \sum^{M}_{m=1}\cQ\left(\nabla f^{\pi^i_m}_m(x^i_t)\right),
\end{equation*}
where $\tau = \gamma n$. For convenience, we denote
\begin{equation*}
	g_t \eqdef \frac{1}{Mn}\sum^{n-1}_{i=0} \sum^{M}_{m=1}\cQ\left(\nabla f^{\pi^i_m}_m(x^i_t)\right), 
\end{equation*}
which allows us to write the update rule as $x_{t+1} = x_t - \tau g_t$.

\begin{lemma}[Lemma 1 from ~\citep{malinovsky22_server_side_steps_sampl_without}]
	\label{Sampling_without_replacement}
	For any $k \in [n]$, let $\xi_{\pi_1}, \dots,\xi_{\pi_k}$ be sampled uniformly without replacement from a set of vectors $\{\xi_1, \dots, \xi_n\}$ and $\bar{\xi}_{\pi}$ be their average. Then	
	\begin{equation}
		\EE \bar{\xi}_{\pi} = \bar{\xi},\qquad \EE\left[\|\bar{\xi}_{\pi} - \bar{\xi}\|^2\right] = \frac{n-k}{k(n-1)}\sigma^2, 
	\end{equation}
	where $\bar{\xi} \eqdef \frac{1}{n}\sum^n_{i=1}\xi_i$, $\bar{\xi}_{\pi} \eqdef \frac{1}{k}\sum^k_{i=1}\xi_{\pi_i}$, and $\sigma^2 \eqdef \frac{1}{n}\sum^n_{i=1}\|\xi_i - \bar{\xi}\|^2$.
\end{lemma}

\begin{lemma}
	\label{lem_inner_product_cvx_case} 
	Under Assumptions \ref{asm:quantization_operators}, \ref{asm:sc_general_f},  \ref{asm:lip_max_f_m}, \ref{asm:lip_avr_f}, the following inequality holds
	\begin{equation*}
		\EE_{\cQ}\left[-2\tau\la g_t, x_t-x_{\star}\ra\right] \leq -\frac{\tau\mu}{2}\|x_t-x_{\star}\|^2-\tau(f(x_t)- f(x_{\star}))+\frac{\tau\widetilde{L}}{n}\sum^{n-1}_{i=0}\|x^i_t-x_t\|^2.
	\end{equation*}
\end{lemma}
\begin{proof}Using that $\EE_{\cQ}\left[g_t\right] = \frac{1}{Mn}\sum^{n-1}_{i=0} \sum^{M}_{m=1}\nabla f^{\pi^i_m}_m(x^i_t) $ and definition of $h^{\star}$, we get
		\begin{eqnarray*}
			-2\tau\EE_{\cQ}\left[\left\la g_t, x_t - x_{\star}\right\ra\right] 
			&=& -\frac{1}{Mn}\sum^{n-1}_{i=0} \sum^{M}_{m=1}\left\la \nabla f^{\pi^i_m}_m(x^i_t), x_t - x_{\star}\right\ra\\
			&=& -\frac{1}{Mn}\sum^M_{m=1}\sum^{n-1}_{i=0}\left\la \nabla f^{\pi^i_m}_m(x^i_{t}) - \nabla f^{\pi^i_m}_m(x_{\star}), x_t - x_{\star}\right\ra.
		\end{eqnarray*}
		
		Using three-point identity, we obtain 
		\begin{eqnarray*}
			-2\tau\EE_{\cQ}\left[\left\la g_t , x_t - x_{\star}\right\ra\right] 
			&=& -\frac{2\tau}{Mn}\sum^M_{m=1}\sum^{n-1}_{i=0}\left( D_{f^{\pi^i_m}_m}(x_t, x_{\star})+ D_{f^{\pi^i_m}_m}( x_{\star}, x^{i}_{t}) - D_{f^{\pi^i_m}_m}(x_t, x^{i}_{t})\right)\\
			&=& - 2\tau D_{f}(x_t, x_{\star}) - \frac{2\tau}{n}\sum^{n-1}_{i=0} D_{f^{\pi^i}}(x_{\star}, x^i_{t})+ \frac{2\tau}{n}\sum^{n-1}_{i=0} D_{f^{\pi^i}}(x_{t}, x^i_{t})\\
			&\leq& - 2\tau D_{f}(x_t, x_{\star}) + \frac{\tau\widetilde{L}}{n}\sum^{n-1}_{i=0} \| x^i_{t} - x_{t} \|^2,
		\end{eqnarray*}
		where in the last inequality we apply $\widetilde{L}$-smoothness  and convexity of each function $f^{\pi^i}$. Finally, using $\mu$-strong convexity of $f$, we finish the proof of the lemma.
\end{proof}

\begin{lemma}
	\label{lem_norm_grad_cvx_case}
	Under Assumptions \ref{asm:quantization_operators}, \ref{asm:sc_general_f},  \ref{asm:lip_max_f_m}, \ref{asm:lip_avr_f}, the following inequality holds
	\begin{eqnarray*}
		\EE_{\cQ}\left[\|g_t\|^2\right] 
		&\leq& 2\widetilde{L}\left(\widetilde{L} + \frac{\omega}{Mn}L_{\max}\right)\frac{1}{n}\sum^{n-1}_{i=0}\EE\left[\|x^i_t-x_t\|^2\right]+ \frac{4\omega}{Mn}\left(\zeta^2_{\star}+\sigma^2_{\star}\right)\\
		&&+8\left(\widetilde{L}+\frac{\omega}{Mn}L_{\max}\right)(f(x_t)-f(x_{\star})).
	\end{eqnarray*}
\end{lemma}
\begin{proof}
	Taking the expectation with respect to $\cQ$ and using variance decomposition $\EE\left[\|\xi\|^2\right] = \EE\left[\|\xi-\EE\left[\xi\right]\|^2\right] +\|\EE\xi\|^2$, we get 
	\begin{eqnarray*}
		\EE_{\cQ}\left[\|g_t\|^2\right]
		&=& \EE_{\cQ}\left[\left\|\frac{1}{Mn}\sum^{n-1}_{i=0} \sum^{M}_{m=1}\cQ\left(\nabla f^{\pi^i_m}_m(x^i_t)\right)\right\|^2\right]\\
		&=& \EE_{\cQ}\left[\left\|\frac{1}{Mn}\sum^{n-1}_{i=0} \sum^{M}_{m=1}\left(\cQ\left(\nabla f^{\pi^i_m}_m(x^i_t)\right) - \nabla f^{\pi^i_m}_m(x^i_t)\right)\right\|^2\right]\\
		&&+ \left\|\frac{1}{Mn}\sum^{n-1}_{i=0} \sum^{M}_{m=1}\nabla f^{\pi^i_m}_m(x^i_t)\right\|^2.
	\end{eqnarray*}
	Next, Assumption \ref{asm:quantization_operators} and conditional independence of $\cQ\left(\nabla f^{\pi^i_m}_m(x^i_t)\right)$ for $m = 1,\ldots, M$ and $i = 0,\ldots, n-1$ imply
	\begin{eqnarray*}
		\EE_{\cQ}\left[\|g_t\|^2\right] &=& \frac{1}{M^2n^2}\sum^{n-1}_{i=0} \sum^{M}_{m=1}\EE_{\cQ}\left[\left\|\cQ\left(\nabla f^{\pi^i_m}_m(x^i_t)\right) - \nabla f^{\pi^i_m}_m(x^i_t)\right\|^2\right]\\
		&&+ \left\|\frac{1}{Mn}\sum^{n-1}_{i=0} \sum^{M}_{m=1}\nabla f^{\pi^i_m}_m(x^i_t)\right\|^2\\
		&\leq& \frac{\omega}{M^2n^2} \sum^{n-1}_{i=0} \sum^{M}_{m=1}\left\|\nabla f^{\pi^i_m}_m(x^i_t)\right\|^2 + \left\|\frac{1}{Mn}\sum^{n-1}_{i=0} \sum^{M}_{m=1}\nabla f^{\pi^i_m}_m(x^i_t)\right\|^2\\
		&\leq& \frac{2\omega}{M^2n^2} \sum^{n-1}_{i=0} \sum^{M}_{m=1}\left\|\nabla f^{\pi^i_m}_m(x^i_t) - \nabla f^{\pi^i_m}_m(x_t)\right\|^2 + \frac{2\omega}{M^2n^2} \sum^{n-1}_{i=0} \sum^{M}_{m=1}\left\|\nabla f^{\pi^i_m}_m(x_t)\right\|^2\\
		&&+ 2\left\|\frac{1}{Mn}\sum^{n-1}_{i=0} \sum^{M}_{m=1}\left(\nabla f^{\pi^i_m}_m(x^i_t) - \nabla f^{\pi^i_m}_m(x_t)\right)\right\|^2 + 2\left\|\frac{1}{Mn}\sum^{n-1}_{i=0} \sum^{M}_{m=1}\nabla f^{\pi^i_m}_m(x_t)\right\|^2.
	\end{eqnarray*}
	Using $L_{\max}$-smoothness and convexity of $f^i_m$ and $\widetilde{L}$-smoothness and convexity of $f^{\pi^i} = \frac{1}{M}\sum_{m=1}^M f_m^{\pi_m^i}$, we derive 
	\begin{eqnarray*}
		\EE_{\cQ}\left[\|g_t\|^2\right]&\leq& \frac{4\omega}{M^2n^2}L_{\max} \sum^{n-1}_{i=0} \sum^{M}_{m=1}D_{f^{\pi^i_m}_m}(x^i_t, x_t) + \frac{2\omega}{M^2n^2} \sum^{n-1}_{i=0} \sum^{M}_{m=1}\left\|\nabla f^{\pi^i_m}_m(x_t)\right\|^2\\
		&&+ 4\widetilde{L}\frac{1}{n}\sum^{n-1}_{i=0}D_{f^{\pi^i}}(x^i_t,x_t) + 2\left\|\nabla f(x_t)\right\|^2\\
		&\leq&4\left(\widetilde{L} + \frac{\omega}{Mn}L_{\max}\right) \frac{1}{n}\sum^{n-1}_{i=0} D_{f^{\pi^i}}(x^i_t, x_t)+ \frac{4\omega}{M^2n^2} \sum^{n-1}_{i=0} \sum^{M}_{m=1}\left\|\nabla f^{\pi^i_m}_m(x_{\star})\right\|^2  \\
		&& + \frac{4\omega}{M^2n^2} \sum^{n-1}_{i=0} \sum^{M}_{m=1}\left\|\nabla f^{\pi^i_m}_m(x_t) - \nabla f^{\pi^i_m}_m(x_{\star})\right\|^2 + 2\left\|\nabla f(x_t) - \nabla f(x_{\star})\right\|^2\\
		&\leq& 2\widetilde{L}\left(\widetilde{L} + \frac{\omega}{Mn}L_{\max}\right)\frac{1}{n}\sum^{n-1}_{i=0}\left\| x^i_t -  x_t\right\|^2+ \frac{4\omega}{M^2n^2} \sum^{n-1}_{i=0} \sum^{M}_{m=1}\left\|\nabla f^{\pi^i_m}_m(x_{\star})\right\|^2  \\
		&& + \frac{8\omega}{M^2n^2}L_{\max} \sum^{n-1}_{i=0} \sum^{M}_{m=1}D_{f^{\pi^i_m}_m}(x_t, x_{\star}) + 4\widetilde{L}\left( f(x_t) -  f(x_{\star})\right).
	\end{eqnarray*}
	Taking the full expectation, we obtain
	\begin{eqnarray*}
		\EE\left[\|g_t\|^2\right]
		&\leq& 2\widetilde{L}\left(\widetilde{L} + \frac{\omega}{Mn}L_{\max}\right)\frac{1}{n}\sum^{n-1}_{i=0}\EE\left[\left\| x^i_t -  x_t\right\|^2\right]+ \frac{4\omega}{M^2n^2} \sum^{n-1}_{i=0} \sum^{M}_{m=1}\EE\left[\left\|\nabla f^{\pi^i_m}_m(x_{\star})\right\|^2\right]  \\
		&& +\left(4\widetilde{L}+ \frac{8\omega}{Mn}L_{\max}\right)\EE\left[f(x_t) -  f(x_{\star})\right]\\
		&=& 2\widetilde{L}\left(\widetilde{L} + \frac{\omega}{Mn}L_{\max}\right)\frac{1}{n}\sum^{n-1}_{i=0}\EE\left[\left\| x^i_t -  x_t\right\|^2\right]+ \frac{4\omega}{Mn} \left(\zeta_{\star}^2 + \sigma_{\star}^2\right) \\
		&& +\left(4\widetilde{L}+ \frac{8\omega}{Mn}L_{\max}\right)\EE\left[f(x_t) -  f(x_{\star})\right].
	\end{eqnarray*}
\end{proof}

\begin{lemma}
	\label{lem_dinst_cvx_case} 
	Let Assumptions \ref{asm:quantization_operators}, \ref{asm:sc_general_f},  \ref{asm:lip_max_f_m}, \ref{asm:lip_avr_f} hold and $\tau \leq \frac{1}{2\sqrt{\widetilde{L}\left(\widetilde{L}+\frac{\omega}{Mn}L_{\max}\right)}}$. Then, the following inequality holds
	\begin{eqnarray*}
		\frac{1}{n}\sum^{n-1}_{i=0}\EE\left[\|x^i_t-x_t\|^2\right] 
		&\leq& 24\tau^2\left(\widetilde{L}+\frac{\omega}{Mn}L_{\max}\right)\EE\left[f(x_t)-f(x_{\star})\right] \\
		&&+ 8\tau^2\frac{\omega}{Mn}\left(\zeta^2_{\star}+\sigma^2_{\star}\right) +8\tau^2\frac{\sigma^2_{\star,n}}{n},
	\end{eqnarray*}
	where $\sigma^2_{\star,n} \eqdef \frac{1}{n}\sum^n_{i=1}\|\nabla f^i(x_{\star})\|^2$, and $f^i(x) \eqdef \frac{1}{M}\sum_{m=1}^M f_m^i(x)$ for all $i \in [n]$.
\end{lemma}
\begin{proof}
	Since $$x^i_t= x_t -\frac{\tau}{Mn}\sum^M_{m=1}\sum^{i-1}_{j=0}\cQ\left(\nabla f^{\pi^{j}_m}_m(x^j_t)\right),$$ we have 
	\begin{eqnarray*}
		\EE_{\cQ}\left[\|x^i_t -x_t\|^2\right]
		&=& \tau^2\EE_{\cQ}\left[\left\|\frac{1}{Mn}\sum^M_{m=1}\sum^{i-1}_{j=0}\cQ\left(\nabla f^{\pi^{j}_m}_m(x^j_t)\right)\right\|^2\right]\\
		&=&\tau^2\EE_{\cQ}\left[\left\|\frac{1}{Mn}\sum^M_{m=1}\sum^{i-1}_{j=0}\left(\cQ\left(\nabla f^{\pi^{j}_m}_m(x^j_t)\right) - \nabla f^{\pi^{j}_m}_m(x^j_t)\right)\right\|^2\right]\\
		&&+ \tau^2\left\|\frac{1}{Mn}\sum^M_{m=1}\sum^{i-1}_{j=0}\nabla f^{\pi^{j}_m}_m(x^j_t)\right\|^2\\
		&\leq& \frac{\tau^2}{M^2n^2}\sum^M_{m=1}\sum^{i-1}_{j=0}\EE_{\cQ}\left[\left\|\cQ\left(\nabla f^{\pi^{j}_m}_m(x^j_t)\right) - \nabla f^{\pi^{j}_m}_m(x^j_t)\right\|^2\right]\\
		&&+ \tau^2\left\|\frac{1}{Mn}\sum^M_{m=1}\sum^{i-1}_{j=0}\nabla f^{\pi^{j}_m}_m(x^j_t)\right\|^2.
	\end{eqnarray*}
	Using Assumption \ref{asm:quantization_operators}, $\widetilde{L}$-smoothness  and convexity of $f^{\pi^i} = \frac{1}{M}\sum_{m=1}^M f_m^{\pi_m^i}$ and $L_{\max}$-smoothness  and convexity of $f^{i}_m$, we obtain 
	\begin{eqnarray}
		\EE_{\cQ}\left[\|x^i_t -x_t\|^2\right]
		&\leq& \frac{\tau^2\omega}{M^2n^2}\sum^M_{m=1}\sum^{i-1}_{j=0}\left\|\nabla f^{\pi^{j}_m}_m(x^j_t)\right\|^2 + \tau^2\left\|\frac{1}{Mn}\sum^M_{m=1}\sum^{i-1}_{j=0}\nabla f^{\pi^{j}_m}_m(x^j_t)\right\|^2\notag\\
		&\leq& \frac{2\tau^2\omega}{M^2n^2}\sum^M_{m=1}\sum^{i-1}_{j=0}\left\|\nabla f^{\pi^{j}_m}_m(x^j_t) - \nabla f^{\pi^{j}_m}_m(x_t)\right\|^2 + 2\tau^2\left\|\frac{1}{n}\sum^{i-1}_{j=0}\nabla f^{\pi^j}(x_t)\right\|^2\notag\\
		&& + 2\tau^2\left\|\frac{1}{n}\sum^{i-1}_{j=0}\left(\nabla f^{\pi^j}(x^j_t)-\nabla f^{\pi^j}(x_t)\right)\right\|^2 + \frac{2\tau^2\omega}{M^2n^2}\sum^M_{m=1}\sum^{i-1}_{j=0}\left\|\nabla f^{\pi^{j}_m}_m(x_t)\right\|^2\notag \\
		&\leq& \frac{4\tau^2\omega}{M^2n^2}\sum^M_{m=1}\sum^{n-1}_{j=0}L_{\max}D_{f^{\pi_m^j}_m}(x^j_t,x_t) + 2\tau^2\left\|\frac{1}{n}\sum^{i-1}_{j=0}\nabla f^{\pi^j}(x_t)\right\|^2\notag \\
		&&+ 4\widetilde{L}\tau^2\frac{1}{n}\sum^{n-1}_{j=0}D_{f^{\pi^j}}(x^j_t, x_t)  + \frac{2\tau^2\omega}{M^2n^2}\sum^M_{m=1}\sum^{n-1}_{j=0}\left\|\nabla f^{\pi^{j}_m}_m(x_t)\right\|^2\notag\\
		&=& 4\tau^2\left(\widetilde{L} + \frac{\omega}{Mn}L_{\max}\right)\frac{1}{n}\sum^{n-1}_{j=0}D_{f^{\pi^j}}(x^j_t, x_t)\notag\\
		&&+ 2\tau^2\left\|\frac{1}{n}\sum^{i-1}_{j=0}\nabla f^{\pi^j}(x_t)\right\|^2 + \frac{2\tau^2\omega}{M^2n^2}\sum^M_{m=1}\sum^{n-1}_{j=0}\left\|\nabla f^{\pi^{j}_m}_m(x_t)\right\|^2. \label{eq:cjshjdjsdbhs}
	\end{eqnarray}
	
	Next, we need to estimate the second term from the previous inequality. Taking the full expectation and using Lemma \ref{Sampling_without_replacement}  and using new notation $\sigma_t^2 \eqdef \frac{1}{n}\sum_{j=1}^n \EE[\|\nabla f^j (x_t) - \nabla f(x_t)\|^2]$, we get 
	\begin{eqnarray}
		\EE\left[\left\|\frac{1}{n}\sum^{i-1}_{j=0}\nabla f^{\pi^j}(x_t)\right\|^2 \right]&=& \frac{i^2}{n^2}\EE\left[\|\nabla f(x_t)\|^2\right] + \frac{i^2}{n^2}\EE\left[\left\|\frac{1}{i}\sum^{i-1}_{j=0}\left(\nabla f^{\pi^j}(x_t) - \nabla f(x_t)\right)\right\|^2\right]\notag\\
		&\leq& \frac{i^2}{n^2}\EE\left[\|\nabla f(x_t)\|^2\right] +\frac{i^2}{n^3}\frac{n-i}{i(n-1)}\sum^n_{j=1}\EE\left[\|\nabla f^j(x_t) - \nabla f(x_t)\|^2\right]\notag \\
		&\leq& \EE\left[\|\nabla f(x_t)\|^2\right] +\frac{1}{n}\sigma^2_{t}. \label{eq:kdjkcnsdbciscsnd}
	\end{eqnarray}
	Taking the full expectation from \eqref{eq:cjshjdjsdbhs} and using \eqref{eq:kdjkcnsdbciscsnd}, we obtain 
	\begin{eqnarray*}
		\EE\left[\|x^i_t -x_t\|^2\right]
		&\leq& 4\tau^2\left(\widetilde{L} + \frac{\omega}{Mn}L_{\max}\right)\sum^{n-1}_{j=0}\EE\left[D_{f^{\pi^j}}(x^j_t, x_t)\right]\\
		&&+ 2\tau^2\EE\left[\|\nabla f(x_t)\|^2\right] +\frac{2\tau^2}{n}\sigma^2_{t} + \frac{2\tau^2\omega}{M^2n^2}\sum^M_{m=1}\sum^{n-1}_{j=0}\EE\left[\left\|\nabla f^{\pi^{j}_m}_m(x_t)\right\|^2\right].
	\end{eqnarray*}
	Using $\widetilde{L}$-smoothness of $f^{\pi^j}$, we get
	\begin{eqnarray*}
		\EE\left[\|x^i_t -x_t\|^2\right]
		&\leq& 2\widetilde{L}\tau^2\left(\widetilde{L} + \frac{\omega}{Mn}L_{\max}\right)\sum^{n-1}_{j=0}\EE\left[\|x^j_t - x_t\|^2\right]\\
		&&+ 2\tau^2\EE\left[\|\nabla f(x_t)\|^2\right] +\frac{2\tau^2}{n}\sigma^2_{t} + \frac{2\tau^2\omega}{M^2n^2}\sum^M_{m=1}\sum^{n-1}_{j=0}\EE\left[\left\|\nabla f^{\pi^{j}_m}_m(x_t)\right\|^2\right].
	\end{eqnarray*}
	Since $\tau \leq\frac{1}{2\sqrt{\widetilde{L}\left(\widetilde{L}+\frac{\omega}{Mn}L_{\max}\right)}}$, we have 
	\begin{eqnarray*}
		\EE\left[\|x^i_t -x_t\|^2\right]
		&\leq& 2\left(1-2\widetilde{L}\tau^2\left(\widetilde{L} + \frac{\omega}{Mn}L_{\max}\right)\right)\sum^{n-1}_{j=0}\EE\left[\|x^j_t - x_t\|^2\right]\\
		&\leq& 4\tau^2\EE\left[\|\nabla f(x_t)\|^2\right] +\frac{4\tau^2}{n}\sigma^2_{t} + \frac{4\tau^2\omega}{M^2n^2}\sum^M_{m=1}\sum^{n-1}_{j=0}\EE\left[\left\|\nabla f^{\pi^{j}_m}_m(x_t)\right\|^2\right]\\
		&\leq& \frac{8\tau^2\omega}{M^2n^2}\sum^M_{m=1}\sum^{n-1}_{j=0}\EE\left[\left\|\nabla f^{\pi^{j}_m}_m(x_t) - \nabla f^{\pi^{j}_m}_m(x_{\star})\right\|^2\right] \\
		&& + \frac{8\tau^2\omega}{M^2n^2}\sum^M_{m=1}\sum^{n-1}_{j=0}\EE\left[\left\|\nabla f^{\pi^{j}_m}_m(x_{\star})\right\|^2\right] + 4\tau^2\EE\left[\|\nabla f(x_t) - \nabla f(x_{\star})\|^2\right] \\
		&& + \frac{4\tau^2}{n}\left(\frac{1}{n}\sum^{n}_{j=1}\EE\left[\|\nabla f^j(x_t)\|\right] -\EE\left[\|\nabla f(x_t)\|^2\right]\right)\\
		&\leq& \frac{8\tau^2\omega}{M^2n^2}\sum^M_{m=1}\sum^{n-1}_{j=0}\EE\left[\left\|\nabla f^{\pi^{j}_m}_m(x_t) - \nabla f^{\pi^{j}_m}_m(x_{\star})\right\|^2\right] \\
		&& + \frac{8\tau^2\omega}{M^2n^2}\sum^M_{m=1}\sum^{n-1}_{j=0}\EE\left[\left\|\nabla f^{\pi^{j}_m}_m(x_{\star})\right\|^2\right] + 8\tau^2\EE\left[\|\nabla f(x_t) - \nabla f(x_{\star})\|^2\right]\\
		&& + \frac{8\tau^2}{n^2}\sum^{n}_{j=1}\EE\left[\|\nabla f^j(x_t) - \nabla f^j(x_{\star})\|^2\right] + \frac{8\tau^2}{n^2}\sum^{n}_{j=1}\EE\left[\|\nabla f^j(x_{\star})\|^2\right].
	\end{eqnarray*}
	Summing from $i=0$ to $n-1$ and using $\widetilde{L}$-smoothness of $f^i$ and $L_{\max}$-smoothness of $f^i_m$, we obtain 
	\begin{eqnarray*}
		\frac{1}{n}\sum^{n-1}_{i=0}\EE\left[\|x^i_t -x_t\|^2\right]
		&\leq& \frac{16\tau^2\omega}{Mn}L_{\max}\EE\left[f(x_t) - f(x_{\star})\right] + \frac{16\tau^2}{n}\widetilde{L}\EE\left[ f(x_t) -  f(x_{\star})\right] \\
		&& + \frac{8\tau^2\omega}{Mn}\left(\zeta_{\star}^2+\sigma^2_{\star}\right) + \frac{8\tau^2}{n}\sigma^2_{\star, n}+ 8\tau^2\widetilde{L}\EE\left[f(x_t) - f(x_{\star})\right].
	\end{eqnarray*}
\end{proof}

\begin{theorem}
	\label{th_convergence_Q_RR_cvx_case}
	Let Assumptions \ref{asm:quantization_operators}, \ref{asm:sc_general_f},  \ref{asm:lip_max_f_m}, \ref{asm:lip_avr_f} hold and stepsize $\gamma$ satisfy
	\begin{equation}
		\label{stepsize_q_rr_cvx_case}
		0 < \gamma \leq \frac{1}{16n\left(\widetilde{L}+\frac{\omega}{Mn}L_{\max}\right)}.
	\end{equation}
	Then, for all $T \geq 0$ the iterates produced by \algname{Q-RR} (Algorithm~\ref{alg_new_Q_RR}) satisfy
	\begin{eqnarray*}
		\EE\left[\|x_T-x_{\star}\|^2\right]
		&\leq& \left(1-\frac{n\gamma\mu}{2}\right)^T\|x_0 - x_{\star}\|^2 +18\frac{\gamma^2n\widetilde{L}}{\mu}\left(\frac{\omega}{M}(\zeta^2_{\star}+\sigma_{\star}^2) +\sigma^2_{\star,n}\right)\\
		&& +  8\frac{\gamma\omega}{\mu M}(\zeta^2_{\star}+\sigma_{\star}^2),
	\end{eqnarray*}
	where 
	\begin{equation}
		\label{def_sigma_n}
		\sigma^2_{\star,n} \eqdef \frac{1}{n}\sum^n_{i=1}\|\nabla f^i(x_{\star})\|^2.
	\end{equation}
\end{theorem}
\begin{proof}
	Taking expectation w.r.t.\ $\cQ$ and using Lemma \ref{lem_norm_grad_cvx_case}, we get 
	\begin{eqnarray*}
		\EE_{\cQ}\left[\|x_{t+1} - x_{\star}\|^2\right] 
		&=& \|x_t - x_{\star}\|^2 -2\tau\EE_{\cQ}\left[\la g_t, x_t- x_{\star}\ra\right] + \tau^2\EE_{\cQ}\left[\|g^t\|^2\right]\\
		&\leq& \|x_t - x_{\star}\|^2 -2\tau\EE_{\cQ}\left[\left\la g^t, x_t- x_{\star}\right\ra\right]\\
		&& + 2\tau^2\widetilde{L}\left(\widetilde{L} + \frac{\omega}{Mn}L_{\max}\right)\frac{1}{n}\sum^{n-1}_{i=0}\EE\left[\|x^i_t-x_t\|^2\right]\\
		&&+ 8\tau^2\left(\widetilde{L}+\frac{\omega}{Mn}L_{\max}\right)(f(x_t)-f(x_{\star})) + \frac{4\tau^2\omega}{Mn}(\zeta^2_{\star} + \sigma_{\star}^2).
	\end{eqnarray*}
	Using Lemma \ref{lem_inner_product_cvx_case}, we obtain 
	\begin{eqnarray*}
		\EE_{\cQ}\left[\|x_{t+1} - x_{\star}\|^2\right] 
		&\leq& \|x_t - x_{\star}\|^2 -\frac{\tau\mu}{2}\|x_t-x_{\star}\|^2-\tau(f(x_t)- f(x_{\star}))+\frac{\tau\widetilde{L}}{n}\sum^{n-1}_{i=0}\|x^i_t-x_t\|^2\\
		&& + 2\tau^2\widetilde{L}\left(\widetilde{L} + \frac{\omega}{Mn}L_{\max}\right)\frac{1}{n}\sum^{n-1}_{i=0}\EE\left[\|x^i_t-x_t\|^2\right]\\
		&&+ 8\tau^2\left(\widetilde{L}+\frac{\omega}{Mn}L_{\max}\right)(f(x_t)-f(x_{\star})) + \frac{4\tau^2\omega}{Mn}(\zeta^2_{\star} + \sigma_{\star}^2)\\
		&\leq& \left(1 -\frac{\tau\mu}{2}\right)\|x_t-x_{\star}\|^2-\tau\left(1- 8\tau\left(\widetilde{L}+\frac{\omega}{Mn}L_{\max}\right)\right)(f(x_t)- f(x_{\star}))\\
		&& + \tau\widetilde{L}\left(1+ 2\tau\left(\widetilde{L} + \frac{\omega}{Mn}L_{\max}\right)\right)\frac{1}{n}\sum^{n-1}_{i=0}\EE\left[\|x^i_t-x_t\|^2\right] + \frac{4\tau^2\omega}{Mn}(\zeta^2_{\star} + \sigma_{\star}^2).
	\end{eqnarray*}
	Next, we take the full expectation and apply Lemma \ref{lem_dinst_cvx_case}: 
	\begin{eqnarray*}
		\EE\left[\|x_{t+1} - x_{\star}\|^2\right] 
		&\leq& \left(1 -\frac{\tau\mu}{2}\right)\EE\left[\|x_t-x_{\star}\|^2\right]-\tau\left(1- 8\tau\left(\widetilde{L}+\frac{\omega}{Mn}L_{\max}\right)\right)\EE\left[f(x_t)- f(x_{\star})\right]\\
		&&+ 24\tau^3\widetilde{L}\left(1+ 2\tau\left(\widetilde{L} + \frac{\omega}{Mn}L_{\max}\right)\right)\left(\widetilde{L}+\frac{\omega}{Mn}L_{\max}\right)(f(x_t)-f(x_{\star})) \\
		&&+ 8\tau^3\widetilde{L}\left(1+ 2\tau\left(\widetilde{L} + \frac{\omega}{Mn}L_{\max}\right)\right)\left(\frac{\omega}{Mn}(\zeta^2_{\star}+\sigma_{\star}^2) +\frac{\sigma^2_{\star,n}}{n}\right)+ \frac{4\tau^2\omega}{Mn}(\zeta^2_{\star} + \sigma_{\star}^2).
	\end{eqnarray*}
	Using \eqref{stepsize_q_rr_cvx_case}, we derive
	\begin{eqnarray*}
		\EE\left[\|x_{t+1} - x_{\star}\|^2\right] 
		\leq \left(1 -\frac{\tau\mu}{2}\right)\EE\left[\|x_t-x_{\star}\|^2\right] + 9\tau^3\widetilde{L}\left(\frac{\omega}{Mn}(\zeta^2_{\star} + \sigma_{\star}^2) +\frac{\sigma^2_{\star,n}}{n}\right)+ \frac{4\tau^2\omega}{Mn}(\zeta^2_{\star} + \sigma_{\star}^2).
	\end{eqnarray*}
	Recursively unrolling the inequality, substituting $\tau=n\gamma$ and using $\sum\limits^{+\infty}_{t = 0}\left(1-\frac{\tau\mu}{2}\right)^t \leq \frac{2}{\mu\tau}$, we get the result.
\end{proof}

\begin{corollary}
	\label{cor_convergence_new_Q_RR_cvx_case}
	Let the assumptions of Theorem~\ref{th_convergence_Q_RR_cvx_case} hold and
	\begin{equation}
		\label{new_gamma_new_Q_RR_cvx_case}
		\gamma = \min\left\{\frac{1}{16n\left(\widetilde{L}+\frac{\omega}{Mn}L_{\max}\right)}, \sqrt{\frac{\varepsilon\mu}{8^2n\widetilde{L}}}\left(\frac{\omega}{M}\Delta^2_{\star} +\sigma^2_{\star,n}\right)^{-\frac{1}{2}}, \frac{\varepsilon\mu M}{24\omega\Delta^2_{\star}}\right\},
	\end{equation}
	where $\Delta^2_{\star} \eqdef \zeta^2_{\star}+\sigma^2_{\star}$. Then, \algname{Q-RR} (Algorithm~\ref{alg_new_Q_RR}) finds a solution with accuracy $\varepsilon>0$ after the following number of communication rounds: 
	\begin{equation}
		\widetilde{\cO}\left(\frac{n\widetilde{L}}{\mu}+\frac{\omega}{M}\frac{L_{\max}}{\mu}+ \frac{\omega}{M}\frac{\zeta^2_{\star}+\sigma^2_{\star}}{\varepsilon\mu^2} + \sqrt{\frac{n\widetilde{L}}{\varepsilon\mu^3}}\sqrt{\frac{\omega}{M}\left(\zeta^2_{\star}+\sigma^2_{\star}\right) +\sigma^2_{\star,n}}\right).\notag
	\end{equation}
\end{corollary}
\begin{proof}
	Theorem~\ref{th_convergence_Q_RR_cvx_case} implies
	\begin{eqnarray*}
		\EE\left[\|x_T-x_{\star}\|^2\right]
		&\leq& \left(1-\frac{n\gamma\mu}{2}\right)^T\|x_0 - x_{\star}\|^2 +18\frac{\gamma^2n\widetilde{L}}{\mu}\left(\frac{\omega}{M}\left(\zeta^2_{\star}+\sigma^2_{\star}\right) +\sigma^2_{\star,n}\right)\\
		&&+  8\frac{\gamma\omega}{\mu M}\left(\zeta^2_{\star}+\sigma^2_{\star}\right).
	\end{eqnarray*}
	To estimate the number of communication rounds required to find a solution with accuracy $\varepsilon >0$, we need to upper bound each term from the right-hand side by $\nicefrac{\varepsilon}{3}$. Thus, we get additional conditions on $\gamma$:
	\begin{equation*}
		18\frac{\gamma^2n\widetilde{L}}{\mu}\left(\frac{\omega}{M}\left(\zeta^2_{\star}+\sigma^2_{\star}\right) +\sigma^2_{\star,n}\right) < \frac{\varepsilon}{3},\qquad 8\frac{\gamma\omega}{\mu M}\left(\zeta^2_{\star}+\sigma^2_{\star}\right)<\frac{\varepsilon}{3},
	\end{equation*}
	and also the upper bound on the number of communication rounds $nT$
	\begin{equation*}
		nT = \widetilde{\cO}\left(\frac{1}{\gamma\mu}\right).
	\end{equation*}
Substituting \eqref{new_gamma_new_Q_RR_cvx_case} in the previous equation, we get the result. 
\end{proof}

\newpage

\section{Missing Proofs For DIANA-RR}

\subsection{Proof of Theorem \ref{th_conv_rr_diana}}

\begin{lemma}
\label{lem_rr_diana_conv_h}
Let Assumptions \ref{asm:quantization_operators}, \ref{asm:lip_max_f_m}, \ref{asm:lip_avr_f}, \ref{asm:sc_each_f_m_refined} hold  and $\alpha\leq\frac{1}{1+\omega}$. Then, the iterates of \algname{DIANA-RR} (Algorithm~\ref{alg_new_RR_DIANA}) satisfy
\begin{eqnarray*}
	\frac{1}{M}\sum^M_{m=1}\EE_{\cQ}\left[\|h^{\pi^i_m}_{t+1,m} - \nabla f^{\pi^i_m}_m(x_{\star})\|^2\right] &\leq& \frac{1-\alpha}{M}\sum^M_{m=1}\|h^{\pi^i_m}_{t,m} -\nabla f^{\pi^i_m}_m(x_{\star})\|^2\\
	&& + \frac{2\alpha L_{\max}}{M}\sum^M_{m=1} D_{f^{\pi^{i}_m}_m}(x^{i}_{t},x_{\star}).
\end{eqnarray*}
\end{lemma}
\begin{proof}
Taking expectation w.r.t.\ $\cQ$, we obtain
\begin{eqnarray*}
	\EE_{\cQ}\left[\|h^{\pi^i_m}_{t+1,m} - \nabla f^{\pi^i_m}_m(x_{\star})\|^2\right] &=& \EE_{\cQ}\left[\|h^{\pi^i_m}_{t,m} +\alpha\cQ(\nabla f^{\pi^i_m}_m(x^i_{t})-h^{\pi^i_m}_{t,m})- \nabla f^{\pi^i_m}_m(x_{\star})\|^2\right]\\
	&=& \|h^{\pi^i_m}_{t,m}- \nabla f^{\pi^i_m}_m(x_{\star})\|^2\\
	&& +2\alpha\EE_{\cQ}\left[\left\la\cQ(\nabla f^{\pi^i_m}_m(x^i_{t})-h^{\pi^i_m}_{t,m}),h^{\pi^i_m}_{t,m}- \nabla f^{\pi^i_m}_m(x_{\star})\right\ra\right] \\
	&&+ \alpha^2 \EE_{\cQ}\left[\| \cQ(\nabla f^{\pi^i_m}_m(x^i_{t})-h^{\pi^i_m}_{t,m})\|^2\right]\\
	&=& \|h^{\pi^i_m}_{t,m}- \nabla f^{\pi^i_m}_m(x_{\star})\|^2 +2\alpha\left\la \nabla f^{\pi^i_m}_m(x^i_{t})-h^{\pi^i_m}_{t,m},h^{\pi^i_m}_{t,m}- \nabla f^{\pi^i_m}_m(x_{\star})\right\ra \\
	&&+ \alpha^2 \EE_{\cQ}\left[\| \cQ(\nabla f^{\pi^i_m}_m(x^i_{t})-h^{\pi^i_m}_{t,m})\|^2\right].
\end{eqnarray*}
Assumption \ref{asm:quantization_operators}, $L_{\max}$-smoothness and convexity of $f^i_m$ and $\alpha \leq \nicefrac{1}{(1+\omega)}$ imply
\begin{eqnarray}
	\EE_{\cQ}\left[\|h^{\pi^i_m}_{t+1,m} - \nabla f^{\pi^i_m}_m(x_{\star})\|^2\right]
	&\leq& \|h^{\pi^i_m}_{t,m}- \nabla f^{\pi^i_m}_m(x_{\star})\|^2  \notag\\
	&&+2\alpha\left\la \nabla f^{\pi^i_m}_m(x^i_{t})-h^{\pi^i_m}_{t,m},h^{\pi^i_m}_{t,m}- \nabla f^{\pi^i_m}_m(x_{\star})\right\ra \notag\\
	&&+ \alpha^2(1+\omega)\| \nabla f^{\pi^i_m}_m(x^i_{t}) -h^{\pi^i_m}_{t,m}\|^2 \notag\\
	&\leq& \|h^{\pi^i_m}_{t,m}- \nabla f^{\pi^i_m}_m(x_{\star})\|^2 \notag\\
	&& + \alpha\left\la\nabla f^{\pi^i_m}_m(x^i_{t}) -h^{\pi^i_m}_{t,m}, h^{\pi^i_m}_{t,m} +\nabla f^{\pi^i_m}_m(x^i_{t}) - 2\nabla f^{\pi^i_m}_m(x_{\star})\right\ra \notag \\
	&\leq& \|h^{\pi^i_m}_{t,m}- \nabla f^{\pi^i_m}_m(x_{\star})\|^2 \notag \\
	&& + \alpha\|\nabla f^{\pi^i_m}_m(x^i_{t}) - \nabla f^{\pi^i_m}_m(x_{\star})\|^2 - \alpha\|h^{\pi^i_m}_{t,m}-\nabla f^{\pi^i_m}_m(x_{\star})\|^2\notag\\
	&\leq& (1-\alpha)\|h^{\pi^i_m}_{t,m}- \nabla f^{\pi^i_m}_m(x_{\star})\|^2 \notag\\
	&&+ \alpha\|\nabla f^{\pi^i_m}_m(x^i_{t}) - \nabla f^{\pi^i_m}_m(x_{\star})\|^2 \label{eq:ndskjjdcidscuidbid}\\
	&\leq& (1-\alpha)\|h^{\pi^i_m}_{t,m}- \nabla f^{\pi^i_m}_m(x_{\star})\|^2 + 2\alpha L_{\max} D_{f^{\pi^{i}_m}_m}(x^{i}_{t},x_{\star}).\notag
\end{eqnarray}
Summing up the above inequality for $m=1, \ldots, M$, we get the result.
\end{proof}

\begin{theorem}[Theorem~\ref{th_conv_rr_diana}]
	\label{thm:advanced_conv_diana_rr_sc}
	Let Assumptions \ref{asm:quantization_operators}, \ref{asm:lip_max_f_m}, \ref{asm:lip_avr_f}, \ref{asm:sc_each_f_m_refined}  hold and $0<\gamma \leq \min\left\{\frac{\alpha}{2n\widetilde\mu}, \frac{1}{\widetilde{L}+\frac{6\omega}{M}L_{\max}}\right\}$, $\alpha \leq \frac{1}{1+\omega}$.
	Then, for all $T \geq 0$ the iterates produced by \algname{DIANA-RR} (Algorithm~\ref{alg_new_RR_DIANA}) satisfy
	\begin{equation*}
		\EE\left[\Psi_T\right] \leq \left(1-\gamma\widetilde{\mu}\right)^{nT}\Psi_0 +\frac{2\gamma^2\sigma^2_{\text{rad}}}{\widetilde\mu},
	\end{equation*}
	where $\Psi_t$ is defined in \eqref{lyapunov_func_rr_diana}.
\end{theorem}
\begin{proof}
Using  \(x^{i+1}_{\star} = x^{i}_{\star} -\frac{\gamma}{M}\sum^M_{m=1}\nabla f^{\pi^i_m}_m(x_{\star})\)  and line 9 of Algorithm \ref{alg_new_RR_DIANA}, we derive
\begin{eqnarray*}
\|x^{i+1}_t-x^{i+1}_{\star}\|^2 
&=& \left\|x^i_t-x^{i}_{\star}- \gamma\frac{1}{M}\sum^M_{m=1} \left(\hat{g}^{\pi^{i}_m}_{t,m} - \nabla f^{\pi^{i}_m}_m(x_{\star})\right)\right\|^2\\
&=& \left\|x^i_t-x^{i}_{\star}\right\|^2 - \frac{2\gamma}{M}\sum^M_{m=1}\left\la \left(\hat{g}^{\pi^{i}_m}_{t,m} - \nabla f^{\pi^{i}_m}_m(x_{\star})\right), x^i_t-x^{i}_{\star} \right\ra \\
&&+\gamma^2\left\|\frac{1}{M}\sum^M_{m=1} \left(\hat{g}^{\pi^{i}_m}_{t,m} - \nabla f^{\pi^{i}_m}_m(x_{\star})\right)\right\|^2.
\end{eqnarray*}
Taking  expectation w.r.t.\ $\cQ$ and using $\EE\|\xi - c\|^2 = \EE\|\xi - \EE\xi\|^2 + \|\EE\xi - c\|^2$, we obtain 
\begin{eqnarray*}
\EE_{\cQ}\left[\|x^{i+1}_t-x^{i+1}_{\star}\|^2 \right]
&=& \left\|x^i_t-x^{i}_{\star}\right\|^2 - \frac{2\gamma}{M}\sum^M_{m=1} \left\la\nabla f^{\pi^{i}_m}_m(x^{i}_{t}) - \nabla f^{\pi^{i}_m}_m(x_{\star}), x^i_t-x^{i}_{\star} \right\ra\\
&&+\gamma^2\EE_{\cQ}\left[\left\|\frac{1}{M}\sum^M_{m=1} \left(\cQ\left(\nabla f^{\pi^{i}_m}_m(x^{i}_{t}) - h^{\pi^{i}_m}_{t,m}\right) + h^{\pi^{i}_m}_{t,m}  - \nabla f^{\pi^{i}_m}_m(x_{\star})\right)\right\|^2\right]\\
&\leq& \left\|x^i_t-x^{i}_{\star}\right\|^2 - \frac{2\gamma}{M}\sum^M_{m=1} \left\la\nabla f^{\pi^{i}_m}_m(x^{i}_{t}) - \nabla f^{\pi^{i}_m}_m(x_{\star}), x^i_t-x^{i}_{\star} \right\ra\\
&&+\gamma^2\EE_{\cQ}\left[\left\|\frac{1}{M}\sum^M_{m=1} \left(\cQ\left(\nabla f^{\pi^{i}_m}_m(x^{i}_{t}) - h^{\pi^{i}_m}_{t,m}\right) - \nabla f^{\pi^{i}_m}_m(x^{i}_{t}) + h^{\pi^{i}_m}_{t,m}  \right)\right\|^2\right]\\
&&+\gamma^2\left\|\frac{1}{M}\sum^M_{m=1} \left( \nabla f^{\pi^{i}_m}_m(x_{\star})  - \nabla f^{\pi^{i}_m}_m(x^{i}_{t})\right)\right\|^2.
\end{eqnarray*}
Independence of $\cQ\left(\nabla f^{\pi^{i}_m}_m(x^{i}_{t}) - h^{\pi^{i}_m}_{t,m}\right)$, $m \in [M]$, Assumption \ref{asm:quantization_operators}, and three-point identity imply
\begin{eqnarray*}
\EE_{\cQ}\left[\|x^{i+1}_t-x^{i+1}_{\star}\|^2 \right]
&\leq& \left\|x^i_t-x^{i}_{\star}\right\|^2 - \frac{2\gamma}{M}\sum^M_{m=1} \left(D_{f^{\pi^{i}_m}_m}(x^{i}_{\star},x^{i}_{t}) +D_{f^{\pi^{i}_m}_m}( x^i_t, x_{\star}) -D_{f^{\pi^{i}_m}_m}(x^{i}_{\star},x_{\star})  \right)\\
&&+\frac{\gamma^2\omega}{M^2}\sum^M_{m=1}\left\|\nabla f^{\pi^{i}_m}_m(x^{i}_{t}) - h^{\pi^{i}_m}_{t,m}\right\|^2\\
&& +\gamma^2\left\|\frac{1}{M}\sum^M_{m=1} \left( \nabla f^{\pi^{i}_m}_m(x_{\star})  - \nabla f^{\pi^{i}_m}_m(x^{i}_{t})\right)\right\|^2\\
&\leq& \left\|x^i_t-x^{i}_{\star}\right\|^2 - \frac{2\gamma}{M}\sum^M_{m=1} \left(D_{f^{\pi^{i}_m}_m}(x^{i}_{\star},x^{i}_{t}) +D_{f^{\pi^{i}_m}_m}( x^i_t, x_{\star}) -D_{f^{\pi^{i}_m}_m}(x^{i}_{\star},x_{\star})  \right)\\
&&+\frac{2\gamma^2\omega}{M}\frac{1}{M}\sum^M_{m=1}\left\|\nabla f^{\pi^{i}_m}_m(x^{i}_{t}) - \nabla f^{\pi^{i}_m}_m(x_{\star})\right\|^2\\
&& + \gamma^2\left\|\frac{1}{M}\sum^M_{m=1} \left( \nabla f^{\pi^{i}_m}_m(x_{\star})  - \nabla f^{\pi^{i}_m}_m(x^{i}_{t})\right)\right\|^2\\
&&+\frac{2\gamma^2\omega}{M^2}\sum^M_{m=1}\left\|h^{\pi^{i}_m}_{t,m} - \nabla f^{\pi^{i}_m}_m(x_{\star})\right\|^2.
\end{eqnarray*}
Using $L_{\max}$-smoothness and $\mu$-strong convexity of functions $f^i_m$ and $\widetilde{L}$-smoothness and $\widetilde{\mu}$-strong convexity of $f^{\pi^i} = \frac{1}{M}\sum^M_{i=1}f^{\pi_m^i}_m$, we obtain 
\begin{eqnarray*}
\EE_{\cQ}\left[\|x^{i+1}_t-x^{i+1}_{\star}\|^2 \right]
&\leq& (1-\gamma\widetilde\mu)\left\|x^i_t-x^{i}_{\star}\right\|^2  -2\gamma\left(1 - \gamma\left(\widetilde{L}+\frac{2\omega}{M}L_{\max}\right)\right) \frac{1}{M}\sum^M_{m=1} D_{f^{\pi^i_m}_m}(x^i_t,x_{\star}) \\
&& + \frac{2\gamma}{M}\sum^M_{m=1} D_{f^{\pi^{i}_m}_m}(x^{i}_{\star},x_{\star}) +\frac{2\gamma^2\omega}{M^2}\sum^M_{m=1}\left\|h^{\pi^{i}_m}_{t,m} - \nabla f^{\pi^{i}_m}_m(x_{\star})\right\|^2.
\end{eqnarray*}
Taking the full expectation and using Definition 2, we derive
\begin{eqnarray*}
\EE\left[\|x^{i+1}_t-x^{i+1}_{\star}\|^2 \right]
&\leq& (1-\gamma\widetilde\mu)\EE\left[\left\|x^i_t-x^{i}_{\star}\right\|^2\right]  -2\gamma\left(1 - \gamma \left(\widetilde{L}+\frac{2\omega}{M}L_{\max}\right)\right) \frac{1}{M}\sum^M_{m=1}\EE\left[D_{f^{\pi^{i}_m}_m}(x^{i}_{t},x_{\star})\right]  \\
&& + 2\gamma^3\sigma^2_{\text{rad}} +\frac{2\gamma^2\omega}{M^2}\sum^M_{m=1}\EE\left[\left\|h^{\pi^{i}_m}_{t,m} - \nabla f^{\pi^{i}_m}_m(x_{\star})\right\|^2 \right].
\end{eqnarray*}
Recursively unrolling the inequality, we get
\begin{eqnarray*}
\EE\left[\|x_{t+1}-x_{\star}\|^2 \right]
&\leq& (1-\gamma\widetilde\mu)^n\EE\left[\left\|x_t-x_{\star}\right\|^2\right] +\frac{2\gamma^2\omega}{M^2}\sum^M_{m=1}\sum^{n-1}_{j=0}(1-\gamma\widetilde\mu)^j\EE\left[\left\|h^{\pi^{i}_m}_{t,m} - \nabla f^{\pi^{i}_m}_m(x_{\star})\right\|^2 \right]  \\
&& -2\gamma\left(1 - \gamma \left(\widetilde{L}+\frac{2\omega}{M}L_{\max}\right)\right) \frac{1}{M}\sum^M_{m=1}\sum^{n-1}_{j=0}(1-\gamma\widetilde\mu)^j\EE\left[D_{f^{\pi^{i}_m}_m}(x^{i}_{t},x_{\star})\right]\\
&& + 2\gamma^3\sigma^2_{\text{rad}}\sum^{n-1}_{j=0}(1-\gamma\widetilde\mu)^j.
\end{eqnarray*}
Next, we apply \eqref{lyapunov_func_rr_diana} and Lemma \ref{lem_rr_diana_conv_h}: 
\begin{eqnarray*}
\EE\left[\Psi_{t+1}\right]
&\leq& (1-\gamma\widetilde\mu)^n\EE\left[\left\|x_t-x_{\star}\right\|^2\right] + 2\gamma^3\sigma^2_{\text{rad}}\sum^{n-1}_{j=0}(1-\gamma\widetilde\mu)^j\\ &&+\left(c(1-\alpha)+\frac{2\omega}{M}\right)\frac{\gamma^2}{M}\sum^M_{m=1}\sum^{n-1}_{j=0}(1-\gamma\widetilde\mu)^j\EE\left[\left\|h^{\pi^{i}_m}_{t,m} - \nabla f^{\pi^{i}_m}_m(x_{\star})\right\|^2 \right]  \\
&& -2\gamma\left(1 -c\gamma\alpha L_{\max}-  \gamma 
\left(\widetilde{L}+\frac{2\omega}{M}L_{\max}\right) \right) \frac{1}{M}\sum^M_{m=1}\sum^{n-1}_{j=0}(1-\gamma\widetilde\mu)^j\EE\left[D_{f^{\pi^{i}_m}_m}(x^{i}_{\star},x_{\star})\right],
\end{eqnarray*}
where $c = \frac{4\omega}{\alpha M^2}$. Using $\alpha \leq \frac{1}{1+\omega}$ and $\gamma \leq \min\left\{\frac{\alpha}{2n\mu}, \frac{1}{\left(\widetilde{L}+\nicefrac{6\omega}{M}L_{\max}\right)}\right\}$, we obtain
\begin{eqnarray*}
\EE\left[\Psi_{t+1} \right]
&\leq& (1-\gamma\widetilde\mu)^n\EE\left[\left\|x_t-x_{\star}\right\|^2\right] +\left(1-\frac{\alpha}{2}\right)\frac{4\omega\gamma^2}{\alpha M^2}\sum^M_{m=1}\sum^{n-1}_{j=0}(1-\gamma\widetilde\mu)^j\EE\left[\left\|h^{\pi^{i}_m}_{t,m} - \nabla f^{\pi^{i}_m}_m(x_{\star})\right\|^2 \right] \\
&& + 2\gamma^2\sigma^3_{\text{rad}}\sum^{n-1}_{j=0}(1-\gamma\widetilde\mu)^j\\
&\leq&\max\left\{(1-\gamma\widetilde\mu)^n, \left(1-\frac{\alpha}{2}\right)\right\}\EE\left[\Psi_{t} \right] + 2\gamma^2\sigma^3_{\text{rad}}\sum^{n-1}_{j=0}(1-\gamma\widetilde\mu)^j\\
&\leq&(1-\gamma\widetilde\mu)^n\EE\left[\Psi_{t} \right] + 2\gamma^3\sigma^2_{\text{rad}}\sum^{n-1}_{j=0}(1-\gamma\widetilde\mu)^j.
\end{eqnarray*}

Recursively rewriting the inequality, we obtain 
\begin{eqnarray*}
\EE\left[\Psi_{T} \right]
&\leq&(1-\gamma\widetilde\mu)^{nT}\Psi_{0} + 2\gamma^3\sigma^2_{\text{rad}}\sum^{T-1}_{t=0}(1-\gamma\widetilde\mu)^{tn}\sum^{n-1}_{j=0}(1-\gamma\widetilde\mu)^j\\
&\leq&(1-\gamma\widetilde\mu)^{nT}\Psi_{0} + 2\gamma^3\sigma^2_{\text{rad}}\sum^{nT-1}_{k=0}(1-\gamma\widetilde\mu)^{k}
\end{eqnarray*}
Using that $\sum\limits^{+\infty}_{k = 0}\left(1-\frac{\gamma\widetilde\mu}{2}\right)^k \leq \frac{2}{\widetilde\mu\gamma}$, we finish proof.
\end{proof}

\begin{corollary}
	\label{cor_convergence_DIANA_RR}
	Let the assumptions of Theorem~\ref{thm:advanced_conv_diana_rr_sc} hold, $\alpha = \frac{1}{1+\omega}$ and
	\begin{equation}
		\label{new_gamma_DIANA_RR}
		\gamma = \min\left\{\frac{\alpha}{2n\widetilde{\mu}},\frac{1}{\widetilde{L}+\frac{6\omega}{M}L_{\max}}, \frac{\sqrt{\varepsilon\widetilde{\mu}}}{2\sigma_{\text{rad}}}\right\}.
	\end{equation}
	Then \algname{DIANA-RR} (Algorithm~\ref{alg_new_RR_DIANA}) finds a solution with accuracy $\varepsilon>0$ after the following number of communication rounds: 
	\begin{equation}
		\widetilde{\cO}\left(n(1+\omega)+\frac{\widetilde{L}}{\widetilde{\mu}}+\frac{\omega}{M}\frac{L_{\max}}{\widetilde{\mu}} + \frac{\sigma_{\text{rad}}}{\sqrt{\varepsilon\widetilde{\mu}^3}}\right).\notag
	\end{equation}
\end{corollary}
\begin{proof}
	Theorem~\ref{thm:advanced_conv_diana_rr_sc} implies
	\begin{equation*}
		\EE\left[\Psi_T\right] \leq \left(1-\gamma\widetilde{\mu}\right)^{nT}\Psi_0 +\frac{2\gamma^2\sigma^2_{\text{rad}}}{\widetilde\mu}.
	\end{equation*}
	To estimate the number of communication rounds required to find a solution with accuracy $\varepsilon >0$, we need to upper bound each term from the right-hand side by $\frac{\varepsilon}{2}$. Thus, we get an additional condition on $\gamma$:
	\begin{equation*}
		\frac{2\gamma^2\sigma^2_{\text{rad}}}{\widetilde\mu}<\frac{\varepsilon}{2},
	\end{equation*}
	and also the upper bound on the number of communication rounds $nT$
	\begin{equation*}
		nT = \widetilde{\cO}\left(\frac{1}{\gamma\mu}\right).
	\end{equation*}
Substituting \eqref{new_gamma_DIANA_RR} in the previous equation, we get the result.
\end{proof}

\subsection{Non-Strongly Convex Summands}
In this section, we provide the analysis of \algname{DIANA-RR} without using Assumptions~\ref{asm:sc_each_f_m} and \ref{asm:sc_each_f_m_refined}. We emphasize that $x^{i+1}_t = x^{i}_t - \gamma \frac{1}{M}\sum^{M}_{m=1}\hat{g}^{\pi^i_m}_{t,m}$. Then we have 
\begin{equation*}
	x_{t+1} = x_t - \gamma \sum^{n-1}_{i=0}\frac{1}{M}\sum^{M}_{m=1}\hat{g}^{\pi^i_m}_{t,m} = x_t -\tau \frac{1}{Mn}\sum^{n-1}_{i=0} \sum^{M}_{m=1}\hat{g}^{\pi^i_m}_{t,m}.
\end{equation*}
We denote $\hat{g}_t \eqdef \frac{1}{Mn}\sum^{n-1}_{i=0} \sum^{M}_{m=1}\hat{g}^{\pi^i_m}_{t,m}$.

\begin{lemma}
	\label{lem_inner_product_diana_rr_cvx_case}
	 Let Assumptions \ref{asm:quantization_operators}, \ref{asm:sc_general_f},  \ref{asm:lip_max_f_m}, \ref{asm:lip_avr_f} hold. Then, the following inequality holds
	\begin{equation*}
		-2\tau\EE_{\cQ}\left[\la\hat{g}_t- h_{\star}, x_t-x_{\star}\ra \right] \leq -\frac{\tau\mu}{2}\|x_t-x_{\star}\|^2 - \tau\left(f(x_t)-f(x_{\star})\right) + \tau\widetilde{L}\frac{1}{n}\sum^{n-1}_{i=1}\|x_t-x^i_t\|^2,
	\end{equation*}
	where $h^{\star} = \nabla f(x_{\star}) = 0$.
\end{lemma}
\begin{proof}
	Since $h^{\star} = \nabla f(x_{\star}) = 0$, the proof of Lemma \ref{lem_inner_product_diana_rr_cvx_case} is identical to the proof of Lemma \ref{lem_inner_product_cvx_case}.
\end{proof}

\begin{lemma}
	\label{lem_norm_of_grad_diana_rr_cvx_case}
	 Let Assumptions \ref{asm:quantization_operators}, \ref{asm:sc_general_f},  \ref{asm:lip_max_f_m}, \ref{asm:lip_avr_f} hold. Then, the following inequality holds
	\begin{eqnarray*}
		\EE_{\cQ}\left[\|\hat{g}_t - h_{\star}\|^2\right] &\leq& 2\widetilde{L}\left(\widetilde{L} + \frac{\omega}{Mn}L_{\max}\right)\frac{1}{n}\sum^{n-1}_{i=0}\|x^i_t-x_t\|^2 \\
		&&+ 8\left(\widetilde{L} + \frac{\omega}{Mn}L_{\max}\right)\left(f(x_t)-f(x_{\star})\right)\\
		&&+ \frac{4\omega}{M^2n^2}\sum^{n-1}_{i=0}\sum^M_{m=1}\|h^{\pi^i_m}_{t,m}-\nabla f^{\pi^i_m}_{m}(x_{\star})\|^2
	\end{eqnarray*}
\end{lemma}
\begin{proof}
	Taking expectation w.r.t. $\cQ$, we get 
	\begin{eqnarray*}
		\EE_{\cQ}\left[\left\|\hat{g}_t - h_{\star}\right\|^2\right] 
		&=& \EE_{\cQ}\left[\left\|\frac{1}{Mn}\sum^{n-1}_{i=0} \sum^{M}_{m=1}\hat{g}^{\pi^i_m}_{t,m} - h_{\star}\right\|^2\right]\\
		&=& \EE_{\cQ}\left[\left\|\frac{1}{Mn}\sum^{n-1}_{i=0} \sum^{M}_{m=1}\left(h^{\pi^i_m}_{t,m} + \cQ\left(\nabla f^{\pi^i_m}_m(x^i_{t}) - h^{\pi^i_m}_{t,m}\right)\right) - h_{\star}\right\|^2\right]\\
		&=& \EE_{\cQ}\left[\left\|\frac{1}{Mn}\sum^{n-1}_{i=0} \sum^{M}_{m=1}\left(h^{\pi^i_m}_{t,m} - \nabla f^{\pi^i_m}_m(x^i_{t}) + \cQ\left(\nabla f^{\pi^i_m}_m(x^i_{t}) - h^{\pi^i_m}_{t,m}\right)\right) \right\|^2\right]\\
		&&+ \left\|\frac{1}{Mn}\sum^{n-1}_{i=0} \sum^{M}_{m=1}\nabla f^{\pi^i_m}_m(x^i_{t})  - h_{\star}\right\|^2.
	\end{eqnarray*}
	Independence of $\cQ\left(\nabla f^{\pi^i_m}_m(x^i_{t}) - h^{\pi^i_m}_{t,m}\right)$, $m \in [M]$ and Assumption \ref{asm:quantization_operators} imply 
	\begin{eqnarray*}
		\EE_{\cQ}\left[\|\hat{g}_t - h_{\star}\|^2\right] 
		&=& \frac{1}{M^2n^2}\sum^{n-1}_{i=0} \sum^{M}_{m=1}\EE_{\cQ}\left[\left\|h^{\pi^i_m}_{t,m} - \nabla f^{\pi^i_m}_m(x^i_{t}) + \cQ\left(\nabla f^{\pi^i_m}_m(x^i_{t}) - h^{\pi^i_m}_{t,m}\right)\right\|^2\right]\\
		&&+ \left\|\frac{1}{Mn}\sum^{n-1}_{i=0} \sum^{M}_{m=1}\nabla f^{\pi^i_m}_m(x^i_{t})  - h_{\star}\right\|^2\\
		&\leq& \frac{\omega}{M^2n^2}\sum^{n-1}_{i=0} \sum^{M}_{m=1}\left\|\nabla f^{\pi^i_m}_m(x^i_{t}) -  h^{\pi^i_m}_{t,m} \right\|^2 + \left\|\frac{1}{n}\sum^{n-1}_{i=0}\nabla f^{\pi^i}(x^i_{t})  - h_{\star}\right\|^2\\
		&\leq& \frac{2\omega}{M^2n^2}\sum^{n-1}_{i=0} \sum^{M}_{m=1}\left\|\nabla f^{\pi^i_m}_m(x^i_{t}) - \nabla f^{\pi^i_m}_m(x_{t})\right\|^2 + \frac{2}{n}\sum^{n-1}_{i=0}\left\|\nabla f^{\pi^i}(x^i_{t})  - \nabla f^{\pi^i}(x_{t})\right\|^2\\
		&&+ \frac{2\omega}{M^2n^2}\sum^{n-1}_{i=0} \sum^{M}_{m=1}\left\|h^{\pi^i_m}_{t,m} - \nabla f^{\pi^i_m}_m(x_{t}) \right\|^2 + 2\left\|\frac{1}{n}\sum^{n-1}_{i=0}\nabla f^{\pi^i}(x_{t})  - h_{\star}\right\|^2.
	\end{eqnarray*}
	Using $L_{\max}$-smoothness and convexity of $f^i_m$ and $\widetilde{L}$-smoothness and convexity of $f^{\pi^i}$, we obtain 
	\begin{eqnarray*}
		\EE_{\cQ}\left[\|\hat{g}_t - h_{\star}\|^2\right] 
		&\leq& \frac{4\omega L_{\max}}{M^2n^2}\sum^{n-1}_{i=0} \sum^{M}_{m=1}D_{f^{\pi^i_m}_m}(x^i_{t}, x_{t}) + \frac{4\widetilde{L}}{n}\sum^{n-1}_{i=0}D_{f^{\pi^i}}(x^i_{t},x_{t})\\
		&&+ \frac{4\omega}{M^2n^2}\sum^{n-1}_{i=0} \sum^{M}_{m=1}\left\|h^{\pi^i_m}_{t,m} - \nabla f^{\pi^i_m}_m(x_{\star}) \right\|^2 + 4\widetilde{L}\left(f(x_{t})  - f(x_{\star})\right)\\
		&&+ \frac{4\omega}{M^2n^2}\sum^{n-1}_{i=0} \sum^{M}_{m=1}\left\|\nabla f^{\pi^i_m}_m(x_{t}) - \nabla f^{\pi^i_m}_m(x_{\star}) \right\|^2 \\
		&\leq& 2\widetilde{L}\left(\widetilde{L} +\frac{\omega}{Mn}L_{\max}\right)\frac{1}{n}\sum^{n-1}_{i=0} \|x^i_{t} - x_{t}\|^2 + 4\widetilde{L}\left(f(x_{t})  - f(x_{\star})\right)\\
		&&+ \frac{8\omega}{Mn}L_{\max}\frac{1}{Mn}\sum^{n-1}_{i=0} \sum^{M}_{m=1}D_{f^{\pi^i_m}_m}(x_{t},x_{\star})\\
		&&+ \frac{4\omega}{M^2n^2}\sum^{n-1}_{i=0} \sum^{M}_{m=1}\left\|h^{\pi^i_m}_{t,m} - \nabla f^{\pi^i_m}_m(x_{\star}) \right\|^2. 
	\end{eqnarray*}
\end{proof}

\begin{lemma}
	\label{lem_conv_h_diana_rr_cvx_case}
	Let $\alpha\leq\frac{1}{1+\omega}$ and  Assumptions \ref{asm:quantization_operators}, \ref{asm:sc_general_f},  \ref{asm:lip_max_f_m}, \ref{asm:lip_avr_f} hold. Then, the iterates produced by \algname{DIANA-RR} (Algorithm~\ref{alg_new_RR_DIANA}) satisfy
	\begin{eqnarray*}
		\frac{1}{Mn}\sum^{n-1}_{i=0}\sum^M_{m=1}\EE_{\cQ}\left[\|h^{\pi^i_m}_{t+1,m} - \nabla f^{\pi^i_m}_m(x_{\star})\|^2\right] &\leq& \frac{1-\alpha}{Mn}\sum^{n-1}_{i=0}\sum^M_{m=1}\|h^{\pi^i_m}_{t,m} -\nabla f^{\pi^i_m}_m(x_{\star})\|^2\\
		&&+ \frac{2\alpha\widetilde{L} L_{\max}}{n}\sum^{n-1}_{i=0} \|x^i_t-x_t\|^2 \\
		&&+ 4 \alpha L_{\max}\left(f(x_t)-f(x_{\star})\right).
	\end{eqnarray*}
\end{lemma}
\begin{proof}
	Fist of all, we introduce new notation: $$\cH_{t} \eqdef \frac{1}{Mn}\sum^{n-1}_{i=0}\sum^M_{m=1}\EE_{\cQ}\left[\|h^{\pi^i_m}_{t,m} - \nabla f^{\pi^i_m}_m(x_{\star})\|^2\right].$$ Using \eqref{eq:ndskjjdcidscuidbid} and summing it up for $i=0,\ldots, n-1$, we obtain 
	\begin{eqnarray*}
		\cH_{t+1} &\leq& \frac{1-\alpha}{Mn}\sum^{n-1}_{i=0}\sum^M_{m=1}\|h^{\pi^i_m}_{t,m}- \nabla f^{\pi^i_m}_m(x_{\star})\|^2 + \frac{\alpha}{Mn}\sum^{n-1}_{i=0}\sum^M_{m=1}\|\nabla f^{\pi^i_m}_m(x^i_{t}) - \nabla f^{\pi^i_m}_m(x_{\star})\|^2\\
		&\leq& \frac{1-\alpha}{Mn}\sum^{n-1}_{i=0}\sum^M_{m=1}\|h^{\pi^i_m}_{t,m}- \nabla f^{\pi^i_m}_m(x_{\star})\|^2 + \frac{2\alpha}{Mn}\sum^{n-1}_{i=0}\sum^M_{m=1}\|\nabla f^{\pi^i_m}_m(x^i_{t}) - \nabla f^{\pi^i_m}_m(x_{t}) \|^2\\
		&&+ \frac{2\alpha}{Mn}\sum^{n-1}_{i=0}\sum^M_{m=1}\|\nabla f^{\pi^i_m}_m(x_{t}) - \nabla f^{\pi^i_m}_m(x_{\star})\|^2.
	\end{eqnarray*}
	Next, we apply $L_{\max}$-smoothness and convexity of $f^i_m$ and $\widetilde{L}$-smoothness and convexity of $f^{\pi^i}$:
	\begin{eqnarray*}
		\cH_{t+1} 
		&\leq& \frac{1-\alpha}{Mn}\sum^{n-1}_{i=0}\sum^M_{m=1}\|h^{\pi^i_m}_{t,m}- \nabla f^{\pi^i_m}_m(x_{\star})\|^2 + \frac{4\alpha}{Mn}L_{\max}\sum^{n-1}_{i=0}\sum^M_{m=1}D_{f^{\pi^i_m}_m}(x^i_{t}, x_{t})\\
		&&+ \frac{4\alpha}{Mn}L_{\max}\sum^{n-1}_{i=0}\sum^M_{m=1} D_{f^{\pi^i_m}_m}(x_{t}, x_{\star})\\
		&\leq& \frac{1-\alpha}{Mn}\sum^{n-1}_{i=0}\sum^M_{m=1}\|h^{\pi^i_m}_{t,m}- \nabla f^{\pi^i_m}_m(x_{\star})\|^2 + \frac{2\alpha}{n}\widetilde{L}L_{\max}\sum^{n-1}_{i=0}\|x^i_{t} - x_{t}\|^2\\
		&&+ \frac{4\alpha}{Mn}L_{\max}\sum^{n-1}_{i=0}\sum^M_{m=1} D_{f^{\pi^i_m}_m}(x_{t}, x_{\star}).
	\end{eqnarray*}
\end{proof}

\begin{lemma}
	\label{lem_dinst_diana_rr_cvx_case} 
	Let Assumptions \ref{asm:quantization_operators}, \ref{asm:sc_general_f},  \ref{asm:lip_max_f_m}, \ref{asm:lip_avr_f} and $\tau \leq \frac{1}{2\sqrt{\widetilde{L}\left(\widetilde{L}+\frac{\omega}{Mn}L_{\max}\right)}}$. Then, the following inequality holds
	\begin{eqnarray*}
		\frac{1}{n}\sum^{n-1}_{i=0}\EE\left[\|x^i_t-x_t\|^2\right] &\leq& 24\tau^2\left(\widetilde{L}+\frac{\omega}{Mn}L_{\max}\right)\EE\left[f(x_t)-f(x_{\star})\right]  +8\tau^2\frac{\sigma^2_{\star,n}}{n}\\
		&&+ 8\frac{\tau^2\omega}{M^2n^2}\sum^{n-1}_{i=0}\sum^M_{m=1}\EE\left[\|h^{\pi^i_m}_{t,m}-\nabla f^{\pi^i_m}_m (x_{\star})\|^2\right],
	\end{eqnarray*}
	where $\sigma^2_{\star,n} \eqdef \frac{1}{n}\sum^n_{i=1}\|\nabla f^i(x_{\star})\|^2$.
\end{lemma}
\begin{proof}
	Since $x^i_t= x_t -\frac{\tau}{Mn}\sum \limits^M_{m=1}\sum^{i-1} \limits_{j=0}\left(h^{\pi^j_m}_{t,m} + \cQ\left(\nabla f^{\pi^j_m}_m(x^j_{t}) - h^{\pi^i_m}_{t,m}\right)\right)$, we have 
	\begin{eqnarray*}
		\EE_{\cQ}\left[\|x^i_t -x_t\|^2\right]
		&=& \tau^2\EE_{\cQ}\left[\left\|\frac{1}{Mn}\sum^M_{m=1}\sum^{i-1}_{j=0}\left(h^{\pi^j_m}_{t,m} + \cQ\left(\nabla f^{\pi^j_m}_m(x^j_{t}) - h^{\pi^j_m}_{t,m}\right)\right)\right\|^2\right]\\
		&=&\tau^2\EE_{\cQ}\left[\left\|\frac{1}{Mn}\sum^M_{m=1}\sum^{i-1}_{j=0}\left(h^{\pi^j_m}_{t,m}- \nabla f^{\pi^{j}_m}_m(x^j_t) + \cQ\left(\nabla f^{\pi^i_m}_m(x^j_{t}) - h^{\pi^j_m}_{t,m}\right)\right)\right\|^2\right]\\
		&&+ \tau^2\left\|\frac{1}{Mn}\sum^M_{m=1}\sum^{i-1}_{j=0}\nabla f^{\pi^{j}_m}_m(x^j_t)\right\|^2.
	\end{eqnarray*}
	Independence of $\cQ\left(\nabla f^{\pi^i_m}_m(x^j_{t}) - h^{\pi^j_m}_{t,m}\right)$, $m \in [M]$ and Assumption \ref{asm:quantization_operators} imply 
	\begin{eqnarray*}
		\EE_{\cQ}\left[\|x^i_t -x_t\|^2\right]
		&=&\frac{\tau^2}{M^2n^2}\sum^M_{m=1}\sum^{i-1}_{j=0}\EE_{\cQ}\left[\left\|h^{\pi^j_m}_{t,m}- \nabla f^{\pi^{j}_m}_m(x^j_t) + \cQ\left(\nabla f^{\pi^i_m}_m(x^j_{t}) - h^{\pi^j_m}_{t,m}\right)\right\|^2\right]\\
		&&+ \tau^2\left\|\frac{1}{Mn}\sum^M_{m=1}\sum^{i-1}_{j=0}\nabla f^{\pi^{j}_m}_m(x^j_t)\right\|^2\\
		&\leq& \frac{\tau^2\omega}{M^2n^2}\sum^M_{m=1}\sum^{i-1}_{j=0}\left\|\nabla f^{\pi^{j}_m}(x^j_t) - h^{\pi^j_m}_{t,m} \right\|^2 + \tau^2\left\|\frac{1}{n}\sum^{i-1}_{j=0}\nabla f^{\pi^j}(x^j_t)\right\|^2\\
		&\leq& \frac{2\tau^2\omega}{M^2n^2}\sum^M_{m=1}\sum^{n-1}_{j=0}\left\|\nabla f^{\pi^{j}_m}_m(x^j_t) - \nabla f^{\pi^{j}_m}_m(x_t)\right\|^2 + 2\tau^2\left\|\frac{1}{n}\sum^{i-1}_{j=0}\nabla f^{\pi^j}(x_t)\right\|^2\\
		&& + \frac{2\tau^2\omega}{M^2n^2}\sum^M_{m=1}\sum^{n-1}_{j=0}\left\|h^{\pi^j_m}_{t,m} - \nabla f^{\pi^{j}_m}_m(x_t) \right\|^2 + \frac{2\tau^2}{n}\sum^{n-1}_{j=0}\left\|\nabla f^{\pi^j}(x^j_t) - \nabla f^{\pi^j}(x_t)\right\|^2.
	\end{eqnarray*}
	Using $L_{\max}$-smoothness and convexity of $f^i_m$ and $\widetilde{L}$-smoothness and convexity of $f^{\pi^j}$, we obtain
	\begin{eqnarray*}
		\EE_{\cQ}\left[\|x^i_t -x_t\|^2\right]
		&\leq& \frac{4\tau^2\omega}{M^2n^2}L_{\max}\sum^M_{m=1}\sum^{n-1}_{j=0}D_{f^{\pi^{j}_m}_m}(x^j_t,x_t) + 2\tau^2\left\|\frac{1}{n}\sum^{i-1}_{j=0}\nabla f^{\pi^j}(x_t)\right\|^2\\
		&& + \frac{2\tau^2\omega}{M^2n^2}\sum^M_{m=1}\sum^{n-1}_{j=0}\left\|h^{\pi^j_m}_{t,m} - \nabla f^{\pi^{j}_m}_m(x_t) \right\|^2 + \frac{2\tau^2\widetilde{L}^2}{n}\sum^{n-1}_{j=0}\left\|x^j_t - x_t\right\|^2\\
		&\leq& 2\tau^2\widetilde{L}\left(\widetilde{L}+\frac{\omega}{Mn}L_{\max}\right)\frac{1}{n}\sum^{n-1}_{j=0}\|x^j_t - x_t\|^2 + 2\tau^2\left\|\frac{1}{n}\sum^{i-1}_{j=0}\nabla f^{\pi^j}(x_t)\right\|^2\\
		&&+ \frac{2\tau^2\omega}{M^2n^2}\sum^M_{m=1}\sum^{n-1}_{j=0}\left\|h^{\pi^j_m}_{t,m} - \nabla f^{\pi^{j}_m}_m(x_t) \right\|^2. 
	\end{eqnarray*}
	Taking the full expectation and using \eqref{eq:kdjkcnsdbciscsnd}, we derive
	\begin{eqnarray*}
		\EE\left[\|x^i_t -x_t\|^2\right]
		&\leq& 2\tau^2\widetilde{L}\left(\widetilde{L}+\frac{\omega}{Mn}L_{\max}\right)\frac{1}{n}\sum^{n-1}_{j=0}\EE\left[\|x^j_t - x_t\|^2\right] + 2\tau^2 \EE\left[\|\nabla f(x_t)\|^2\right]\\
		&&+ \frac{4\tau^2\omega}{M^2n^2}\sum^M_{m=1}\sum^{n-1}_{j=0}\EE\left[\left\|h^{\pi^j_m}_{t,m} - \nabla f^{\pi^{j}_m}_m(x_{\star}) \right\|^2\right]  + \frac{2\tau^2}{n}\EE\left[\sigma^2_t\right] \\
		&&+ \frac{8\tau^2\omega}{M^2n^2}L_{\max}\sum^M_{m=1}\sum^{n-1}_{j=0}\EE\left[D_{f^{\pi^{j}_m}_m}(x_t,x_{\star})\right].
	\end{eqnarray*}
	Using $L_{\max}$-smoothness and convexity of $f^i_m$ and $\widetilde{L}$-smoothness and convexity of $f^{\pi^j}$, we obtain
	\begin{eqnarray*}
		\EE\left[\|x^i_t -x_t\|^2\right]
		&\leq& 2\tau^2\widetilde{L}\left(\widetilde{L}+\frac{\omega}{Mn}L_{\max}\right)\frac{1}{n}\sum^{n-1}_{j=0}\EE\left[\|x^j_t - x_t\|^2\right] \\
		&&+ \frac{4\tau^2\omega}{M^2n^2}\sum^M_{m=1}\sum^{n-1}_{j=0}\EE\left[\left\|h^{\pi^j_m}_{t,m} - \nabla f^{\pi^{j}_m}_m(x_{\star}) \right\|^2\right]  + \frac{2\tau^2}{n}\EE\left[\sigma^2_t\right] \\
		&&+ 4\tau^2\left(\widetilde{L} + \frac{2\omega}{M^2n^2}L_{\max}\right)\EE\left[f(x_t)-f(x_{\star})\right].
	\end{eqnarray*}
	Now we need to estimate $\frac{2\tau^2}{n}\EE\left[\sigma^2_t\right]$. Due to $\EE\left[\sigma^2_t\right] \leq \frac{1}{n}\sum^{n}_{i=1}\EE\left[\|\nabla f^i(x_t)\|^2\right]$, we get
	\begin{eqnarray*}
		\frac{2\tau^2}{n}\EE\left[\sigma^2_t\right]
		&\leq& \frac{2\tau^2}{n^2}\sum^{n}_{j=1}\EE\left[\|\nabla f^j(x_t)\|^2\right] \\
		&\leq& \frac{4\tau^2}{n^2}\sum^{n}_{j=1}\EE\left[\|\nabla f^j(x_t) - \nabla f^j(x_{\star})\|^2\right] + \frac{4\tau^2}{n^2}\sum^{n}_{j=1}\EE\left[\|\nabla f^j(x_{\star})\|^2\right]\\
		&\leq&  \frac{8\tau^2}{n^2}\widetilde{L}\sum^{n}_{j=1}\EE\left[D_{f^j}(x_t,x_{\star})\right] + \frac{4\tau^2}{n^2}\sum^{n}_{j=1}\sigma^2_{n,\star}.
	\end{eqnarray*}
	Combining two previous inequalities, we get 
	\begin{eqnarray*}
		\EE\left[\|x^i_t -x_t\|^2\right]
		&\leq& 2\tau^2\widetilde{L}\left(\widetilde{L}+\frac{\omega}{Mn}L_{\max}\right)\frac{1}{n}\sum^{n-1}_{j=0}\EE\left[\|x^j_t - x_t\|^2\right] \\
		&&+ \frac{4\tau^2\omega}{M^2n^2}\sum^M_{m=1}\sum^{n-1}_{j=0}\EE\left[\left\|h^{\pi^j_m}_{t,m} - \nabla f^{\pi^{j}_m}_m(x_{\star}) \right\|^2\right]   \\
		&&+ 4\tau^2\left(\widetilde{L} + \frac{2\omega}{M^2n^2}L_{\max}\right)\EE\left[f(x_t)-f(x_{\star})\right]\\
		&&+ \frac{8\tau^2}{n}\widetilde{L}\EE\left[f(x_t)- f(x_{\star})\right] + \frac{4\tau^2}{n^2}\sum^{n}_{j=1}\sigma^2_{n,\star}.
	\end{eqnarray*}
	Summing from $i=0$ to $n-1$ and using $\tau \leq \frac{1}{2\sqrt{\widetilde{L}\left(\widetilde{L}+\frac{\omega}{Mn}L_{\max}\right)}}$, we obtain
	\begin{eqnarray*}
		\frac{1}{n}\sum^{n-1}_{i=0}\EE\left[\|x^i_t -x_t\|^2\right]
		&\leq& 2\left(1-2\tau^2\widetilde{L}\left(\widetilde{L}+\frac{\omega}{Mn}L_{\max}\right)\right)\frac{1}{n}\sum^{n-1}_{i=0}\EE\left[\|x^i_t - x_t\|^2\right] \\
		&\leq& \frac{8\tau^2\omega}{M^2n^2}\sum^M_{m=1}\sum^{n-1}_{j=0}\EE\left[\left\|h^{\pi^j_m}_{t,m} - \nabla f^{\pi^{j}_m}_m(x_{\star}) \right\|^2\right]  \\
		&&+ 8\tau^2\left(\widetilde{L} + \frac{2\omega}{M^2n^2}L_{\max}\right)\EE\left[f(x_t)-f(x_{\star})\right]\\
		&&+ \frac{16\tau^2}{n}\widetilde{L}\EE\left[f(x_t)- f(x_{\star})\right] + \frac{8\tau^2}{n^2}\sum^{n}_{j=1}\sigma^2_{n,\star}.
	\end{eqnarray*}
\end{proof}

We consider the following Lyapunov function:
\begin{equation}
	\label{lyapunov_func_diana_rr_cvx_case}
	\Psi_{t+1} \eqdef \|x_{t+1}-x_{\star}\|^2 +\frac{c\tau^2}{Mn}\sum^M_{m=1}\sum^{n-1}_{j=0}\left\|h^{\pi^i_m}_{t+1,m} - \nabla f^{\pi^i_m}_m(x_{\star})\right\|^2.
\end{equation}

\begin{theorem}
	\label{thm:advanced_conv_for_DIANA_RR}
	 Let Assumptions \ref{asm:quantization_operators}, \ref{asm:sc_general_f},  \ref{asm:lip_max_f_m}, \ref{asm:lip_avr_f} hold  and
	\begin{equation*}
		\gamma \leq \min\left\{\frac{\alpha}{n\mu}, \frac{1}{12n\left(\widetilde{L}+\frac{11\omega}{Mn}L_{\max}\right)}\right\},\quad \alpha \leq \frac{1}{1+\omega},\quad c = \frac{10\omega}{\alpha Mn}.
	\end{equation*}
	Then, for all $T \geq 0$ the iterates produced by \algname{DIANA-RR} (Algorithm~\ref{alg_new_RR_DIANA}) satisfy
	\begin{equation*}
		\EE\left[\Psi_T\right] \leq \left(1-\frac{n\gamma\mu}{2}\right)^{T}\Psi_0 +20\frac{\gamma^2n\widetilde{L}}{\mu}\sigma^2_{\star,n}.
	\end{equation*}
\end{theorem}
\begin{proof}
	Taking expectation w.r.t.\ $\cQ$ and using Lemma \ref{lem_inner_product_diana_rr_cvx_case}, we get 
	\begin{eqnarray*}
		\EE_{\cQ}\left[\|x_{t+1} - x_{\star}\|^2\right]
		&=& \|x_t -\tau\hat{g}_t - x_{\star}+\tau h^{\star}\|^2\\
		&=& \|x_t - x_{\star}\|^2 -2\tau\EE_{\cQ}\left[\la \hat{g}_t - h^{\star}, x_t- x_{\star}\ra\right] + \tau^2\EE_{\cQ}\left[\|\hat{g}_t - h^{\star}\|^2\right]\\
		&\leq& \|x_t - x_{\star}\|^2 -\frac{\tau\mu}{2}\|x_t-x_{\star}\|^2  + \tau^2\EE_{\cQ}\left[\|\hat{g}_t - h^{\star}\|^2\right]\\
		&& - \tau\left(f(x_t)-f(x_{\star})\right) + \tau\widetilde{L}\frac{1}{n}\sum^{n-1}_{i=1}\|x_t-x^i_t\|^2.
	\end{eqnarray*}
	Next, due to Lemma \ref{lem_norm_of_grad_diana_rr_cvx_case} we have 
	\begin{eqnarray*}
		\EE_{\cQ}\left[\|x_{t+1} - x_{\star}\|^2\right]
		&\leq& \left(1 -\frac{\tau\mu}{2}\right)\|x_t-x_{\star}\|^2 - \tau\left(f(x_t)-f(x_{\star})\right) + \tau\widetilde{L}\frac{1}{n}\sum^{n-1}_{i=1}\|x_t-x^i_t\|^2\\
		&&+ 2\tau^2\widetilde{L}\left(\widetilde{L} + \frac{\omega}{Mn}L_{\max}\right)\frac{1}{n}\sum^{n-1}_{i=0}\|x^i_t-x_t\|^2 \\
		&&+ 8\tau^2\left(\widetilde{L} + \frac{\omega}{Mn}L_{\max}\right)\left(f(x_t)-f(x_{\star})\right)\\
		&&+ \frac{4\omega\tau^2}{M^2n^2}\sum^{n-1}_{i=0}\sum^M_{m=1}\|h^{\pi^i_m}_{t,m}-\nabla f^{\pi^i_m}_{m}(x_{\star})\|^2\\
		&\leq& \left(1 -\frac{\tau\mu}{2}\right)\|x_t-x_{\star}\|^2  + \frac{4\omega\tau^2}{M^2n^2}\sum^{n-1}_{i=0}\sum^M_{m=1}\|h^{\pi^i_m}_{t,m}-\nabla f^{\pi^i_m}_{m}(x_{\star})\|^2\\
		&&- \tau\left(1-8\tau\left(\widetilde{L} + \frac{\omega}{Mn}L_{\max}\right)\right)\left(f(x_t)-f(x_{\star})\right)\\
		&&+ \tau\widetilde{L}\left(1+2\tau\left(\widetilde{L} + \frac{\omega}{Mn}L_{\max}\right)\right)\frac{1}{n}\sum^{n-1}_{i=0}\|x^i_t-x_t\|^2.
	\end{eqnarray*}
	Using \eqref{lyapunov_func_diana_rr_cvx_case}, we obtain
	\begin{eqnarray*}
		\EE_{\cQ}\left[\Psi_{t+1}\right]
		&\leq& \left(1 -\frac{\tau\mu}{2}\right)\|x_t-x_{\star}\|^2  + \frac{4\omega\tau^2}{M^2n^2}\sum^{n-1}_{i=0}\sum^M_{m=1}\|h^{\pi^i_m}_{t,m}-\nabla f^{\pi^i_m}_{m}(x_{\star})\|^2\\
		&&- \tau\left(1-8\tau\left(\widetilde{L} + \frac{\omega}{Mn}L_{\max}\right)\right)\left(f(x_t)-f(x_{\star})\right)\\
		&&+ \tau\widetilde{L}\left(1+2\tau\left(\widetilde{L} + \frac{\omega}{Mn}L_{\max}\right)\right)\frac{1}{n}\sum^{n-1}_{i=0}\|x^i_t-x_t\|^2 \\
		&&+ \frac{c\tau^2}{Mn}\sum^M_{m=1}\sum^{n-1}_{j=0}\EE\left[\left\|h^{\pi^i_m}_{t+1,m} - \nabla f^{\pi^i_m}_m(x_{\star})\right\|^2\right].
	\end{eqnarray*}
	To estimate the last term in the above inequality, we apply Lemma \ref{lem_conv_h_diana_rr_cvx_case}:
	\begin{eqnarray*}
		\EE_{\cQ}\left[\Psi_{t+1}\right]
		&\leq& \left(1 -\frac{\tau\mu}{2}\right)\|x_t-x_{\star}\|^2  + \frac{4\omega\tau^2}{M^2n^2}\sum^{n-1}_{i=0}\sum^M_{m=1}\|h^{\pi^i_m}_{t,m}-\nabla f^{\pi^i_m}_{m}(x_{\star})\|^2\\
		&&- \tau\left(1-8\tau\left(\widetilde{L} + \frac{\omega}{Mn}L_{\max}\right)\right)\left(f(x_t)-f(x_{\star})\right)\\
		&&+ \tau\widetilde{L}\left(1+2\tau\left(\widetilde{L} + \frac{\omega}{Mn}L_{\max}\right)\right)\frac{1}{n}\sum^{n-1}_{i=0}\|x^i_t-x_t\|^2 \\
		&&+ c\tau^2\frac{1-\alpha}{Mn}\sum^{n-1}_{i=0}\sum^M_{m=1}\|h^{\pi^i_m}_{t,m} -\nabla f^{\pi^i_m}_m(x_{\star})\|^2\\
		&&+ c\tau^2\frac{2\alpha\widetilde{L} L_{\max}}{n}\sum^{n-1}_{i=0} \|x^i_t-x_t\|^2 + 4 c\tau^2\alpha L_{\max}\left(f(x_t)-f(x_{\star})\right)\\
		&\leq& \left(1 -\frac{\tau\mu}{2}\right)\|x_t-x_{\star}\|^2  + \left(1-\alpha+ \frac{4\omega}{cMn}\right)\frac{c\tau^2}{Mn}\sum^{n-1}_{i=0}\sum^M_{m=1}\|h^{\pi^i_m}_{t,m}-\nabla f^{\pi^i_m}_{m}(x_{\star})\|^2\\
		&&- \tau\left(1-4c\tau\alpha L_{\max}-8\tau\left(\widetilde{L} + \frac{\omega}{Mn}L_{\max}\right)\right)\left(f(x_t)-f(x_{\star})\right)\\
		&& + \tau\widetilde{L}\left(1+ 2c\tau\alpha L_{\max}+2\tau\left(\widetilde{L} + \frac{\omega}{Mn}L_{\max}\right)\right)\frac{1}{n}\sum^{n-1}_{i=0}\|x^i_t-x_t\|^2.
	\end{eqnarray*}	
	Let $\cH_t \eqdef \frac{c\tau^2}{Mn}\sum^{n-1}_{i=0}\sum^M_{m=1}\EE\left[\|h^{\pi^i_m}_{t,m}-\nabla f^{\pi^i_m}_{m}(x_{\star})\|^2\right]$. Taking the full expectation and using Lemma \ref{lem_dinst_diana_rr_cvx_case}, we get 
	\begin{eqnarray*}
		\EE\left[\Psi_{t+1}\right]
		&\leq& \left(1 -\frac{\tau\mu}{2}\right)\EE\left[\|x_t-x_{\star}\|^2\right] + \left(1-\alpha+ \frac{4\omega}{cMn}\right)\cH_t\\
		&&- \tau\left(1-4c\tau\alpha L_{\max}-8\tau\left(\widetilde{L} + \frac{\omega}{Mn}L_{\max}\right)\right)\EE\left[f(x_t)-f(x_{\star})\right]\\
		&& + \tau\widetilde{L}\left(1+ 2c\tau\alpha L_{\max}+2\tau\left(\widetilde{L} + \frac{\omega}{Mn}L_{\max}\right)\right)\frac{1}{n}\sum^{n-1}_{i=0}\EE\left[\|x^i_t-x_t\|^2\right] \\
		&\leq& \left(1 -\frac{\tau\mu}{2}\right)\EE\left[\|x_t-x_{\star}\|^2\right] + \left(1-\alpha+ \frac{4\omega}{cMn}\right)\cH_t\\
		&&- \tau\left(1-4c\tau\alpha L_{\max}-8\tau\left(\widetilde{L} + \frac{\omega}{Mn}L_{\max}\right)\right)\EE\left[f(x_t)-f(x_{\star})\right]\\
		&&+24\tau^3\widetilde{L}\left(1+ 2c\tau\alpha L_{\max}+2\tau\left(\widetilde{L} + \frac{\omega}{Mn}L_{\max}\right)\right)\left(\widetilde{L}+\frac{\omega}{Mn}L_{\max}\right)\EE\left[f(x_t)-f(x_{\star})\right]\\  &&+8\tau^3\widetilde{L}\left(1+ 2c\tau\alpha L_{\max}+2\tau\left(\widetilde{L} + \frac{\omega}{Mn}L_{\max}\right)\right)\frac{\sigma^2_{\star,n}}{n}\\
		&&+ \frac{8\tau\widetilde{L}\omega}{cMn}\left(1+ 2c\tau\alpha L_{\max}+2\tau\left(\widetilde{L} + \frac{\omega}{Mn}L_{\max}\right)\right)\cH_t.
	\end{eqnarray*}
	Selecting $c = \frac{A\omega}{\alpha Mn}$, where $A$ is a positive number to be specified later, we have 
	\begin{equation*}
		1+ 2c\tau\alpha L_{\max}+2\tau\left(\widetilde{L} + \frac{\omega}{Mn}L_{\max}\right) = 1 + 2\tau\left(\widetilde{L} + \frac{(A+1)\omega}{Mn}L_{\max}\right),
	\end{equation*}
	\begin{equation*}
		1-4c\tau\alpha L_{\max}-8\tau\left(\widetilde{L} + \frac{\omega}{Mn}L_{\max}\right) \geq 1 - 8\tau\left(\widetilde{L} + \frac{(A+1)\omega}{Mn}L_{\max}\right).
	\end{equation*}
	Then, we have 
	\begin{eqnarray*}
		\EE\left[\Psi_{t+1}\right]
		&\leq& \left(1 -\frac{\tau\mu}{2}\right)\EE\left[\|x_t-x_{\star}\|^2\right] + \left(1-\alpha+ \frac{4\alpha}{A} \right)\cH_t\\
		&&- \tau\left(1 - 8\tau\left(\widetilde{L} + \frac{(A+1)\omega}{Mn}L_{\max}\right)\right)\EE\left[f(x_t)-f(x_{\star})\right]\\
		&&+24\tau^3\widetilde{L}\left(\widetilde{L}+\frac{\omega}{Mn}L_{\max}\right)\left(1 + 2\tau\left(\widetilde{L} + \frac{(A+1)\omega}{Mn}L_{\max}\right)\right)\EE\left[f(x_t)-f(x_{\star})\right]\\  &&+8\tau^3\widetilde{L}\left(1 + 2\tau\left(\widetilde{L} + \frac{(A+1)\omega}{Mn}L_{\max}\right)\right)\frac{\sigma^2_{\star,n}}{n}\\
		&&+ \frac{8\alpha}{A}\tau\widetilde{L}\left(1 + 2\tau\left(\widetilde{L} + \frac{(A+1)\omega}{Mn}L_{\max}\right)\right)\cH_t.
	\end{eqnarray*}
	Taking $\tau = \frac{1}{B\left(\widetilde{L} + \frac{(A+1)\omega}{Mn}L_{\max}\right)}$, where $B$ is some positive constant, we obtain
	\begin{eqnarray*}
		\EE\left[\Psi_{t+1}\right]
		&\leq& \left(1 -\frac{\tau\mu}{2}\right)\EE\left[\|x_t-x_{\star}\|^2\right] + \left(1-\alpha+ \frac{4\alpha}{A} + \frac{8\alpha}{A}\tau\widetilde{L}\left(1 + \frac{2}{B}\right) \right)\cH_t\\
		&&- \tau\left(1 - \frac{8}{B}- \frac{24}{B^2}\left(1 + \frac{2}{B}\right)\right)\EE\left[f(x_t)-f(x_{\star})\right]\\
		&&+8\tau^3\widetilde{L}\left(1 + \frac{2}{B}\right)\frac{\sigma^2_{\star,n}}{n}.
	\end{eqnarray*}
	Choosing $A=10$, $B = 12$, $\tau \leq \frac{\alpha}{\mu}$, we have 
	\begin{eqnarray*}
		\EE\left[\Psi_{t+1}\right]
		&\leq& \left(1 -\min\left\{\frac{\tau\mu}{2}, \frac{\alpha}{2}\right\}\right)\EE\left[\Psi_{t}\right] +10\tau^3\widetilde{L}\frac{\sigma^2_{\star,n}}{n}\\
		&\leq& \left(1 -\frac{\tau\mu}{2}\right)\EE\left[\Psi_{t}\right] +10\tau^3\widetilde{L}\frac{\sigma^2_{\star,n}}{n}\\
	\end{eqnarray*}
	Recursively unrolling the inequality, substituting $\tau=n\gamma$ and using $\sum\limits^{+\infty}_{t = 0}\left(1-\frac{\tau\mu}{2}\right)^t \leq \frac{2}{\mu\tau}$, we finish proof.
\end{proof}

\begin{corollary}
	\label{cor_convergence_DIANA_RR_cvx_case}
	Let the assumptions of Theorem~\ref{thm:advanced_conv_for_DIANA_RR} hold, $\alpha = \frac{1}{1+\omega}$, and
	\begin{equation}
		\label{new_gamma_DIANA_RR_cvx_case}
		\gamma = \min\left\{\frac{\alpha}{2n\mu},\frac{1}{12n\left(\widetilde{L}+\frac{11\omega}{Mn}L_{\max}\right)}, \sqrt{\frac{\varepsilon\mu}{40n\widetilde{L}\sigma^2_{\star,n}}}\right\}.
	\end{equation} 
	Then, \algname{DIANA-RR} (Algorithm~\ref{alg_new_RR_DIANA}) finds a solution with accuracy $\varepsilon>0$ after the following number of communication rounds: 
	\begin{equation}
		\widetilde{\cO}\left(n(1+\omega)+\frac{n\widetilde{L}}{\mu}+\frac{\omega}{M}\frac{L_{\max}}{\mu} + \sqrt{\frac{n\widetilde{L}}{\varepsilon\mu^3}}\sigma_{\star,n}\right).\notag
	\end{equation}
\end{corollary}
\begin{proof}
	Theorem~\ref{thm:advanced_conv_for_DIANA_RR} implies
	\begin{equation*}
		\EE\left[\Psi_T\right] \leq \left(1-\gamma\mu\right)^{nT}\Psi_0 +20\frac{\gamma^2n\widetilde{L}}{\mu}\sigma^2_{\star,n}.
	\end{equation*}
	To estimate the number of communication rounds required to find a solution with accuracy $\varepsilon >0$, we need to upper bound each term from the right-hand side by $\frac{\varepsilon}{2}$. Thus, we get an additional condition on $\gamma$:
	\begin{equation*}
		20\frac{\gamma^2n\widetilde{L}}{\mu}\sigma^2_{\star,n} <\frac{\varepsilon}{2},
	\end{equation*}
	and also the upper bound on the number of communication rounds $nT$
	\begin{equation*}
		nT = \widetilde{\cO}\left(\frac{1}{\gamma\mu}\right).
	\end{equation*}
Substituting \eqref{new_gamma_DIANA_RR_cvx_case} in the previous equation, we obtain the result.
\end{proof}

\newpage

\section{Missing Proofs For Q-NASTYA}

We start with deriving a technical lemma along with stating several useful results from \citep{malinovsky22_server_side_steps_sampl_without}. For convenience, we also introduce the following notation:
\begin{equation*}
	g_{t,m} \eqdef \frac{1}{n}\sum^{n-1}_{i = 0}\nabla f^{\pi^i_m}_m(x^i_{t,m}).
\end{equation*}

\begin{lemma}
	\label{lem_norm_g_t}
	Let Assumptions \ref{asm:quantization_operators}, \ref{asm:sc_general_f}, \ref{asm:lip_max_f_m} hold. Then, for all $t \geq 0$ the iterates produced by \algname{Q-NASTYA} (Algorithm~\ref{alg:Q_NASTYA}) satisfy
	\begin{equation*}
		\EE_{\cQ}\left[\|g_t\|^2\right] \leq \frac{2L_{\max}^2\left(1+\frac{\omega}{M}\right)}{Mn}\sum^{M}_{m = 1}\sum^{n-1}_{i = 0}\left\|x^i_{t,m} - x_t\right\|^2 + 8L_{\max}\left(1+\frac{\omega}{M}\right)\left(f(x_t) -  f(x_{\star})\right) + \frac{4\omega}{M}\zeta^2_{\star},
	\end{equation*}
	where $\EE_{\cQ}$ is expectation w.r.t.\ $\cQ,$ and $\zeta^2_{\star} \eqdef \frac{1}{M}\sum^M_{m = 1}\left\|\nabla f_m(x_{\star})\right\|^2$.
	
\end{lemma}

\begin{proof} Using the variance decomposition $\EE\left[\|\xi\|^2\right] = \EE\left[\|\xi-\EE\left[\xi\right]\|^2\right] +\|\EE\xi\|^2$, we obtain 
	\begin{eqnarray*}
		\EE_{\cQ}\left[\|g_t\|^2\right]
		&=& \frac{1}{M^2}\sum^M_{m=1}\EE_{\cQ}\left[\left\|\cQ\left(\frac{1}{n}\sum^{n-1}_{i = 0}\nabla f^{\pi^i_m}_m(x^i_{t,m})\right) - \frac{1}{n}\sum^{n-1}_{i = 0}\nabla f^{\pi^i_m}_m(x^i_{t,m})\right\|^2\right]\\
		&&+ \left\| \frac{1}{Mn}\sum^M_{m=1}\sum^{n-1}_{i = 0}\nabla f^{\pi^i_m}_m(x^i_{t,m})\right\|^2\\
		&\overset{\text{Asm.} \ref{asm:quantization_operators}}{\leq}&\frac{\omega}{M^2}\sum^M_{m=1}\left\|\frac{1}{n}\sum^{n-1}_{i = 0}\nabla f^{\pi^i_m}_m(x^i_{t,m})\right\|^2   + \left\| \frac{1}{Mn}\sum^M_{m =1}\sum^{n-1}_{i = 0}\nabla f^{\pi^i_m}_m(x^i_{t,m})\right\|^2.
	\end{eqnarray*}
	Next, we use $\nabla f_m(x_t) = \frac{1}{n}\sum^{n-1}_{i=0}\nabla f^{\pi^i_m}_m(x_t)$ and $\|a+b\|^2 \leq 2\|a\|^2 + 2\|b\|^2$:  
	\begin{eqnarray*}
		\EE_{\cQ}\left[\|g_t\|^2\right]
		&\leq& \frac{2\omega}{M^2}\sum^M_{m=1}\left\|\frac{1}{n}\sum^{n-1}_{i = 0}\left(\nabla f^{\pi^i_m}_m(x^i_{t,m}) - \nabla f^{\pi^i_m}_m(x_t)\right)\right\|^2 + \frac{2\omega}{M^2}\sum^M_{m=1}\left\|\nabla f_m(x_t)\right\|^2 \\
		&& + 2\left\| \frac{1}{Mn}\sum^M_{m=1}\sum^{n-1}_{i = 0}\left(\nabla f^{\pi^i_m}_m(x^i_{t,m}) - \nabla f^{\pi^i_m}_m(x_{t}) \right)\right\|^2 + 2\left\| \frac{1}{M}\sum^M_{m=1}\nabla f_m(x_{t})\right\|^2\\
		&\leq& \frac{2\left(1+\frac{\omega}{M}\right)}{M}\sum_{m = 1}^M\left\|\frac{1}{n}\sum^{n-1}_{i = 0}\left(\nabla f^{\pi^i_m}_m(x^i_{t,m}) - \nabla f^{\pi^i_m}_m(x_t)\right)\right\|^2\\
		&&+ \frac{2\omega}{M^2}\sum_{m = 1}^M\left\|\nabla f_m(x_t)\right\|^2 + 2\left\| \nabla f(x_{t})\right\|^2.
	\end{eqnarray*}
	Using $L_{i,m}$-smoothness of $f^i_m$ and $f$ and also convexity of $f_m$, we obtain
	\begin{eqnarray*}
		\EE_{\cQ}\left[\|g_t\|^2\right]
		&\leq& \frac{2\left(1+\frac{\omega}{M}\right)}{Mn}\sum_{m = 1}^M\sum^{n-1}_{i = 0}\left\|\nabla f^{\pi^i_m}_m(x^i_{t,m}) - \nabla f^{\pi^i_m}_m(x_t)\right\|^2 + \frac{4\omega}{M^2}\sum_{m = 1}^M\left\|\nabla f_m(x_t) -\nabla f_m(x_{\star})\right\|^2 \\
		&&+ \frac{4\omega}{M^2}\sum_{m = 1}^M\left\|\nabla f_m(x_{\star})\right\|^2 + 2\left\| \nabla f(x_{t}) - \nabla f(x_{\star})\right\|^2\\
		&\leq& \frac{2L_{\max}^2\left(1+\frac{\omega}{M}\right)}{Mn}\sum_{m = 1}^M\sum^{n-1}_{i = 0}\left\|x^i_{t,m} - x_t\right\|^2 + \frac{8L_{\max}\left(1+\frac{\omega}{M}\right)}{M}\sum_{m = 1}^M D_{f_m}(x_t, x_{\star}) + \frac{4\omega}{M}\zeta_{\star}^2.
	\end{eqnarray*}
\end{proof}

\begin{lemma}[See \citep{malinovsky22_server_side_steps_sampl_without}]
	\label{lem_inner_product}
	Under Assumptions \ref{asm:quantization_operators}, \ref{asm:sc_general_f}, \ref{asm:lip_max_f_m}, we have
	\begin{equation*}
		-\frac{1}{Mn}\sum^{M}_{m = 1}\sum^{n-1}_{i = 0}\left\la f^{\pi^i_m}_m(x^i_{t,m}),x_t-x_{\star}\right\ra \leq -\frac{\mu}{4}\|x_t - x_{\star}\|^2 - \frac{1}{2}\left(f(x_t)-f(x_{\star})\right) + \frac{L_{\max}}{2Mn}\sum^{M}_{m = 1}\sum^{n-1}_{i = 0}\left\|x^i_{t,m} - x_t\right\|^2.
	\end{equation*}
\end{lemma}

\begin{lemma}[See \citep{malinovsky22_server_side_steps_sampl_without}]
	\label{lem_dist}
	Under Assumptions \ref{asm:quantization_operators}, \ref{asm:sc_general_f}, \ref{asm:lip_max_f_m} and $\gamma\leq \frac{1}{2L_{\max}n}$, we have
	\begin{equation}
		\frac{1}{Mn}\sum^{M}_{m = 1}\sum^{n-1}_{i = 0}\left\|x^i_{t,m} - x_t\right\|^2 \leq 8\gamma^2n^2L_{\max}\left(f(x_t)-f(x_{\star})\right) + 2\gamma^2n\left(\sigma^2_{\star} + (n+1)\zeta^2_{\star}\right).\notag
	\end{equation}
\end{lemma}

\begin{theorem}[Theorem~\ref{thm:convergence_Q_NASTYA}]
	\label{th_convergence_Q_RR}
	Let Assumptions \ref{asm:quantization_operators}, \ref{asm:sc_general_f}, \ref{asm:lip_max_f_m} hold and stepsizes $\gamma$, $\eta$ satisfy
	\begin{equation}
		\label{step_sizes}
		0 < \eta \leq \frac{1}{16L_{\max}\left(1+\frac{\omega}{M}\right)}, \quad 0 < \gamma \leq \frac{1}{5nL_{\max}}.
	\end{equation}
	Then, for all $T \geq 0$ the iterates produced by \algname{Q-NASTYA} (Algorithm~\ref{alg:Q_NASTYA}) satisfy
	\begin{eqnarray*}
		\EE\left[\|x_T-x_{\star}\|^2\right]
		&\leq& \left(1-\frac{\eta\mu}{2}\right)^T\|x_0 - x_{\star}\|^2 +\frac{9}{2}\frac{\gamma^2nL_{\max}}{\mu}\left(\sigma^2_{\star} +(n+1)\zeta^2_{\star}\right) +  8\frac{\eta\omega}{\mu M}\zeta^2_{\star}.
	\end{eqnarray*}
\end{theorem}
\begin{proof} Taking expectation w.r.t.\ $\cQ$ and using Lemma \ref{lem_norm_g_t}, we get 
\begin{eqnarray*}
	\EE_{\cQ}\left[\|x_{t+1} - x_{\star}\|^2\right] 
	&=& \|x_t - x_{\star}\|^2 -2\eta\EE_{\cQ}\left[\la g_t, x_t- x_{\star}\ra\right] + \eta^2\EE_{\cQ}\left[\|g^t\|^2\right]\\
	&\leq& \|x_t - x_{\star}\|^2 -2\eta\EE_{\cQ}\left[\left\la \frac{1}{M}\sum_{m=1}^M\cQ\left(\frac{1}{n}\sum^{n-1}_{i = 0}\nabla f^{\pi^i_m}_m(x^i_{t,m})\right), x_t- x_{\star}\right\ra\right]\\
	&& + \frac{2\eta^2L^2_{\max}\left(1+\frac{\omega}{M}\right)}{Mn}\sum^{M}_{m = 1}\sum^{n-1}_{i = 0}\left\|x^i_{t,m} - x_t\right\|^2\\
	&& + 8\eta^2L_{\max}\left(1+\frac{\omega}{M}\right)\left(f(x_t) -  f(x_{\star})\right) + 4\eta^2\frac{\omega}{M}\zeta^2_{\star}\\
	&\leq& \|x_t - x_{\star}\|^2 -2\eta\frac{1}{Mn}\sum^M_{m = 1}\sum^{n-1}_{i = 0}\left\la \nabla f^{\pi^i_m}_m(x^i_{t,m}), x_t- x_{\star}\right\ra\\
	&& + \frac{2\eta^2L^2_{\max}\left(1+\frac{\omega}{M}\right)}{Mn}\sum^{M}_{m = 1}\sum^{n-1}_{i = 0}\left\|x^i_{t,m} - x_t\right\|^2\\
	&& + 8\eta^2L_{\max}\left(1+\frac{\omega}{M}\right)\left(f(x_t) -  f(x_{\star})\right) + 4\eta^2\frac{\omega}{M}\zeta^2_{\star}.
\end{eqnarray*}
Next, Lemma \ref{lem_inner_product} implies 
\begin{eqnarray*}
	\EE_{\cQ}\left[\|x_{t+1} - x_{\star}\|^2\right] 
	&\leq&  \|x_t - x_{\star}\|^2  - \frac{\eta\mu}{2}\|x_t - x_{\star}\|^2 - \eta\left(f(x_t)-f(x_{\star})\right)\\
	&&+ 8\eta^2L_{\max}\left(1+\frac{\omega}{M}\right)\left(f(x_t) -  f(x_{\star})\right) + \frac{\eta L_{\max}}{Mn}\sum^{M}_{m = 1}\sum^{n-1}_{i = 0}\left\|x^i_{t,m} - x_t\right\|^2\\
	&& + \frac{2\eta^2L^2_{\max}\left(1+\frac{\omega}{M}\right)}{Mn}\sum^{M}_{m = 1}\sum^{n-1}_{i = 0}\left\|x^i_{t,m} - x_t\right\|^2 + 4\eta^2\frac{\omega}{M}\zeta^2_{\star}\\
	&\leq& \left(1 -\frac{\eta\mu}{2}\right)\|x_t - x_{\star}\|^2 - \eta\left(1-8\eta L_{\max}\left(1+\frac{\omega}{M}\right)\right)\left(f(x_t) -  f(x_{\star})\right) \\
	&& + \frac{\eta L_{\max}\left(1+ 2\eta L_{\max}\left(1+\frac{\omega}{M}\right)\right)}{Mn}\sum^{M}_{m = 1}\sum^{n-1}_{i = 0}\left\|x^i_{t,m} - x_t\right\|^2 + 4\eta^2\frac{\omega}{M}\zeta^2_{\star}.
\end{eqnarray*}
Using Lemma \ref{lem_dist}, we get 
\begin{eqnarray*}
	\EE_{\cQ}\left[\|x_{t+1} - x_{\star}\|^2\right] 
	&\leq& \left(1 -\frac{\eta\mu}{2}\right)\|x_t - x_{\star}\|^2 - \eta\left(1-8\eta L\left(1+\frac{\omega}{M}\right)\right)\left(f(x_t) -  f(x_{\star})\right)   \\
	&& + \eta L_{\max}\left(1+ 2\eta L_{\max}\left(1+\frac{\omega}{M}\right)\right) \cdot 8\gamma^2n^2L_{\max}\left(f(x_t)-f(x_{\star})\right)\\
	&& + \eta L_{\max}\left(1+ 2\eta L_{\max}\left(1+\frac{\omega}{M}\right)\right) \cdot 2\gamma^2n\left(\sigma^2_{\star} + (n+1)\zeta^2_{\star}\right) \\
	&&  + 4\eta^2\frac{\omega}{M}\zeta^2_{\star}.
\end{eqnarray*}
In view of \eqref{step_sizes}, we have
\begin{eqnarray*}
	\EE_{\cQ}\left[\|x_{t+1} - x_{\star}\|^2\right] 
	&\leq& \left(1 -\frac{\eta\mu}{2}\right)\|x_t - x_{\star}\|^2 + 4\eta^2\frac{\omega}{M}\zeta^2_{\star}  \\
	&& - \eta\left(1-8\eta L_{\max}\left(1+\frac{\omega}{M}\right) - 8\gamma^2n^2L^2_{\max}\left(1+2L_{\max}\eta \left(1 +\frac{\omega}{M}\right)\right)\right)\left(f(x_t) -  f(x_{\star})\right)  \\
	&& + 2\gamma^2n\eta L_{\max}\left(1+ 2\eta L\left(1+\frac{\omega}{M}\right)\right) \left(\sigma^2_{\star} + n\zeta^2_{\star}\right)\\
	&\leq& \left(1 -\frac{\eta\mu}{2}\right)\|x_t - x_{\star}\|^2 + 4\eta^2\frac{\omega}{M}\zeta^2_{\star} + \frac{9}{4}\eta L_{\max}\gamma^2n \left(\sigma^2_{\star} + (n+1)\sigma^2_{\star}\right).
\end{eqnarray*}
Recursively unrolling the inequality and using $\sum\limits^{+\infty}_{t = 0}\left(1-\frac{\eta\mu}{2}\right)^t \leq \frac{2}{\mu\eta}$, we get the result.
\end{proof}

\begin{corollary}
	\label{cor_convergence_Q_NASTYA}
	Let the assumptions of Theorem~\ref{thm:convergence_Q_NASTYA} hold, \(\gamma = \frac{\eta}{n}\), and 
	\begin{equation}
		\label{new_gamma_Q_NASTAY}
		\eta = \min\left\{ \frac{1}{16L_{\max}\left(1+\frac{\omega}{M}\right)}, \sqrt{\frac{\varepsilon\mu n}{9L_{\max}}}\left( (n+1)\zeta^2_{\star} + \sigma^2_{\star}\right)^{-\nicefrac{1}{2}}, \frac{\varepsilon\mu M}{24\omega\zeta^2_{\star}}\right\}.
	\end{equation}
	Then, \algname{Q-NASTYA} (Algorithm~\ref{alg:Q_NASTYA}) finds a solution with accuracy $\varepsilon > 0$ after the following number of communication rounds:
	\begin{equation*}
		\widetilde{\cO}\left(\frac{L_{\max}}{\mu}\left(1+\frac{\omega}{M}\right)+ \frac{\omega}{M}\frac{\zeta^2_{\star}}{\varepsilon\mu^3}+\sqrt{\frac{ L_{\max} }{\varepsilon\mu^3}} \sqrt{\zeta^2_{\star}+\frac{\sigma_{\star}^2}{n}}\right).
	\end{equation*}
	If $\gamma \rightarrow 0$, one can choose $\eta = \min\left\{\frac{1}{16L_{\max}\left(1+\frac{\omega}{M}\right)},\frac{\varepsilon\mu M}{24\omega\zeta^2_{\star}}\right\}$ such that the above complexity bound improves to
	\begin{equation*}
		\widetilde{\cO}\left( \frac{L_{\max}}{\mu}\left(1+\frac{\omega}{M}\right)+\frac{\omega}{M}\frac{\zeta^2_{\star}}{\varepsilon\mu^3}\right).
	\end{equation*}
\end{corollary}
\begin{proof}
	Theorem~\ref{thm:convergence_Q_NASTYA} implies
	\begin{equation}
		\EE\left[\|x_T-x_{\star}\|^2\right] \leq \left(1-\frac{\eta\mu}{2}\right)^T\|x_0 - x_{\star}\|^2 +\frac{9}{2}\frac{\gamma^2nL_{\max}}{\mu}\left((n+1)\zeta^2_{\star} + \sigma_{\star}^2\right) +  8\frac{\eta\omega}{\mu M}\zeta^2_{\star}.\notag
	\end{equation}
	To estimate the number of communication rounds required to find a solution with accuracy $\varepsilon >0$, we need to upper bound each term from the right-hand side by $\nicefrac{\varepsilon}{3}$. Thus, we get additional conditions on $\eta$:
	\begin{equation*}
		\frac{9}{2}\frac{\eta^2 L_{\max}}{n\mu}\left( (n+1)\zeta^2_{\star} + \sigma^2_{\star}\right) <\frac{\varepsilon}{3},\qquad 8\frac{\eta\omega}{\mu M}\zeta^2_{\star} < \frac{\varepsilon}{3}
	\end{equation*}
	and also the upper bound on the number of communication rounds $T$
	\begin{equation*}
		T = \widetilde{\cO}\left(\frac{1}{\eta\mu}\right).
	\end{equation*}
	Substituting \eqref{new_gamma_DIANA_NASTAY} in the previous equation, we get the first part of the result. 
	When $\gamma \rightarrow 0$, the proof follows similar steps. 
\end{proof}

\newpage

\section{Missing Proofs For DIANA-NASTYA}

\begin{lemma}
	\label{rr_diana_inner_product}
	Under Assumptions \ref{asm:quantization_operators}, \ref{asm:sc_general_f}, \ref{asm:lip_max_f_m}, the iterates produced by \algname{DIANA-NASTYA} (Algorithm~\ref{alg:diana-nastya}) satisfy
	\begin{eqnarray*}
		-\EE_{\cQ}\left[\frac{1}{M}\sum^M_{m=1}\left\la \hat{g}_{t,m} - h^{\star}, x_t - x_{\star}\right\ra\right] 
		&\leq& -\frac{\mu}{4}\|x_t - x_{\star}\|^2 -\frac{1}{2}\left(f(x_t) - f(x_{\star})\right)\\
		&& - \frac{1}{Mn}\sum^M_{m=1}\sum^{n-1}_{i=0} D_{f^{\pi^i_m}_m}(x_{\star}, x^i_{t,m})\\
		&& + \frac{L_{\max}}{2Mn}\sum^M_{m=1}\sum^{n-1}_{i=0}\|x_t - x^i_{t,m}\|^2,
	\end{eqnarray*}
	where $h^{\star} = \nabla f(x_{\star})$.
\end{lemma}
\begin{proof} Using that $\EE_{\cQ}\left[\hat{g}_{t,m}\right] = g_{t,m}$ and definition of $h^{\star}$, we get
	\begin{eqnarray*}
		-\EE_{\cQ}\left[\frac{1}{M}\sum^M_{m=1}\left\la \hat{g}_{t,m} - h^{\star}, x_t - x_{\star}\right\ra\right] 
		&=& -\frac{1}{M}\sum^M_{m=1}\left\la g_{t,m} - h^{\star}, x_t - x_{\star}\right\ra\\
		&=& -\frac{1}{Mn}\sum^M_{m=1}\sum^{n-1}_{i=0}\left\la \nabla f^{\pi^i_m}_m(x^i_{t,m}) - \nabla f^{\pi^i_m}_m(x_{\star}), x_t - x_{\star}\right\ra.
	\end{eqnarray*}
	Next, three-point identity and $L_{\max}$-smoothness of each function $f^{i}_m$ imply 
	\begin{eqnarray*}
		-\EE_{\cQ}\left[\frac{1}{M}\sum^M_{m=1}\left\la \hat{g}_{t,m} - h^{\star}, x_t - x_{\star}\right\ra\right] 
		&=& -\frac{1}{Mn}\sum^M_{m=1}\sum^{n-1}_{i=0}\left( D_{f^{\pi^i_m}_m}(x_t, x_{\star})+ D_{f^{\pi^i_m}_m}( x_{\star}, x^{i}_{t,m}) - D_{f^{\pi^i_m}_m}(x_t, x^{i}_{t,m})\right)\\
		&\leq& - D_{f}(x_t, x_{\star}) - \frac{1}{Mn}\sum^M_{m=1}\sum^{n-1}_{i=0} D_{f^{\pi^i_m}_m}(x_{\star}, x^i_{t,m})\\
		&& + \frac{L_{\max}}{2Mn}\sum^M_{m=1}\sum^{n-1}_{i=0} \|x_t - x^{i}_{t,m}\|^2 .
	\end{eqnarray*}
	Finally, using $\mu$-strong convexity of $f$, we finish the proof of lemma.
\end{proof}

\begin{lemma}
	\label{rr_diana_dinst}
	Under Assumptions \ref{asm:quantization_operators}, \ref{asm:sc_general_f}, \ref{asm:lip_max_f_m}, the iterates produced by \algname{DIANA-NASTYA} (Algorithm~\ref{alg:diana-nastya}) satisfy
	\begin{eqnarray*}
		\EE_{\cQ}\left[\|\hat{g}_t - h^{\star}\|^2\right] 
		&\leq& \frac{2L^2_{\max}\left(1 + \frac{\omega}{M}\right)}{Mn}\sum^M_{m=1}\sum^{n-1}_{i=0}\|x^{i}_{t,m} - x_t\|^2 + 8L_{\max}\left(1+\frac{\omega}{M}\right)\left(f(x_t) - f(x_{\star})\right)\\ 
		&&+ \frac{4\omega}{M^2}\sum^M_{m=1}\|h_{t,m}-h^{\star}_m\|^2.
	\end{eqnarray*}
\end{lemma}
\begin{proof}
	Since $g_t = \frac{1}{M}\sum^M_{m=1}g_{t,m}$ and $\EE\|\xi - c\|^2 = \EE\|\xi - \EE\xi\|^2 + \EE\|\EE\xi - c\|^2$, we have
	\begin{eqnarray*}
		\EE_{\cQ}\left[\|\hat{g}_t - h^{\star}\|^2\right] 
		&=& \EE_{\cQ}\left[\left\|\frac{1}{M}\sum^M_{m=1}\left( h_{t,m} + \cQ\left(g_{t,m} - h_{t,m}\right) - h^{\star}_m\right)\right\|^2\right] \\
		&=& \EE_{\cQ}\left[\left\|\frac{1}{M}\sum^M_{m=1}\left( h_{t,m} + \cQ\left(g_{t,m} - h_{t,m}\right)\right) - g_t\right\|^2\right] + \left\|g_t-h^{\star}\right\|^2.
	\end{eqnarray*}
	Next, independence of $\cQ\left(g_{t,m} - h_{t,m}\right)$, $m \in M$, Assumption \ref{asm:quantization_operators}, and $L_{\max}$-smoothness  and convexity of each function $f^{i}_{m}$ imply
	\begin{eqnarray*}
		\EE_{\cQ}\left[\|\hat{g}_t - h^{\star}\|^2\right] 
		&\leq& \frac{\omega}{M^2}\sum^M_{m=1}\left\| g_{t,m} - h_{t,m}\right\|^2 + \left\|g_t-h^{\star}\right\|^2\\
		&\leq&\frac{2\omega}{M^2}\sum^M_{m=1}\left\| \frac{1}{n}\sum^{n-1}_{i=0}\nabla f^{\pi^i_m}_m(x^i_{t,m}) - \nabla f_m(x_t)\right\|^2 + \frac{2\omega}{M^2}\sum^M_{m=1}\left\|\nabla f_m(x_t) - h_{t,m}\right\|^2\\
		&&+ 2\left\|g_t-\nabla f(x_t)\right\|^2 + 2\left\|\nabla f(x_t)-h^{\star}\right\|^2\\
		&\leq& \frac{2\omega}{M^2}\sum^M_{m=1}\left\| \frac{1}{n}\sum^{n-1}_{i=0}\nabla f^{\pi^i_m}_m(x^i_{t,m}) - \nabla f_m(x_t)\right\|^2 + \frac{2\omega}{M^2}\sum^M_{m=1}\left\|\nabla f_m(x_t) - h_{t,m}\right\|^2\\
		&&+ 2\left\|\frac{1}{M}\sum^M_{m=1}\left(\frac{1}{n}\sum^{n-1}_{i=0}\nabla f^{\pi^i_m}_m(x^i_{t,m}) - \nabla f_m(x_t)\right)\right\|^2 + 2\left\|\nabla f(x_t)-h^{\star}\right\|^2\\
		&\leq& \frac{2L^2_{\max}\left(1+\frac{\omega}{M}\right)}{Mn}\sum^M_{m=1}\sum^{n-1}_{i=0}\|x^i_{t,m} - x_t\|^2 + \frac{2\omega}{M^2}\sum^M_{m=1}\|\nabla f_m(x_t) - h_{t,m}\|^2 \\
		&& + 2\left\|\nabla f(x_t)-h^{\star}\right\|^2.
	\end{eqnarray*}
	Using $L_{\max}$-smoothness and convexity of $f_m$, we get
	\begin{eqnarray*}
		\EE_{\cQ}\left[\|\hat{g}_t - h^{\star}\|^2\right] 
		&\leq& \frac{2L^2_{\max}\left(1+\frac{\omega}{M}\right)}{Mn}\sum^M_{m=1}\sum^{n-1}_{i=0}\|x^i_{t,m} - x_t\|^2 + \frac{2\omega}{M^2}\sum^M_{m=1}\|\nabla f_m(x_t) - h_{t,m}\|^2\\
		&&+ 4L_{\max}\left(f(x_t)-f(x_{\star})\right)\\
		&\leq& \frac{2L^2_{\max}\left(1+\frac{\omega}{M}\right)}{Mn}\sum^M_{m=1}\sum^{n-1}_{i=0}\|x^i_{t,m} - x_t\|^2 + \frac{4\omega}{M^2}\sum^M_{m=1}\|\nabla f_m(x_t) - h^{\star}_{m}\|^2\\
		&&+\frac{4\omega}{M^2}\sum^M_{m=1}\|h_{t,m}- h^{\star}_{m}\|^2+ 4L_{\max}\left(f(x_t)-f(x_{\star})\right)\\
		&\leq& \frac{2L^2_{\max}\left(1+\frac{\omega}{M}\right)}{Mn}\sum^M_{m=1}\sum^{n-1}_{i=0}\|x^i_{t,m} - x_t\|^2 + \frac{8L_{\max}\omega}{M^2}\sum^M_{m=1}D_{f_m}(x_t, x_{\star})\\
		&&+\frac{4\omega}{M^2}\sum^M_{m=1}\|h_{t,m}- h^{\star}_{m}\|^2+ 4L_{\max}\left(f(x_t)-f(x_{\star})\right).
	\end{eqnarray*}
\end{proof}

\begin{lemma}
	\label{rr_diana_descent_lemma_for_h}
	Under Assumptions \ref{asm:quantization_operators}, \ref{asm:sc_general_f}, \ref{asm:lip_max_f_m}, and $\alpha\leq\frac{1}{1+\omega}$, the iterates produced by \algname{DIANA-NASTYA} (Algorithm~\ref{alg:diana-nastya}) satisfy 
	\begin{eqnarray*}
		\frac{1}{M}\sum^M_{m=1}\EE_{\cQ}\left[\|h_{t+1,m} - h^{\star}_m\|^2\right] &\leq& \frac{1-\alpha}{M}\sum^M_{m=1}\|h_{t,m} - h^{\star}_m\|^2 + \frac{2\alpha L^2_{\max}}{Mn}\sum^M_{m=1}\sum^{n-1}_{i=0}\|x^i_{t,m} - x_t\|^2 \\
		&&+ 4\alpha L_{\max}\left(f(x_t) - f(x_{\star})\right).
	\end{eqnarray*}
\end{lemma}
\begin{proof}
	Taking expectation w.r.t.\ $\cQ$ and using Assumption \ref{asm:quantization_operators}, we obtain
	\begin{eqnarray*}
		\frac{1}{M}\sum^M_{m=1}\EE_{\cQ}\left[\|h_{t+1,m} - h^{\star}_m\|^2\right] &=& \frac{1}{M}\sum^M_{m=1}\EE_{\cQ}\left[\|h_{t,m} +\alpha\cQ(g_{t,m}-h_{t,m})- h^{\star}_m\|^2\right]\\
		&\leq& \frac{1}{M}\sum^M_{m=1}\left(\|h_{t,m}- h^{\star}_m\|^2 +2\alpha\EE_{\cQ}\left[\left\la\cQ(g_{t,m}-h_{t,m}),h_{t,m}- h^{\star}_m\right\ra\right]\right) \\
		&&+ \alpha^2 \frac{1}{M}\sum^M_{m=1}\EE_{\cQ}\left[\| \cQ(g_{t,m}-h_{t,m})\|^2\right]\\
		&\leq& \frac{1}{M}\sum^M_{m=1}\left(\|h_{t,m}- h^{\star}_m\|^2 +2\alpha\left\la g_{t,m}-h_{t,m},h_{t,m}- h^{\star}_m\right\ra\right) \\
		&&+ \alpha^2(1+\omega) \frac{1}{M}\sum^M_{m=1}\| g_{t,m}-h_{t,m}\|^2 .
	\end{eqnarray*}
	Using $\alpha \leq \frac{1}{1+\omega}$, we get
	\begin{eqnarray*}
		\frac{1}{M}\sum^M_{m=1}\EE_{\cQ}\left[\|h_{t+1,m} - h^{\star}_m\|^2\right] 
		&\leq& \frac{1}{M}\sum^M_{m=1}\left(\|h_{t,m}- h^{\star}_m\|^2 + \alpha\left\la g_{t,m}-h_{t,m},h_{t,m} +g_{t,m} - 2h^{\star}_m\right\ra\right) \\
		&\leq& \frac{1}{M}\sum^M_{m=1}\left(\|h_{t,m}- h^{\star}_m\|^2 + \alpha\|g_{t,m} - h^{\star}_m\|^2 - \alpha\|h_{t,m}-h^{\star}_m\|^2\right)\\
		&\leq& \frac{1-\alpha}{M}\sum^M_{m=1}\|h_{t,m}- h^{\star}_m\|^2 + \frac{\alpha}{M}\sum^M_{m=1}\|g_{t,m} - h^{\star}_m\|^2.
	\end{eqnarray*}
	Finally, $L_{\max}$-smoothness and convexity of $f_m$ imply
	\begin{eqnarray*}
		\frac{1}{M}\sum^M_{m=1}\EE_{\cQ}\left[\|h_{t+1,m} - h^{\star}_m\|^2\right]
		&\leq& \frac{1-\alpha}{M}\sum^M_{m=1}\|h_{t,m}- h^{\star}_m\|^2\\
		&&+\frac{2\alpha}{M}\sum^M_{m=1}\left(\|g_{t,m} - \nabla f_{m}(x_t)\|^2 + \|\nabla f_{m}(x_t) - h^{\star}_m\|^2\right)\\
		&\leq& \frac{1-\alpha}{M}\sum^M_{m=1}\|h_{t,m}- h^{\star}_m\|^2 + \frac{4L_{\max}\alpha}{M}\sum^M_{m=1}D_{f_{m}}(x_t, x_{\star})\\
		&& + \frac{2\alpha}{M}\sum^M_{m=1}\left\|\frac{1}{n}\sum^{n-1}_{i=0}(\nabla f^{\pi^i_m}_m(x^{i}_{t,m}) - \nabla f^i_{m}(x_t))\right\|^2\\
		&\leq& \frac{1-\alpha}{M}\sum^M_{m=1}\|h_{t,m}- h^{\star}_m\|^2 + 4L_{\max}\alpha\left(f(x_t)-f(x_{\star})\right) \\
		&&+ \frac{2L^2_{\max}\alpha}{Mn}\sum^M_{m=1}\sum^{n-1}_{i=0}\left\|x^{i}_{t,m}-x_t\right\|^2.
	\end{eqnarray*}
\end{proof}

\begin{theorem}[Theorem~\ref{thm:diana-nastya-conv}]
	\label{rr_diana_conv_th}
	Let Assumptions \ref{asm:quantization_operators}, \ref{asm:sc_general_f}, \ref{asm:lip_max_f_m} hold and stepsizes $\gamma$, $\eta, \alpha$ satisfy
	\begin{equation}
		\label{step_sizes_diana_nastya}
		0 < \gamma \leq  \frac{1}{16L_{\max}n},\quad 0 < \eta \leq \min\left\{\frac{\alpha}{2\mu}, \frac{1}{16L_{\max}\left(1+\frac{9\omega}{M}\right)}\right\},\quad \alpha \leq \frac{1}{1+\omega}.
	\end{equation}
	Then, for all $T \geq 0$ the iterates produced by \algname{DIANA-NASTYA} (Algorithm~\ref{alg:diana-nastya}) satisfy
	\begin{equation}
		\EE\left[\Psi_T\right] \leq \left(1-\frac{\eta\mu}{2}\right)^T\Psi_0 + \frac{9}{2}\frac{\gamma^2 nL}{\mu}\left( \sigma^2_{\star} + (n+1)\zeta^2_{\star}\right).
	\end{equation}
\end{theorem}
\begin{proof}
We have
\begin{eqnarray*}
	\|x_{t+1}-x_{\star}\|^2 
	&=& \|x_t -\eta \hat{g}_t - x_{\star} + \eta h^{\star}\|^2\\
	&=& \|x_t - x_{\star}\|^2 - 2\eta\la\hat{g}_t - h^{\star},x_t-x_{\star}\ra + \eta^2 \|\hat{g}_t - h^{\star}\|^2.
\end{eqnarray*}
Taking expectation w.r.t.\ $\cQ$ and using Lemma \ref{rr_diana_inner_product}, we obtain
\begin{eqnarray*}
	\EE_{\cQ}\left[\|x_{t+1}-x_{\star}\|^2\right]
	&=& \|x_t - x_{\star}\|^2 - 2\eta\EE_{\cQ}\left[\la\hat{g}_t - h^{\star},x_t-x_{\star}\ra\right] + \eta^2 \EE_{\cQ}\left[\|\hat{g}_t - h^{\star}\|^2\right]\\
	&\leq& \left(1-\frac{\eta\mu}{2}\right)\|x_t - x_{\star}\|^2 - \eta(f(x_t)-f(x_{\star})) - \frac{2\eta}{Mn}\sum^M_{m=1}\sum^{n-1}_{i=0}D_{f^{\pi^i_m}_m}(x_{\star}, x^i_{t,m}) \\
	&&+\frac{L_{\max}\eta}{Mn}\sum^M_{m=1}\sum^{n-1}_{i=0}\|x^i_{t,m} - x_t\|^2 + \eta^2 \EE_{\cQ}\left[\|\hat{g}_t - h^{\star}\|^2\right].
\end{eqnarray*}
Next, Lemma \ref{rr_diana_dinst} implies
\begin{eqnarray*}
	\EE_{\cQ}\left[\|x_{t+1}-x_{\star}\|^2\right]
	&\leq& \left(1-\frac{\eta\mu}{2}\right)\|x_t - x_{\star}\|^2 - \eta(f(x_t)-f(x_{\star})) - \frac{2\eta}{Mn}\sum^M_{m=1}\sum^{n-1}_{i=0}D_{f^{\pi^i_m}_m}(x_{\star}, x^i_{t,m}) \\
	&&+\frac{L_{\max}\eta}{Mn}\sum^M_{m=1}\sum^{n-1}_{i=0}\|x^i_{t,m} - x_t\|^2 + \frac{2\eta^2L^2_{\max}\left(1 + \frac{\omega}{M}\right)}{Mn}\sum^M_{m=1}\sum^{n-1}_{i=0}\|x^{i}_{t,m} - x_t\|^2\\
	&& + \eta^2\left( 8L_{\max}\left(1+\frac{\omega}{M}\right)\left(f(x_t) - f(x_{\star})\right)+ \frac{4\omega}{M^2}\sum^M_{m=1}\|h_{t,m}-h^{\star}_m\|^2\right)\\
	&\leq& \left(1-\frac{\eta\mu}{2}\right)\|x_t - x_{\star}\|^2 - \eta\left(1- 8\eta L_{\max}\left(1+\frac{\omega}{M}\right)\right)\left(f(x_t) - f(x_{\star})\right)\\
	&&+ L_{\max}\eta\left(1 +2\eta L_{\max}\left(1+\frac{\omega}{M}\right)\right) \frac{1}{Mn}\sum^M_{m=1}\sum^{n-1}_{i=0}\|x^{i}_{t,m} - x_t\|^2\\
	&&- \frac{2\eta}{Mn}\sum^M_{m=1}\sum^{n-1}_{i=0}D_{f^{\pi^i_m}_m}(x_{\star}, x^i_{t,m}) + \frac{4\eta^2\omega}{M^2}\sum^M_{m=1}\|h_{t,m}-h^{\star}_m\|^2.
\end{eqnarray*}
Using \eqref{lyapunov_func_rr_diana} and Lemma \ref{rr_diana_descent_lemma_for_h}, we get
\begin{eqnarray*}
	\EE_{\cQ}\left[\Psi_{t+1}\right]
	&\leq& \left(1-\frac{\eta\mu}{2}\right)\|x_t - x_{\star}\|^2 - \eta\left(1- 8\eta L_{\max}\left(1+\frac{\omega}{M}\right)\right)\left(f(x_t) - f(x_{\star})\right)\\
	&&+ L_{\max}\eta\left(1 +2\eta L_{\max}\left(1+\frac{\omega}{M}\right)\right) \frac{1}{Mn}\sum^M_{m=1}\sum^{n-1}_{i=0}\|x^{i}_{t,m} - x_t\|^2\\
	&&- \frac{2\eta}{Mn}\sum^M_{m=1}\sum^{n-1}_{i=0}D_{f^{\pi^i_m}_m}(x_{\star}, x^i_{t,m}) + \frac{4\eta^2\omega}{M^2}\sum^M_{m=1}\|h_{t,m}-h^{\star}_m\|^2\\
	&& + c\eta^2\left(\frac{1-\alpha}{M}\sum^M_{m=1}\|h_{t,m} - h^{\star}_m\|^2 + \frac{2\alpha L^2_{\max}}{Mn}\sum^M_{m=1}\sum^{n-1}_{i=0}\|x^i_{t,m} - x_t\|^2 + 4\alpha L_{\max}\left(f(x_t) - f(x_{\star})\right)\right)\\
	&\leq& \left(1-\frac{\eta\mu}{2}\right)\|x_t - x_{\star}\|^2 + \eta^2\left(c(1-\alpha)+\frac{4\omega}{M}\right)\frac{1}{M}\sum^M_{m=1}\|h_{t,m}-h^{\star}_m\|^2 \\
	&&- \eta\left(1- 8\eta L_{\max}\left(1+\frac{\omega}{M}\right) - 4\alpha\eta cL_{\max}\right)\left(f(x_t) - f(x_{\star})\right)\\
	&&+ L\eta\left(1 +2\eta L_{\max}\left(1+\frac{\omega}{M}\right)+ 2\alpha \eta cL_{\max}\right) \frac{1}{Mn}\sum^M_{m=1}\sum^{n-1}_{i=0}\|x^{i}_{t,m} - x_t\|^2.
\end{eqnarray*}
Taking the full expectation, we derive 
\begin{eqnarray*}
	\EE\left[\Psi_{t+1}\right]
	&\leq& \left(1-\frac{\eta\mu}{2}\right)\EE\left[\|x_t - x_{\star}\|^2\right] + \eta^2\left(c(1-\alpha)+\frac{4\omega}{M}\right)\frac{1}{M}\sum^M_{m=1}\EE\left[\|h_{t,m}-h^{\star}_m\|^2\right] \\
	&&- \eta\left(1- 8\eta L_{\max}\left(1+\frac{\omega}{M}\right) - 4\alpha\eta cL\right)\EE\left[f(x_t) - f(x_{\star})\right]\\
	&&+ L_{\max}\eta\left(1 +2\eta L_{\max}\left(1+\frac{\omega}{M}\right)+ 2\alpha \eta cL_{\max}\right) \frac{1}{Mn}\sum^M_{m=1}\sum^{n-1}_{i=0}\EE\left[\|x^{i}_{t,m} - x_t\|^2\right].
\end{eqnarray*}
Using Lemma \ref{lem_dist}, we get
\begin{eqnarray*}
	\EE\left[\Psi_{t+1}\right]
	&\leq& \left(1-\frac{\eta\mu}{2}\right)\EE\left[\|x_t - x_{\star}\|^2\right] + \eta^2\left(c(1-\alpha)+\frac{4\omega}{M}\right)\frac{1}{M}\sum^M_{m=1}\EE\left[\|h_{t,m}-h^{\star}_m\|^2\right] \\
	&&- \eta\left(1- 8\eta L_{\max}\left(1+\frac{\omega}{M}\right) - 4\alpha\eta cL_{\max}\right)\EE\left[f(x_t) - f(x_{\star})\right]\\
	&&+ 8\gamma^2n^2L^2_{\max}\eta\left(1 +2\eta L_{\max}\left(1+\frac{\omega}{M}\right)+ 2\alpha \eta cL_{\max}\right)\EE\left[f(x_t) - f(x_{\star})\right]\\
	&&+ 2\gamma^2nL_{\max}\eta\left(1 +2\eta L_{\max}\left(1+\frac{\omega}{M}\right)+ 2\alpha \eta cL_{\max}\right) \left(\sigma^2_{\star} + (n+1)\zeta^2_{\star}\right).
\end{eqnarray*}
In view of \eqref{step_sizes_diana_nastya}, we have
\begin{eqnarray*}
	\EE\left[\Psi_{t+1}\right]
	&\leq& \left(1-\frac{\eta\mu}{2}\right)\EE\left[\|x_t - x_{\star}\|^2\right] + \left(1-\frac{\alpha}{2}\right)\frac{c\eta^2}{M}\sum^M_{m=1}\EE\left[\|h_{t,m}-h^{\star}_m\|^2\right] \\
	&&+ \frac{9}{4}\gamma^2nL_{\max}\eta \left(\sigma^2_{\star} + (n+1)\zeta^2_{\star}\right)
\end{eqnarray*}
Using the definition of the Lyapunov function and using $\sum\limits^{+\infty}_{t = 0}\left(1-\frac{\eta\mu}{2}\right)^t \leq \frac{2}{\mu\eta}$, we get the result.
\end{proof}

\begin{corollary}
	\label{cor_convergence_DIANA_NASTYA}
	Let the assumptions of Theorem~\ref{thm:diana-nastya-conv} hold, $\gamma = \frac{\eta}{n}$, $\alpha = \frac{1}{1+\omega}$, and
	\begin{equation}
		\label{new_gamma_DIANA_NASTAY}
		\eta = \min\left\{\frac{\alpha}{2\mu}, \frac{1}{16L_{\max}\left(1+\frac{9\omega}{M}\right)}, \sqrt{\frac{\varepsilon\mu n}{9L_{\max}}}\left( (n+1)\zeta^2_{\star} + \sigma^2_{\star}\right)^{-\nicefrac{1}{2}}\right\}.
	\end{equation}
	Then, \algname{DIANA-NASTYA} (Algorithm~\ref{alg:diana-nastya}) finds a solution with accuracy $\varepsilon>0$ after the following number of communication rounds: 
	\begin{equation}
		\widetilde{\cO}\left(\omega+\frac{L_{\max}}{\mu}\left(1+\frac{\omega}{M}\right) + \sqrt{\frac{L_{\max}}{\varepsilon\mu^3}}\sqrt{ \zeta^2_{\star} + \nicefrac{\sigma^2_{\star}}{n}}\right).\notag
	\end{equation}
	If $\gamma \rightarrow 0$, one can choose $\eta = \min\left\{\frac{\alpha}{2\mu}, \frac{1}{16L_{\max}\left(1+\frac{9\omega}{M}\right)}\right\}$
	such that the number of communication rounds $T$ to find solution with accuracy $\varepsilon>0$ is 
	\begin{equation}
		\widetilde{\cO}\left(\omega +\frac{L_{\max}}{\mu}\left(1+\frac{\omega}{M}\right)\right).\notag
	\end{equation}
\end{corollary}
\begin{proof}
	Theorem~\ref{thm:diana-nastya-conv} implies
	\begin{equation}
		\EE\left[\Psi_T\right] \leq \left(1-\frac{\eta\mu}{2}\right)^T\Psi_0 + \frac{9}{2}\frac{\gamma^2 nL_{\max}}{\mu}\left( (n+1)\zeta^2_{\star} + \sigma^2_{\star}\right).\notag
	\end{equation}
	To estimate the number of communication rounds required to find a solution with accuracy $\varepsilon >0$, we need to upper bound each term from the right-hand side by $\frac{\varepsilon}{2}$. Thus, we get an additional restriction on $\eta$:
	\begin{equation*}
		\frac{9}{2}\frac{\eta^2 L_{\max}}{n\mu}\left( (n+1)\zeta^2_{\star} + \sigma^2_{\star}\right) <\frac{\varepsilon}{2},
	\end{equation*}
	and also the upper bound on the number of communication rounds $T$
	\begin{equation*}
		T = \widetilde{\cO}\left(\frac{1}{\eta\mu}\right).
	\end{equation*}
	Substituting \eqref{new_gamma_DIANA_NASTAY} in the previous equation, we get the first part of the result. 
	When $\gamma \rightarrow 0$, the proof follows similar steps. 
\end{proof}

\newpage

\section{Alternative Analysis of Q-NASTYA}
In this analysis, we will use additional sequence:
\begin{equation}
	x^i_{\star,m} = x_\star - \gamma\sum_{j=0}^{i-1}\nabla f_m(x_\star).
\end{equation}
\begin{theorem}\label{thm:Q_NASTYA_alternative_proof}
	Let Assumptions \ref{asm:quantization_operators}, \ref{asm:lip_max_f_m}, \ref{asm:sc_each_f_m} hold. Moreover, we assume that $(1-\gamma\mu)^n\leq \frac{\nicefrac{9}{10} - \nicefrac{1}{C}}{1+\nicefrac{1}{C}} = \widehat{C}<1$ for some numerical constant $C > 1$. Also let $\beta = \frac{\eta}{\gamma n} \leq \frac{1}{3C\frac{\omega}{M}+1}$ and $\gamma \leq \frac{1}{L_{\max}}$. Then, for all $T \geq 0$ the iterates produced by \algname{Q-NASTYA} (Algorithm~\ref{alg:Q_NASTYA}) satisfy
\begin{align*}
	\mathbb{E}\left[\|x_T - x_\star\|^2\right]  &\leq	\left(1-\frac{\beta}{10}\right)\|x_t - x_*\|^2+\frac{4}{\mu}\beta\gamma^2 \hat\sigma_{\mathrm{rad}}^{2} + 3\beta^2 \frac{\omega}{M} \frac{1}{M}\hat{\Delta}_\star,
\end{align*}
where $\hat{\Delta}_\star \eqdef \frac{1}{M}\sum_{m=1}^{M}\|x^n_{\star,m} - x_\star\|^2$ and $\hat{\sigma}^2_{\text{rad}} \leq L_{\max}\left( \zeta_{\star}^2 + \nicefrac{n\sigma_{\star}^2}{4} \right)$.
\end{theorem}

\begin{proof}
	The update rule for one epoch can be rewritten as
	\begin{align*}
		x_{t+1} = x_t - \eta \frac{1}{M}\sum_{m=1}^{M}Q\left( \frac{x_t - x^n_{t,m}}{\gamma n} \right).
	\end{align*}
	Using this, we derive
	\begin{align*}
		\| x_{t+1} - x_* \|^2 &= \left\| x_t - \eta \frac{1}{M} \sum_{m=1}^{M}Q\left( \frac{x_t - x^n_{t,m}}{\gamma n}\right) - x_* \right\|^2\\
		&=\|x_t - x_*\|^2 - 2\eta\left\langle x_t - x_*,\frac{1}{M} \sum_{m=1}^{M}Q\left( \frac{x_t - x^n_{t,m}}{\gamma n} \right) \right\rangle\\
		&\quad + \eta^2\left\| \frac{1}{M} \sum_{m=1}^{M}Q\left( \frac{x_t - x^n_{t,m}}{\gamma n} \right)\right\|^2.
	\end{align*}
	Taking conditional expectation w.r.t.\ the randomness comming from compression, we get 
	\begin{align*}
		\mathbb{E}_Q\|x_{t+1} - x_*\|^2 &= \|x_t - x_*\|^2  - 2\eta\left\langle x_t - x_*,\frac{1}{M} \sum_{m=1}^{M}\left( \frac{x_t - x^n_{t,m}}{\gamma n} \right) \right\rangle\\
		&\quad + \eta^2	\mathbb{E}_Q\left\| \frac{1}{M} \sum_{m=1}^{M}Q\left( \frac{x_t - x^n_{t,m}}{\gamma n} \right)\right\|^2.
	\end{align*}
	Next, we use the definition of quantization operator and independence of $Q\left( \frac{x_t - x^n_{t,m}}{\gamma n} \right)$, $m \in [M]$: 
	\begin{align*}
		\mathbb{E}_Q \|x_{t+1} - x_*\|^2 & \leq \|x_t - x_*\|^2  - 2\eta\left\langle x_t - x_*,\frac{1}{M} \sum_{m=1}^{M}\left( \frac{x_t - x^n_{t,m}}{\gamma n} \right) \right\rangle\\
		& \quad + \eta^2\left( \frac{\omega}{M} \frac{1}{M}\sum_{m=1}^M\left\|\frac{x_t - x^n_{t,m}}{\gamma n}\right\|^2 + \left\| \frac{1}{M}\sum^M_{m=1}\frac{x_t - x^n_{t,m}}{\gamma n}  \right\|^2\right).
	\end{align*}
	Since $\beta = \frac{\eta}{\gamma n}$, we obtain
	\begin{align*}
		\mathbb{E}_Q \|x_{t+1} - x_*\|^2 & \leq \|x_t - x_*\|^2  - 2\beta\left\langle x_t - x_*,\frac{1}{M} \sum_{m=1}^{M}\left( x_t - x^n_{t,m} \right) \right\rangle\\
		&\quad  + \beta^2 \frac{\omega}{M} \frac{1}{M}\sum_{m=1}^M\left\|x_t - x^n_{t,m}\right\|^2 +\beta^2 \left\| \frac{1}{M}\sum^M_{m=1}\left(x_t - x^n_{t,m}\right)  \right\|^2\\
		&=\|x_t - x_*\|^2  + 2\beta\left\langle x_t - x_*,\frac{1}{M} \sum_{m=1}^{M}\left( x^n_{t,m} - x_t   \right) \right\rangle\\
		&\quad  + \beta^2 \frac{\omega}{M} \frac{1}{M}\sum_{m=1}^M\left\|x_t - x^n_{t,m}\right\|^2 +\beta^2 \left\| \frac{1}{M}\sum^M_{m=1}\left(x^n_{t,m} - x_t\right)  \right\|^2 \\
		&=\left\| x_t - x_* + \beta\left( \frac{1}{M}\sum_{m=1}^M\left( x^n_{t,m} - x_t  \right) \right)\right\|^2+\beta^2 \frac{\omega}{M} \frac{1}{M}\sum_{m=1}^M\left\|x_t - x^n_{t,m}\right\|^2  \\
		&=\left\| (1-\beta)(x_t - x_*) + \beta\left( \frac{1}{M}\sum_{m=1}^M\left( x^n_{t,m}   \right) - x_* \right)\right\|^2\\
		&\quad +\beta^2 \frac{\omega}{M} \frac{1}{M}\sum_{m=1}^M\left\|x_t - x^n_{t,m}\right\|^2 .
	\end{align*}
	Using the condition that $x_* = \frac{1}{M}\sum_{m=1}^M x^n_{*,m}$ we have:
	\begin{align*}
		\mathbb{E}_Q \|x_{t+1} - x_*\|^2 & \leq\left\| (1-\beta)(x_t - x_*) + \beta\left( \frac{1}{M}\sum_{m=1}^M\left( x^n_{t,m} - x^n_{*,m}  \right)  \right)\right\|^2+\beta^2 \frac{\omega}{M} \frac{1}{M}\sum_{m=1}^M\left\|x_t - x^n_{t,m}\right\|^2 .
	\end{align*}
	Convexity of squared norm and Jensen's inequality imply
	\begin{align*}
		\mathbb{E}_Q \|x_{t+1} - x_*\|^2 \leq (1-\beta)\|x_t - x_*\|^2 + \beta \left\|  \frac{1}{M}\sum_{m=1}^M\left( x^n_{t,m} - x^n_{*,m}  \right) \right\|^2 +\beta^2 \frac{\omega}{M} \frac{1}{M}\sum_{m=1}^M\left\|x_t - x^n_{t,m}\right\|^2 .
	\end{align*}
	Next, from Young's inequality we get
	\begin{align*}
		\mathbb{E}_Q \|x_{t+1} - x_*\|^2 &\leq (1-\beta)\|x_t - x_*\|^2 + \beta \left\|  \frac{1}{M}\sum_{m=1}^M\left( x^n_{t,m} - x^n_{*,m}  \right) \right\|^2 + 3\beta^2 \frac{\omega}{M}\|x_t - x_*\|^2\\
		& + 3\beta^2 \frac{\omega}{M} \frac{1}{M}\sum_{m=1}^{M}\|x^n_{t,m} - x^n_{*,m}\|^2+ 3\beta^2 \frac{\omega}{M} \frac{1}{M}\sum_{m=1}^{M}\| x^n_{*,m} - x_* \|^2.
	\end{align*}
	Theorem 4 from \citep{mishchenko21_proxim_feder_random_reshuf} gives 
	\begin{align*}
	\mathbb{E}\left[	\frac{1}{M}\sum_{m=1}^{M}\|x^n_{t,m} - x^n_{*,m}\|^2\right]& \leq(1-\gamma \mu)^{n} \left[\left\|x_{t}-x_{*}\right\|^{2}\right]+2 \gamma^{3} \hat\sigma_{\mathrm{rad}}^{2}\left(\sum_{j=0}^{n-1}(1-\gamma \mu)^{j}\right)\\
		& = (1-\gamma \mu)^{n} \left[\left\|x_{t}-x_{*}\right\|^{2}\right]+2 \gamma^{2} \hat\sigma_{\mathrm{rad}}^{2}\frac{1}{\gamma\mu}.
	\end{align*}
	It leads to
	\begin{align*}
		\mathbb{E} \|x_{t+1} - x_*\|^2 &\leq (1-\beta)\|x_t - x_*\|^2 + \beta \left((1-\gamma \mu)^{n} \left[\left\|x_{t}-x_{*}\right\|^{2}\right]+2 \gamma^{3} \hat\sigma_{\mathrm{rad}}^{2}\frac{1}{\gamma\mu}\right)\\
		&\quad + 3\beta^2 \frac{\omega}{M}\|x_t - x_*\|^2 + 3\beta^2 \frac{\omega}{M} \left( (1-\gamma \mu)^{n} \left[\left\|x_{t}-x_{*}\right\|^{2}\right]+2 \gamma^{3} \hat\sigma_{\mathrm{rad}}^{2}\frac{1}{\gamma\mu}\right)\\
		&\quad + 3\beta^2 \frac{\omega}{M} \frac{1}{M}\sum_{m=1}^{M}\| x^n_{*,m} - x_* \|^2\\
		&\leq \left(1-\beta+\beta(1-\gamma\mu)^n+3\beta^2\frac{\omega}{M}+3\beta^2\frac{\omega}{M}(1-\gamma\mu)^n\right)\|x_t - x_*\|^2\\
		&\quad +2\beta\gamma^3 \hat\sigma_{\mathrm{rad}}^{2}\frac{1}{\gamma\mu}\left( 1+3\beta\frac{\omega}{M} \right) + 3\beta^2 \frac{\omega}{M} \frac{1}{M}\sum_{m=1}^{M}\| x^n_{*,m} - x_* \|^2.
	\end{align*}
	
	Using $(1-\gamma\mu)^n \leq \frac{\nicefrac{9}{10} - \nicefrac{1}{C}}{1+\nicefrac{1}{C}} $, we have 
	\begin{eqnarray*}
		(1-\gamma\mu)^n &\leq & \frac{\nicefrac{9}{10} - \nicefrac{1}{C}}{1+\nicefrac{1}{C}} \\
		(1-\gamma\mu)^n\left(1+\frac{1}{C}\right) & \leq & \frac{9}{10} - \frac{1}{C}\\
		-\frac{9}{10}\beta+\beta(1-\gamma\mu)^n+\frac{\beta}{C}+\frac{\beta}{C}(1-\gamma\mu)^n & \leq & 0\\
		1-\beta+\beta(1-\gamma\mu)^n + \frac{\beta}{C}+\frac{\beta}{C}(1-\gamma\mu)^n & \leq & 1 - \frac{\beta}{10}.
	\end{eqnarray*}
	
	Next, applying $\beta \leq \frac{1}{1+3C\frac{\omega}{M}}$, we derive the inequality 
	\begin{align*}
		1-\beta+\beta(1-\gamma\mu)^n+3\beta^2\frac{\omega}{M}+3\beta^2\frac{\omega}{M}(1-\gamma\mu)^n \leq 1- \frac{\beta}{10}.
	\end{align*}
	Finally, we have 
	\begin{align*}
		\mathbb{E}\|x_{t+1} - x_\star\|^2&\leq	\left(1-\frac{\beta}{10}\right)\|x_t - x_*\|^2+2\beta\gamma^2 \hat\sigma_{\mathrm{rad}}^{2}\frac{1}{\mu}\left( 1+\frac{1}{C} \right)\\
		& \quad + 3\beta^2 \frac{\omega}{M} \frac{1}{M}\sum_{m=1}^{M}\| x^n_{*,m} - x_* \|^2\\
		&\leq	\left(1-\frac{\beta}{10}\right)\|x_t - x_*\|^2+\frac{4}{\mu}\beta\gamma^2 \hat\sigma_{\mathrm{rad}}^{2}\\
		& \quad + 3\beta^2 \frac{\omega}{M} \frac{1}{M}\sum_{m=1}^{M}\| x^n_{*,m} - x_* \|^2.
	\end{align*}
\end{proof}

\newpage

\section{Alternative Analysis of DIANA-NASTYA}
\begin{theorem}\label{thm:Q_NASTYA_alternative_proof}
	Let Assumptions \ref{asm:quantization_operators}, \ref{asm:lip_max_f_m}, \ref{asm:sc_each_f_m} hold. Moreover, we assume that $(1-\gamma\mu)^n\leq \frac{\nicefrac{9}{10} - \nicefrac{1}{B}}{1+\nicefrac{1}{B}} = \widehat{B}<1$ for some numerical constant $B > 1$. Also let $\beta = \frac{\eta}{\gamma n} \leq \frac{1}{12B\frac{\omega}{M}+1}$ and $\gamma \leq \frac{1}{L_{\max}}$ and also $\alpha\leq \frac{1}{\omega+1}$. Then, for all $T \geq 0$ the iterates produced by \algname{DIANA-NASTYA} (Algorithm~\ref{alg:diana-nastya}) satisfy
	\begin{align}
		\mathbb{E}\Psi_{T}	&\leq \max \left(1-\frac{\beta}{10}, 1-\frac{\alpha}{2}\right)^T \Psi_{0} + \frac{2}{\mu\min(\frac{\beta}{10},\frac{\alpha}{2})}\beta\gamma^2\hat{\sigma}_{\text{rad}}^2.
	\end{align}
\end{theorem}
\begin{proof}
We start with expanding the square:
\begin{align*}
	\|x_{t+1} - x_*\|^2 &= \| x_t - \eta \hat{g}_t - x_* \|^2\\
	&= \left\| x_t - \eta \frac{1}{M}\sum_{m=1}^M\left( h_{t,m} + Q(g_{t,m} - h_{t,m}) \right) - x_* \right\|^2\\
	& = \|x_t - x_*\|^2 - 2\eta \left\langle \frac{1}{M}\sum_{m=1}^M\left(h_{t,m}+Q(g_{t,m} - h_{t,m})\right),x_t - x_* \right\rangle\\
	& \quad + \eta^2\left\| \frac{1}{M}\sum_{m=1}^M \left(h_{t,m}+Q(g_{t,m} - h_{t,m})\right) \right\|^2.
\end{align*}
Taking the expectation w.r.t.\ $\cQ$, we get
\begin{align*}
	\mathbb{E}_Q	\|x_{t+1} - x_*\|^2 & = \|x_t - x_*\|^2 - 2\eta \left\langle \frac{1}{M}\sum_{m=1}^M g_{t,m} ,x_t - x_* \right\rangle\\
	& \quad + \eta^2\mathbb{E}_Q\left\| \frac{1}{M}\sum_{m=1}^M \left(h_{t,m}+Q(g_{t,m} - h_{t,m})\right) \right\|^2\\
	&= \|x_t - x_*\|^2 - 2\eta \left\langle \frac{1}{M}\sum_{m=1}^M g_{t,m} ,x_t - x_* \right\rangle\\
	& \quad + \eta^2\mathbb{E}_Q\left\| \frac{1}{M}\sum_{m=1}^M \left(h_{t,m}+Q(g_{t,m} - h_{t,m}) - g_{t,m}\right) \right\|^2+\eta^2\left\| \frac{1}{M}\sum_{m=1}^{M}g_{t,m} \right\|^2\\
	&\leq\|x_t - x_*\|^2 - 2\eta \left\langle \frac{1}{M}\sum_{m=1}^M g_{t,m} ,x_t - x_* \right\rangle\\
	&\quad  + \eta^2\frac{\omega}{M^2}\sum_{m=1}^{M}\|g_{t,m} - h_{t,m}\|^2+\eta^2\left\| \frac{1}{M}\sum_{m=1}^{M}g_{t,m} \right\|^2\\
	&\leq\|x_t - x_*\|^2 - 2\eta \left\langle \frac{1}{M}\sum_{m=1}^M g_{t,m} ,x_t - x_* \right\rangle\\
	& \quad + \eta^2\frac{2\omega}{M^2}\sum_{m=1}^{M}\|g_{t,m} - h_{*,m}\|^2+\eta^2\frac{2\omega}{M^2}\sum_{m=1}^{M}\|h_{t,m} - h_{*,m}\|^2+\eta^2\left\| \frac{1}{M}\sum_{m=1}^{M}g_{t,m} \right\|^2.
\end{align*}
Next, using definition of $g_{t,m}$, we obtain 
\begin{align*}
	\mathbb{E}	\|x_{t+1} - x_*\|^2&\leq	\|x_t - x_*\|^2 - 2\eta \left\langle \frac{1}{M}\sum_{m=1}^M \frac{x_t - x^n_{t,m}}{\gamma n},x_t - x_* \right\rangle+\eta^2\left\| \frac{1}{M}\sum_{m=1}^{M} \frac{x_t - x^n_{t,m}}{\gamma n} \right\|^2\\
	& \quad + \eta^2\frac{2\omega}{M^2}\sum_{m=1}^{M}\|g_{t,m} - h_{*,m}\|^2+\eta^2\frac{2\omega}{M^2}\sum_{m=1}^{M}\|h_{t,m} - h_{*,m}\|^2\\
	&= \|x_t - x_*\|^2 + 2\alpha  \left\langle \frac{1}{M}\sum_{m=1}^M \left(x^n_{t,m}-x_t\right),x_t - x_* \right\rangle + \alpha^2\left\|  \frac{1}{M}\sum_{m=1}^{M} \left(  x^n_{t,m}-x_t \right) \right\|^2\\
	& \quad + \eta^2\frac{2\omega}{M^2}\sum_{m=1}^{M}\|g_{t,m} - h_{*,m}\|^2+\eta^2\frac{2\omega}{M^2}\sum_{m=1}^{M}\|h_{t,m} - h_{*,m}\|^2\\
	&=\left\| x_t - x_* + \alpha  \frac{1}{M}\sum_{m=1}^{M} \left(  x^n_{t,m}-x_t \right)  \right\|^2\\
	& \quad + \eta^2\frac{2\omega}{M^2}\sum_{m=1}^{M}\|g_{t,m} - h_{*,m}\|^2+\eta^2\frac{2\omega}{M^2}\sum_{m=1}^{M}\|h_{t,m} - h_{*,m}\|^2\\
	&=\left\| (1-\beta)(x_t - x_*) + \beta  \left(\frac{1}{M}\sum_{m=1}^{M} \left(  x^n_{t,m}-x^n_{*,m} \right)\right)  \right\|^2\\
	&\leq (1-\beta)\|x_t - x_*\|^2 + \beta \frac{1}{M}\sum_{m=1}^{M}\| x^n_{t,m}-x^n_{*,m} \|^2\\
	&\quad +\eta^2\frac{2\omega}{M^2}\sum_{m=1}^{M}\|g_{t,m} - h_{*,m}\|^2+\eta^2\frac{2\omega}{M^2}\sum_{m=1}^{M}\|h_{t,m} - h_{*,m}\|^2.
\end{align*}

Let us consider recursion for control variable:
\begin{align*}
	\| h_{t+1,m} - h_{*,m} \|^2 &= \|h_{t,m} + \alpha Q(g_{t,m} - h_{t,m}) - h_{*,m}\|^2 \\
	&=\|h_{t,m} - h_{*,m}\|^2 + \alpha \left\langle Q(g_{t,m} - h_{t,m}),h_{t,m} - h_{*,m} \right\rangle+\alpha^2\|Q(g_{t,m} - h_{t,m})\|^2.
\end{align*}

Taking the expectation w.r.t.\ $\cQ$, we have 
\begin{align*}
	\mathbb{E}_{\cQ}	\| h_{t+1,m} - h_{*,m} \|^2 \leq \|h_{t,m} - h_{*,m}\|^2 + 2\alpha \left\langle g_{t,m} - h_{t,m},h_{t,m} - h_{*,m} \right\rangle + \alpha^2\left( \omega + 1 \right)\left\|g_{t,m} - h_{t,m}\right\|^2.
\end{align*}

Using $\alpha \leq \frac{1}{\omega+1}$ we have
\begin{align*}
	\mathbb{E}	\| h_{t+1,m} - h_{*,m} \|^2 &\leq \|h_{t,m} - h_{*,m}\|^2\\
	& \quad + 2\alpha \left\langle g_{t,m} - h_{t,m},h_{t,m} - h_{*,m} \right\rangle + \alpha\left\|g_{t,m} - h_{t,m}\right\|^2\\
	& = \|h_{t,m} - h_{*,m}\|^2\\
	&\quad  + 2\alpha \left\langle g_{t,m} - h_{t,m},h_{t,m} - h_{*,m} \right\rangle + \alpha\left\langle g_{t,m} - h_{t,m},g_{t,m} - h_{t,m}\right\rangle\\
	& = \| h_{t,m} - h_{*,m} \|^2\\
	& \quad + \alpha\left\langle g_{t,m} - h_{t,m},g_{t,m} - h_{t,m}+2h_{t,m}-2h_{*,m}\right\rangle\\	
	& = \| h_{t,m} - h_{*,m} \|^2\\
	& \quad + \alpha\left\langle g_{t,m} - h_{t,m},g_{t,m} +h_{t,m}-2h_{*,m}\right\rangle\\
	& = \| h_{t,m} - h_{*,m} \|^2\\
	& \quad + \alpha\left\langle g_{t,m} - h_{t,m}-h_{*,m}+h_{*,m},g_{t,m} +h_{t,m}-2h_{*,m}\right\rangle\\
	& = \| h_{t,m} - h_{*,m} \|^2\\
	& \quad + \alpha\left\langle g_{t,m} -h_{*,m}- (h_{t,m}-h_{*,m}),(g_{t,m}-h_{*,m}) +(h_{t,m}-h_{*,m})\right\rangle\\
	&= \| h_{t,m} - h_{*,m} \|^2+\alpha\| g_{t,m} - h_{*,m} \|^2 - \alpha \| h_{t,m} - h_{*,m} \|^2\\
	& = (1-\alpha) \| h_{t,m} - h_{*,m} \|^2 +\alpha \| g_{t,m} - h_{*,m} \|^2.
\end{align*}
Using this bound we get that 
\begin{align*}
	\frac{1}{M}\sum_{m=1}^{M}		\mathbb{E}_{\cQ}	\| h_{t+1,m} - h_{*,m} \|^2 \leq (1-\alpha) \frac{1}{M}\sum_{m=1}^{M}\| h_{t,m} - h_{*,m} \|^2 +\alpha \frac{1}{M}\sum_{m=1}^{M} \| g_{t,m} - h_{*,m} \|^2.
\end{align*}
Let us consider the Lyapunov function
\begin{align*}
	\Psi_{t} \eqdef \|x_t - x_*\|^2+\frac{4\omega\eta^2}{\alpha M} \frac{1}{M}\sum_{m=1}^{M}\| h_{t,m} - h_{*,m} \|^2.
\end{align*}
Using previous bounds and Theorem 4 from \citep{mishchenko21_proxim_feder_random_reshuf} we have 
\begin{align*}
	\mathbb{E}\Psi_{t+1}&\leq (1-\beta)\|x_t - x_*\|^2 + \beta\left( (1-\gamma\mu)^n \EE\|x_t - x_*\|^2 + \gamma^3\frac{1}{\gamma\mu}\hat{\sigma}^2_{rad} \right)\\
	&\quad +\eta^2\frac{2\omega}{M}\frac{1}{M}\sum_{m=1}^{M}\EE\| g_{t,m} - h_{*,m} \|^2
	\quad  + \eta^2\frac{2\omega}{M}\frac{1}{M}\sum_{m=1}^{M}\EE\| h_{t,m} - h_{*,m} \|^2\\
	&\quad +(1-\alpha)\frac{4\omega\eta^2}{\alpha M}\frac{1}{M}\sum_{m=1}^{M} \EE\|h_{t,m} - h_{*,m}\|^2 + \alpha \frac{4\omega\eta^2}{\alpha M}\frac{1}{M}\sum_{m=1}^{M} \EE\|g_{t,m} - h_{*,m}\|^2\\
	&\leq \left(1-\frac{\alpha}{2}\right)\frac{4\omega\eta^2}{\alpha M} \frac{1}{M}\sum_{m=1}^{M}\EE\|h_{t,m} - h_{*,m}\|^2 + \eta^2\frac{6\omega}{M}\frac{1}{M}\sum_{m=1}^{M}\EE\| g_{t,m} - h_{*,m} \|^2\\
	&\quad + (1-\beta)\EE\|x_t - x_*\|^2 + \beta\left( (1-\gamma\mu)^n\EE\|x_t - x_*\|^2 + \gamma^3\frac{1}{\gamma\mu}\hat{\sigma}^2_{rad}. \right)
\end{align*}
Let us consider 
\begin{align*}
	\eta^2\frac{1}{M}\sum_{m=1}^{M}\EE\| g_{t,m} - h_{*,m} \|^2  &= \eta^2\frac{1}{M}\sum_{m=1}^{M}\EE\left\| \frac{x_t - x^n_{t,m}}{\gamma n} - \frac{x_* - x^n_{*,m}}{\gamma n}  \right\|^2 \\
	&\leq 2\eta^2 \frac{1}{M}\sum_{m=1}^{M}\EE\left\| \frac{x_t - x_*}{\gamma n} \right\|^2 + 2\eta^2 \frac{1}{M}\sum_{m=1}^{M} \EE\left\|\frac{x^n_{t,m} - x^n_{*,m}}{\gamma n} \right\|^2\\
	&\leq  2\beta^2 \frac{1}{M}\sum_{m=1}^{M}\EE\left\| x_t - x_* \right\|^2 + 2\beta^2 \frac{1}{M}\sum_{m=1}^{M} \EE\left\|x^n_{t,m} - x^n_{*,m} \right\|^2\\
	&\leq  2\beta^2 \EE\left\| x_t - x_* \right\|^2 + 2\beta^2 \frac{1}{M}\sum_{m=1}^{M} \EE\left\|x^n_{t,m} - x^n_{*,m} \right\|^2.\\
\end{align*}
Putting all the terms together and using $(1-\gamma\mu)^n\leq \frac{\nicefrac{9}{10} - \nicefrac{1}{B}}{1+\nicefrac{1}{B}} = \widehat{B}<1$, $\beta \leq \frac{1}{12B\frac{\omega}{M}+1}$  we have 
\begin{align*}
	\mathbb{E}\Psi_{t+1} &\leq \left(1-\beta +12\frac{\omega}{M}\beta^2 + 12\frac{\omega}{M}\beta^2(1-\gamma\mu)^n + \beta(1-\gamma \mu)^n  \right)\EE\|x_t - x_*\|^2 + \beta \gamma^3 \frac{1}{\gamma\mu}\hat{\sigma}^2_{rad}\\ 
	&\quad + 2\beta^2\frac{6\omega}{M}\gamma^3 \frac{1}{\gamma\mu} \hat{\sigma}^2_{rad}
	+\left(1-\frac{\alpha}{2}\right)\frac{4\omega\eta^2}{\alpha M}\frac{1}{M}\sum_{m=1}^M\EE\|h_{t,m} - h_{*,m}\|^2\\
	&\leq \left(1-\frac{\beta}{10}\right)\EE\|x_t - x_\star\|^2 + \frac{2}{\mu}\beta\gamma^2\hat{\sigma}_{\text{rad}}^2 + \left(1-\frac{\alpha}{2}\right)\frac{4\omega\eta^2}{\alpha M} \frac{1}{M}\sum_{m=1}^M\EE\|h_{t,m} - h_{*,m}\|^2\\
	&\leq \max \left(1-\frac{\beta}{10}, 1-\frac{\alpha}{2}\right) \Psi_{t} + \frac{2}{\mu}\beta\gamma^2\hat{\sigma}_{\text{rad}}^2.
\end{align*}
Unrolling this recursion we get the final result.
\end{proof}

\end{document}

%% file: macros.tex
\usepackage{graphicx}
\usepackage{import, ifthen}

\usepackage[pagebackref]{hyperref}       
\hypersetup{
  colorlinks=true,
  linkcolor=blue,
  filecolor=magenta,
  urlcolor=cyan,
  citecolor=blue
}

\renewcommand*{\backrefalt}[4]{%
    \ifcase #1 \footnotesize{(Not cited.)}%
    \or        \footnotesize{(Cited on page~#2)}%
    \else      \footnotesize{(Cited on pages~#2)}%
    \fi}
\usepackage{amsmath,amssymb,amsthm,mathtools,enumitem,hyphenat,float}

\def\<#1,#2>{\langle #1,#2\rangle}

\usepackage{url}            
\usepackage{booktabs}       
\usepackage{amsfonts}       
\usepackage{nicefrac}       
\usepackage{microtype}      
\usepackage{enumitem}
\usepackage{cleveref}
\usepackage{xcolor}
\usepackage{multirow}
\usepackage{threeparttable,adjustbox}
\definecolor{mygreen}{RGB}{77,175,74}
\definecolor{myblue}{RGB}{55,126,184}
\definecolor{mydarkgreen}{RGB}{39,130,67}

\usepackage[colorinlistoftodos,bordercolor=orange,backgroundcolor=orange!20,linecolor=orange,textsize=scriptsize]{todonotes}
\newcommand{\eduard}[1]{\todo[inline]{{\textbf{Eduard:} \emph{#1}}}}

\usepackage{algorithm}
\usepackage[noend]{algpseudocode}

\newcommand{\eqdef}{\overset{\rm def}{=}}
\usepackage{xspace}
\definecolor{mygreen}{HTML}{006B3C}
\newcommand{\algname}[1]{{{\sf \small \color{mygreen} #1}}\xspace}

\newcommand{\cH}{{\cal H}}

\newcommand{\cO}{{\cal O}}

\newcommand{\cQ}{{\cal Q}}

\newcommand{\mA}{{\bf A}}

\newcommand{\mI}{{\bf I}}

\newcommand{\EE}{\mathbf{E}}

\newcommand{\rb}[1]{\left(#1\right)}

\def\R{\mathbb{R}}

\def\R{\mathbb R}

\def\EE{\mathbb E}

\def\la{\langle}
\def\ra{\rangle}

\newtheorem{lemma}{Lemma}
\newtheorem{theorem}{Theorem}
\newtheorem{definition}{Definition}

\newtheorem{assumption}{Assumption}
\newtheorem{corollary}{Corollary}

\newcommand{\norm}[1]{\left\lVert#1\right\rVert}
\newcommand{\sqn}[1]{{\left\lVert#1\right\rVert}^2}

\newcommand\ec[2][]{\ensuremath{\mathbb{E}_{#1} \left[#2\right]}}
\newcommand\ecn[2][]{\ec[#1]{\norm{#2}^2}}
\newcommand\br[1]{\left ( #1 \right )}
\newcommand\ev[1]{\left \langle #1 \right \rangle}